\documentclass{article}
\usepackage{microtype}
\usepackage{graphicx}
\usepackage{booktabs}
\usepackage{hyperref}
\RequirePackage{snapshot}

\usepackage[accepted]{icml2020}
\icmltitlerunning{Confidence-Calibrated Adversarial Training}

\usepackage{amsmath}
\usepackage{xspace}
\usepackage{subcaption}
\usepackage{bbm}
\usepackage{nicefrac}
\usepackage{tabularx}
\usepackage{multirow}
\usepackage{tikz}
\usepackage{amsmath}
\usepackage{amsthm}
\usepackage{amsfonts}
\usepackage{xspace}
\usepackage{mathtools}
\usepackage{algorithm}
\usepackage{algorithmicx}
\usepackage{algcompatible}
\usepackage[labelfont=bf]{caption}
\usepackage{wrapfig}
\usepackage[breakable, theorems, skins]{tcolorbox}
\usepackage{array}
\usepackage{bbold}
\usepackage{placeins}
\def\Vhrulefill{\leavevmode\hrule height 0.7ex depth \dimexpr0.4pt-0.7ex with 1cm\kern0pt}

\usepackage{xcolor}
\definecolor{colorbrewer0}{RGB}{45,45,45}
\definecolor{colorbrewer1}{RGB}{228,26,28}
\definecolor{colorbrewer2}{RGB}{55,126,184}
\definecolor{colorbrewer3}{RGB}{77,175,74}
\definecolor{colorbrewer4}{RGB}{152,78,163}
\definecolor{colorbrewer5}{RGB}{255,127,0}
\definecolor{colorbrewer6}{RGB}{255,255,51}
\definecolor{colorbrewer7}{RGB}{166,86,40}
\definecolor{colorbrewer8}{RGB}{247,129,191}
\definecolor{colorbrewer9}{RGB}{153,153,153}
\definecolor{colorbrewer10}{RGB}{24,167,181}

\newtheorem{proposition}{Proposition}
\definecolor{darkred}{RGB}{255, 0, 0}

\definecolor{gray}{RGB}{100, 100, 100}

\definecolor{colorbrewer0}{RGB}{45,45,45}
\definecolor{colorbrewer1}{RGB}{228,26,28}
\definecolor{colorbrewer2}{RGB}{55,126,184}
\definecolor{colorbrewer3}{RGB}{77,175,74}
\definecolor{colorbrewer4}{RGB}{152,78,163}
\definecolor{colorbrewer5}{RGB}{255,127,0}
\definecolor{colorbrewer6}{RGB}{255,255,51}
\definecolor{colorbrewer7}{RGB}{166,86,40}
\definecolor{colorbrewer8}{RGB}{247,129,191}
\definecolor{colorbrewer9}{RGB}{153,153,153}
\definecolor{colorbrewer10}{RGB}{24,167,181}

\makeatletter
\DeclareRobustCommand\onedot{\futurelet\@let@token\@onedot}
\def\@onedot{\ifx\@let@token.\else.\null\fi\xspace}

\def\eg{\emph{e.g}\onedot} 
\def\ie{\emph{i.e}\onedot} 
\def\cf{\emph{c.f}\onedot} 
\def\etc{\emph{etc}\onedot} 
\def\wrt{w.r.t\onedot} 

\makeatother

\newcommand{\mR}{\mathbb{R}}

\newcommand{\cL}{\mathcal{L}}

\newcommand{\figref}[1]{\Fig~\ref{#1}}
\newcommand{\secref}[1]{\Sec~\ref{#1}}

\renewcommand{\algref}[1]{\Alg~\ref{#1}}
\newcommand{\eqnref}[1]{\Eq~\eqref{#1}}
\newcommand{\tabref}[1]{\Tab~\ref{#1}}

\DeclareMathOperator*{\argmax}{argmax~}

\makeatletter
\DeclareRobustCommand\onedot{\futurelet\@let@token\@onedot}
\def\@onedot{\ifx\@let@token.\else.\null\fi\xspace}
\def\eg{e.g\onedot} 
\def\ie{i.e\onedot} 
 
\def\cf{cf\onedot} 
\def\etc{etc\onedot}

\def\wrt{wrt\onedot}

\def\Fig{Fig\onedot} \def\Eq{Eq\onedot}
\def\Sec{Sec\onedot} \def\Alg{Alg\onedot}
\def\Tab{Tab\onedot} 
\makeatother

\DeclareRobustCommand{\RTE}{%
    \ifmmode
    \text{RErr}
    \else
    RErr\xspace
    \fi
}
\DeclareRobustCommand{\TE}{%
    \ifmmode
    \text{Err}
    \else
    Err\xspace
    \fi
}
\def\PGD{\text{PGD}\xspace}
\def\Normal{\text{Normal}\xspace}
\def\AdvTrain{\text{AT}\xspace}
\def\AdvTrainHalf{\text{AT-50\%}\xspace}
\def\AdvTrainFull{\text{AT-100\%}\xspace}
\def\FConf{\text{Conf}\xspace}
\def\FCE{\text{CE}\xspace}
\def\ConfTrain{\text{CCAT}\xspace}
\def\Wong{MSD\xspace}
\def\WongAT{AT-100\% (Maini)\xspace}
\def\MadryAT{AT-Madry\xspace}
\def\Ma{LID\xspace}
\def\Lee{MAHA\xspace}
\def\TRADES{TRADES\xspace}

\def\Distal{\text{Dist}\xspace}

\def\Random{\text{Rand}\xspace}

\def\BlackBox{\text{Black-Box}\xspace}

\def\Pr{p}

\def\R{\mathbb{R}}

\newcommand{\norm}[1]{\left\|#1\right\|}

\def\Exp{\mathbb{E}}

\newcommand{\Id}{\mathbb{1}}

\def\out{\mathrm{out}}

\def\min{\mathop{\rm min}\nolimits}
\def\max{\mathop{\rm max}\nolimits}

\begin{document}

\twocolumn[
\icmltitle{Confidence-Calibrated Adversarial Training: Generalizing to Unseen Attacks}

\begin{icmlauthorlist}
    \icmlauthor{David Stutz}{mpii}
    \icmlauthor{Matthias Hein}{tue}
    \icmlauthor{Bernt Schiele}{mpii}
\end{icmlauthorlist}

\icmlaffiliation{mpii}{Max Planck Institute for Informatics, Saarland Informatics Campus, Saarbr\"{u}cken}
\icmlaffiliation{tue}{University of T\"{u}bingen, T\"{u}bingen}

\icmlcorrespondingauthor{David Stutz}{david.stutz@mpi-inf.mpg.de}

\icmlkeywords{Adversarial Machine Learning, Adversarial Examples, Adversarial Robustness, Adversarial Training}

\vskip 0.3in
]

\printAffiliationsAndNotice{}

\begin{abstract}
    Adversarial training yields robust models against a specific threat model, \eg, $L_\infty$ adversarial examples. Typically robustness does \emph{not} generalize to previously unseen threat models, \eg, other $L_p$ norms, or larger perturbations.
    Our \textbf{confidence-calibrated adversarial training (\ConfTrain)} tackles this problem by biasing the model towards low confidence predictions on adversarial examples. By allowing to reject examples with low confidence, robustness generalizes beyond the threat model employed during training.
    \ConfTrain, trained \emph{only} on $L_\infty$ adversarial examples, increases robustness against larger $L_\infty$, $L_2$, $L_1$ and $L_0$ attacks, adversarial frames, distal adversarial examples and corrupted examples 
    and yields better clean accuracy compared to adversarial training.
    For thorough evaluation we developed novel white- and black-box attacks directly attacking \ConfTrain by maximizing confidence.
    For each threat model, we use $7$ attacks with up to $50$ restarts and $5000$ iterations and report worst-case robust test error, extended to our confidence-thresholded setting, across \emph{all} attacks.
\end{abstract}
\section{Introduction}
\label{sec:introduction}

\begin{figure}[t]
    
    \hspace*{-0.25cm}
    \begin{minipage}[t]{0.5\textwidth}
        \vspace*{-10px}
        \centering
        \begin{minipage}{0.495\textwidth}
        	\begin{tcolorbox}[
        		left=0pt,
        		right=0pt,
        		top=0pt,
        		bottom=0pt,
        		colback=white,
        		colframe=gray!10!white,
        		width=1\textwidth, 
        		enlarge left by=0mm,
        		boxsep=5pt,
        		arc=0pt,outer arc=0pt,
        		boxrule=1pt,
        		]
        		
        		\centering
        		\small\textbf{Adversarial Training (\AdvTrain):}
                \vspace*{-1px}
        	\end{tcolorbox}
            \begin{tcolorbox}[
                left=0pt,
                right=0pt,
                top=0pt,
                bottom=0pt,
                colback=white,
                colframe=gray!10!white,
                width=1\textwidth, 
                enlarge left by=0mm,
                boxsep=5pt,
                arc=0pt,outer arc=0pt,
                boxrule=1pt,
                ]
                
                \vspace*{-4px}
                \includegraphics[width=3.8cm]{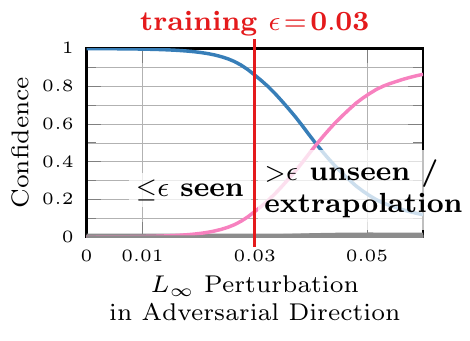}
                \vskip -2px
                
                \includegraphics[width=3.5cm]{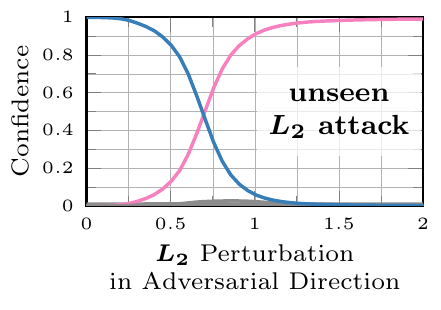}
                \vspace*{-6px}
            \end{tcolorbox}
        \end{minipage}
        \hfill
        \begin{minipage}{0.495\textwidth}
        	\begin{tcolorbox}[
        		left=0pt,
        		right=0pt,
        		top=0pt,
        		bottom=0pt,
        		colback=gray!12!white,
        		colframe=gray!12!white,
        		width=1\textwidth, 
        		enlarge left by=0mm,
        		boxsep=5pt,
        		arc=0pt,outer arc=0pt,
        		boxrule=1pt,
        		]
        		
        		\centering
        		\small\textbf{Ours (CCAT):}
                \vspace*{-1px}
        	\end{tcolorbox}
            \begin{tcolorbox}[
                left=0pt,
                right=0pt,
                top=0pt,
                bottom=0pt,
                colback=gray!12!white,
                colframe=gray!12!white,
                width=1\textwidth, 
                enlarge left by=0mm,
                boxsep=5pt,
                arc=0pt,outer arc=0pt,
                boxrule=1pt,
                ]
                
                \vspace*{-4px}
                \includegraphics[width=3.8cm]{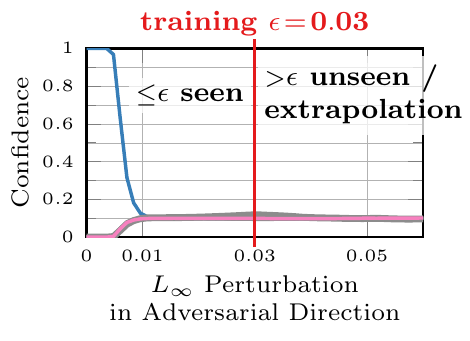}
                \vskip -2px
                
                \includegraphics[width=3.5cm]{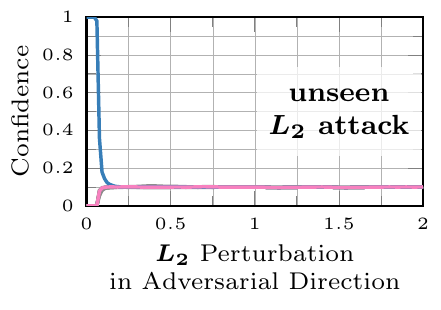}
                \vspace*{-6px}
            \end{tcolorbox}
        \end{minipage}
    
        \begin{tcolorbox}[
            left=0pt,
            right=0pt,
            top=0pt,
            bottom=0pt,
            colback=white,
            colframe=gray!10!white,
            width=1\textwidth, 
            enlarge left by=0mm,
            boxsep=5pt,
            arc=0pt,outer arc=0pt,
            boxrule=1pt,
            ]
            \centering
            \scriptsize
            
            \textcolor{colorbrewer2}{{\rule[2pt]{10pt}{1pt} (Correct) Predicted Class}}\xspace\xspace\xspace \textcolor{colorbrewer8}{{\rule[2pt]{10pt}{1pt} Adversarial Class}}\xspace\xspace\xspace \textcolor{colorbrewer9}{{\rule[2pt]{10pt}{1pt} Other Classes}}
        \end{tcolorbox}
    \end{minipage}
    \vskip -2px
    
    \caption{
    \textbf{Adversarial Training (\AdvTrain) versus our \ConfTrain.} We plot the confidence in the 
    direction of an adversarial example. AT enforces high
    confidence predictions for the correct class on the $L_\infty$-ball of radius $\epsilon$ (\emph{``seen''}
    attack during training, top left). As AT
    enforces no particular bias beyond the $\epsilon$-ball, adversarial examples can be found right beyond this ball. In contrast \ConfTrain enforces a decaying confidence in the correct class up to uniform confidence within the $\epsilon$-ball (top right). Thus, \ConfTrain biases the model to extrapolate uniform confidence beyond
    the $\epsilon$-ball. This behavior also extends to \emph{``unseen''} attacks during training, \eg, $L_2$ attacks (bottom), such that adversarial examples can be
    rejected via confidence-thresholding.
    }
    \label{fig:introduction}
    \vspace*{-2px}
\end{figure}

Deep networks were shown to be susceptible to adversarial examples \cite{SzegedyICLR2014}: adversarially perturbed examples that cause mis-classification while being nearly ``imperceptible'', \ie, close to the original example. Here, ``closeness'' is commonly enforced by constraining the $L_p$ norm of the perturbation, referred to as threat model. Since then, numerous defenses against adversarial examples have been proposed. However, many were unable to keep up with more advanced attacks \cite{AthalyeICML2018b,AthalyeARXIV2018b}. Moreover, most defenses are tailored to only one specific threat model.

Adversarial training \citep{GoodfellowICLR2015,MadryICLR2018}, \ie, training on adversarial examples, can be regarded as state-of-the-art. However, following \figref{fig:introduction}, adversarial training is known to ``overfit'' to the threat model \emph{``seen''} during training, \eg, $L_\infty$ adversarial examples.
Thus, robustness does not extrapolate to larger $L_\infty$ perturbations, \cf \figref{fig:introduction} (top left), or generalize to \emph{``unseen''} attacks, \cf \figref{fig:introduction} (bottom left), \eg, other $L_p$ threat models \citep{SharmaICLRWORK2018,TramerNIPS2019,LiARXIV2019,KangARXIV2019,MainiICML2020}.
We hypothesize this to be a result of enforcing high-confidence predictions on adversarial examples. However, high-confidence predictions are difficult to extrapolate beyond the adversarial examples seen during training. Moreover, it is not meaningful to extrapolate high-confidence predictions to arbitrary regions.
Finally, adversarial training often hurts accuracy, resulting in a robustness-accuracy trade-off \citep{TsiprasICLR2019,StutzCVPR2019,RaghunathanARXIV2019,ZhangICML2019}.

\textbf{Contributions:}
We propose \textbf{confidence-calibrated adversarial training (\ConfTrain)} which trains the network to predict a convex combination of uniform and (correct) one-hot distribution on adversarial examples that becomes more uniform as the distance to the attacked example increases. This is illustrated in \figref{fig:introduction}. Thus, \ConfTrain implicitly biases the network to predict a uniform distribution beyond the threat model \emph{seen} during training, \cf \figref{fig:introduction} (top right). Robustness is obtained by rejecting low-confidence (adversarial) examples through confidence-thresholding. As a result, having seen \emph{only} $L_\infty$ adversarial examples during training, \ConfTrain improves robustness against previously \emph{unseen} attacks, \cf \figref{fig:introduction} (bottom right), \eg, $L_2$, $L_1$ and $L_0$ adversarial examples or larger $L_\infty$ perturbations. Furthermore, robustness extends to adversarial frames \cite{ZajaxAAAIWORK2019}, distal adversarial examples \cite{HeinCVPR2019}, corrupted examples (\eg, noise, blur, transforms \etc) and accuracy of normal training is preserved better than with adversarial training.

For thorough evaluation, following best practices \cite{CarliniARXIV2019}, we adapt several state-of-the-art white- and black-box attacks \cite{MadryICLR2018,IlyasICML2018,AndriushchenkoARXIV2019,NarodytskaCVPRWORK2017,KhouryARXIV2018} to \ConfTrain by explicitly maximizing confidence and improving optimization through a backtracking scheme. In total, we consider $7$ different attacks for each threat model (\ie, $L_p$ for $p \in \{\infty, 2, 1, 0\}$), allowing up to $50$ random restarts and $5000$ iterations each. We report worst-case robust test error, extended to our confidence-thresholded setting, across \emph{all} attacks and restarts, on a \emph{per test example} basis. We demonstrate improved robustness against unseen attacks compared to standard adversarial training \cite{MadryICLR2018}, TRADES \cite{ZhangICML2019}, adversarial training using multiple threat models \cite{MainiICML2020} and two detection methods \cite{MaICLR2018,LeeNIPS2018}, while training \emph{only} on $L_\infty$ adversarial examples.

We make our code (training and evaluation) and pre-trained models publicly available at \href{http://davidstutz.de/ccat}{davidstutz.de/ccat}.
\section{Related Work}
\label{sec:related-work}

\textbf{Adversarial Examples:}
Adversarial examples can roughly be divided into white-box attacks, \ie, with access to the model gradients, \eg \cite{GoodfellowICLR2015,MadryICLR2018,CarliniSP2017}, and black-box attacks, \ie, only with access to the model's output, \eg \cite{IlyasICML2018,NarodytskaCVPRWORK2017,AndriushchenkoARXIV2019}. Adversarial examples were also found to be transferable between models \cite{LiuICLR2017,XieCVPR2019}. In addition to \emph{imperceptible} adversarial examples, adversarial transformations, \eg,  \cite{EngstromICML2019,AlaifariICLR2019}, or adversarial patches \cite{BrownARXIV2017} have also been studied. Recently, projected gradient ascent to maximize the cross-entropy loss or surrogate objectives, \eg, \cite{MadryICLR2018,DongCVPR2018,CarliniSP2017}, has become standard. Instead, we directly maximize the confidence in any but the true class, similar to \cite{HeinCVPR2019,GoodfellowOPENREVIEW2019}, in order to effectively train and attack \ConfTrain.

\textbf{Adversarial Training:}
Numerous defenses have been proposed, of which several were shown to be ineffective \cite{AthalyeICML2018b,AthalyeARXIV2018b}. Currently, adversarial training is standard to obtain robust models. While it was proposed in different variants  \cite{SzegedyICLR2014,MiyatoICLR2016,HuangARXIV2015}, the formulation by \cite{MadryICLR2018} received considerable attention and has been extended in various ways: \cite{ShafahiAAAI2020,PielotARXIV2018} train on universal adversarial examples, in \cite{CaiIJCAI2018}, curriculum learning is used, and in \cite{TramerICLR2018,GrefenstetteARXIV2018} ensemble adversarial training is proposed. 
The increased sample complexity \cite{SchmidtNIPS2018} was addressed in \cite{LambAISEC2019,CarmonNIPS2019,UesatoNIPS2019} by training on interpolated or unlabeled examples. Adversarial training on multiple threat models is also possible \cite{TramerNIPS2019,MainiICML2020}. Finally, the observed robustness-accuracy trade-off has been discussed in \cite{TsiprasICLR2019,StutzCVPR2019,ZhangICML2019,RaghunathanARXIV2019}. Adversarial training has also been combined with self-supervised training \cite{HendrycksNIPS2019}. In contrast to adversarial training, \ConfTrain imposes a target distribution which tends towards a uniform distribution for large perturbations, allowing the model
to extrapolate beyond the threat model used at training time. Similar to adversarial training with an additional ``abstain'' class \cite{LaidlawARXIV2019}, robustness is obtained by rejection. In our case, rejection is based on confidence thresholding.

\textbf{Detection:}
Instead of correctly classifying adversarial examples, several works \cite{GongARXIV2017,GrosseARXIV2017,FeinmanARXIV2017,LiaoCVPR2018,MaICLR2018,AmsalegWIFS2017,MetzenICLR2017,BhagojiARXIV2017,HendrycksICLR2017,LiICCV2017,LeeNIPS2018} try to detect adversarial examples. However, several detectors have been shown to be ineffective against adaptive attacks aware of the detection mechanism \cite{CarliniAISec2017}. Recently, the detection of adversarial examples by confidence, similar to our approach with \ConfTrain, has also been discussed \cite{PangNIPS2018}. Instead, \citet{GoodfellowOPENREVIEW2019} focus on evaluating confidence-based detection methods using adaptive, targeted attacks maximizing confidence. Our attack, although similar in spirit, is untargeted and hence suited for \ConfTrain.
\section{Generalizable Robustness by Confidence Calibration of Adversarial Training}
\label{sec:main}

To start, we briefly review adversarial training on $L_\infty$ adversarial examples \cite{MadryICLR2018}, which has become standard to train robust models, \cf \secref{sec:main-at}.
However, robustness does not generalize to larger perturbations or unseen attacks. We hypothesize this to be the result of enforcing high-confidence predictions on adversarial examples.
\ConfTrain addresses this issue with minimal modifications, \cf \secref{subsec:main-ccat} and \algref{alg:main-ccat}, by encouraging low-confidence predictions on adversarial examples. During testing, adversarial examples can be rejected by confidence thresholding.

\textbf{Notation:}
We consider a classifier $f:\R^d \rightarrow \R^K$ with $K$ classes where $f_k$ denotes the confidence for class $k$. While we use the cross-entropy loss $\cL$ for training, our approach also generalizes to other losses. Given $x \in \R^d$ with class $y{\,\in\,}\{1,\ldots,K\}$, we let $f(x) :=\argmax_k f_k(x)$ denote the predicted class for notational convenience. For $f(x) = y$, an adversarial example $\tilde{x} = x+\delta$ is defined as a ``small'' perturbation $\delta$ such that $f(\tilde{x}) \neq y$, \ie, the classifier changes its decision. The strength of the change $\delta$ is measured by some $L_p$-norm, $p \in \{0,1,2,\infty\}$. Here, $p=\infty$ is a popular choice as it leads to the smallest perturbation per pixel.

\subsection{Problems of Adversarial Training}
\label{sec:main-at}

Following \cite{MadryICLR2018}, adversarial training is given as the following min-max problem:
\vspace*{0px}
\begin{align}
    \min\limits_w \Exp\left[\max\limits_{\|\delta\|_\infty \leq \epsilon}\, \cL(f(x + \delta; w), y)\right]\label{eq:adversarial-training}
\end{align}
\vskip -4px
with $w$ being the classifier's parameters. During mini-batch training the inner maximization problem,
\vspace*{0px}
\begin{align}
    \max\limits_{\|\delta\|_\infty \leq \epsilon} \cL(f(x + \delta; w), y),\label{eq:attack}
\end{align}
\vskip -4px
is approximately solved. In addition to the $L_\infty$-constraint, a box constraint is enforced for images, \ie, $\tilde{x}_i = (x + \delta)_i \in [0,1]$. Note that maximizing the cross-entropy loss is equivalent to finding the adversarial example with \emph{minimal} confidence in the true class. For neural networks, this is generally a non-convex optimization problem. In \citep{MadryICLR2018} the problem is tackled using projected gradient descent (\PGD), initialized using a random $\delta$ with $\|\delta\|_\infty \leq \epsilon$.

In contrast to adversarial training as proposed in \citep{MadryICLR2018}, which computes adversarial examples for the \emph{full} batch in each iteration, others compute adversarial examples only for \emph{half} the examples of each batch \citep{SzegedyICLR2014}. Instead of training \emph{only} on adversarial examples, each batch is divided into $50\%$ clean and $50\%$ adversarial examples. Compared to \eqnref{eq:adversarial-training}, $50\%$/$50\%$ adversarial training effectively minimizes
\vspace*{0px}
\begin{align}
    \underbrace{\Exp\Big[\max\limits_{\|\delta\|_\infty \leq \epsilon} \cL(f(x + \delta; w), y)\Big]}_{\text{50\% adversarial training}} + \underbrace{\Exp\big[\cL(f(x; w), y)\big]}_{\text{50\% ``clean'' training}}.\label{eq:50-50-adversarial-training}
\end{align}
\vskip -4px
This improves test accuracy on clean examples compared to $100\%$ adversarial training but typically leads to worse robustness. Intuitively, by balancing both terms in \eqnref{eq:50-50-adversarial-training}, the trade-off between accuracy and robustness can already be optimized to some extent \citep{StutzCVPR2019}.

\begin{figure}[t]
    \vspace*{-8px}
    
    \centering
    \begin{minipage}{0.49\textwidth}
        \hspace*{-0.2cm}
        \includegraphics[width=1\textwidth]{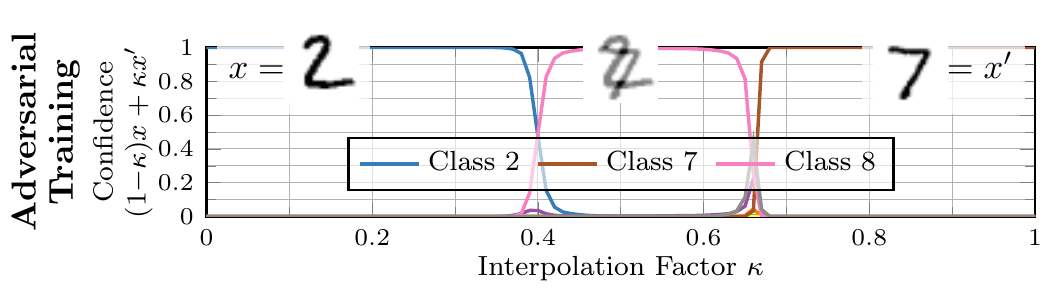}
        
        \hspace*{-0.2cm}
        \includegraphics[width=1\textwidth]{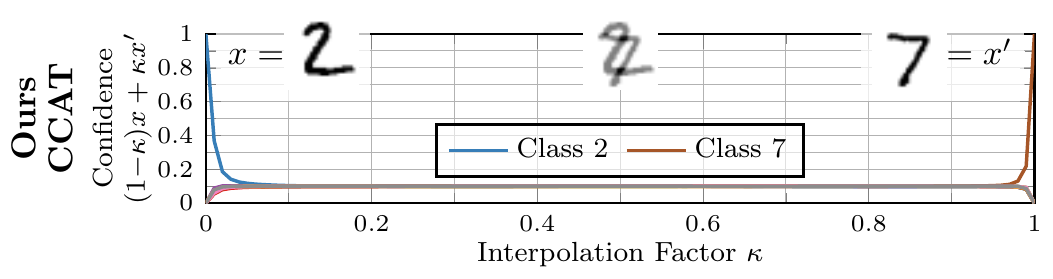}
    \end{minipage}
    \vspace*{-10px}
    
    \caption{\textbf{Extrapolation of Uniform Predictions.} We plot the confidence in each class along an interpolation between two test examples $x$ and $x'$, ``2'' and ``7'', on MNIST \cite{LecunIEEE1998}: $(1 - \kappa)x + \kappa x'$ where $\kappa$ is the interpolation factor. \ConfTrain quickly yields low-confidence, uniform predictions in between both examples, extrapolating the behavior enforced within the $\epsilon$-ball during training. Regular adversarial training, in contrast, consistently produces high-confidence predictions, even on unreasonable inputs.}
    \label{fig:interpolation}
    \vspace*{-2px}
\end{figure}

\textbf{Problems:}
Trained on $L_\infty$ adversarial examples, the robustness of adversarial training does not generalize to previously unseen adversarial examples, including larger perturbations or other $L_p$ adversarial examples. We hypothesize that this is because adversarial training explicitly enforces high-confidence predictions on $L_\infty$ adversarial examples within the $\epsilon$-ball seen during training (``seen'' in \figref{fig:introduction}). However, this behavior is difficult to extrapolate to arbitrary regions in a meaningful way. Thus, it is not surprising that adversarial examples can often be found right beyond the $\epsilon$-ball used during training, \cf \figref{fig:introduction} (top left). This can be described as ``overfitting'' to the $L_\infty$ adversarial examples used during training. Also, larger $\epsilon$-balls around training examples might include (clean) examples from other classes. Then, \eqnref{eq:attack} will focus on these regions and reduce accuracy as considered in our theoretical toy example, see Proposition \ref{prop:toy-example}, and related work \citep{JacobsenICLR2019,JacobsenARXIV2019}.

As suggested in \figref{fig:introduction}, both problems can be addressed by enforcing low-confidence predictions on adversarial examples in the $\epsilon$-ball. In practice, we found that the low-confidence predictions on adversarial examples within the $\epsilon$-ball are extrapolated beyond the $\epsilon$-ball, \ie, to larger perturbations, unseen attacks or distal adversarial examples. This allows to reject adversarial examples based on their low confidence. We further enforce this behavior by explicitly encouraging a ``steep'' transition from high-confidence predictions (on clean examples) to low-confidence predictions (on adversarial examples). As result, the (low-confidence) prediction is almost flat close to the boundary of the $\epsilon$-ball. Additionally, there is no incentive to deviate from the uniform distribution outside of the $\epsilon$-ball. For example, as illustrated in \figref{fig:interpolation}, the confidence stays low in between examples from different classes and only increases if necessary, \ie, close to the examples.

\begin{algorithm}[t]
	\caption{\textbf{Confidence-Calibrated Adversarial Training (\ConfTrain).}  The only changes compared to standard adversarial training are the attack (line \ref{line:attack}) and the probability distribution over the classes (lines \ref{line:lambda} and \ref{line:label}), which becomes more uniform as distance $\norm{\delta}_\infty$ increases. During testing, low-confidence (adversarial) examples are rejected.}
	\label{alg:main-ccat}
	\begin{algorithmic}[1]
		\WHILE{true}
		\STATE choose random batch $(x_1,y_1),\ldots,(x_B,y_B)$.
		\FOR{$b = 1,\ldots,\nicefrac{B}{2}$} 
		\STATE $\delta_b:=\argmax\limits_{\|\delta\|_\infty \leq \epsilon} \max\limits_{k \neq y_b} f_k(x_b{+}\delta)$\label{line:attack} (\eqnref{eq:conf-attack})
		\STATE $\tilde{x}_b := x_b + \delta_b$
		\STATE $\lambda(\delta_b) := (1 - \min(1, \nicefrac{\|\delta_b\|_\infty}{\epsilon}))^\rho$ (\eqnref{eq:lambda})\label{line:lambda}
		\STATE $\tilde{y_b}\,{:=}\,\lambda(\delta_b)\,\text{one\_hot}(y_b)\,{+}\,(1\,{-}\,\lambda(\delta_b)) \frac{1}{K}$ (\eqnref{eq:distribution})\label{line:label}
		\ENDFOR
		\STATE update parameters using \eqnref{eq:50-50-adversarial-training}:\\\hspace*{1cm}$\sum_{b = 1}^{\nicefrac{B}{2}} \mathcal{L}(f(\tilde{x}_b), \tilde{y}_b) + \sum_{b = \nicefrac{B}{2}}^{B} \mathcal{L}(f(x_b), y_b)$
		\ENDWHILE
	\end{algorithmic}
\end{algorithm}

\subsection{Confidence-Calibrated Adversarial Training}
\label{subsec:main-ccat}

\textbf{Confidence-calibrated adversarial training (\ConfTrain)} addresses these problems with minimal modifications, as outlined in \algref{alg:main-ccat}. During training, we train the network to predict a convex combination of (correct) one-hot distribution on clean examples and uniform distribution on adversarial examples as target distribution within the cross-entropy loss. During testing, adversarial examples can be rejected by confidence thresholding: adversarial examples receive near-uniform confidence while test examples receive high-confidence. By extrapolating the uniform distribution beyond the $\epsilon$-ball used during training, previously unseen adversarial examples such as larger $L_\infty$ perturbations can be rejected, as well. In the following, we first introduce an alternative objective for generating adversarial examples. Then, we specifically define the target distribution, which becomes more uniform with larger perturbations $\|\delta\|_\infty$. In \algref{alg:main-ccat}, these changes correspond to lines \ref{line:attack}, \ref{line:lambda} and \ref{line:label}, requiring only few lines of code in practice.

Given an example $x$ with label $y$, our adaptive attack during training maximizes the confidence in any other label $k \neq y$. This results in effective attacks against \ConfTrain, as \ConfTrain will reject low-confidence adversarial examples:
\vspace*{0px}
\begin{align}
 \max\limits_{\|\delta\|_\infty \leq \epsilon} \max\limits_{k\neq y}f_k(x + \delta;w) \label{eq:conf-attack}
\end{align}
\vskip -4px
Note that \eqnref{eq:attack}, in contrast, minimizes the confidence in the true label $y$. Similarly, \citep{GoodfellowOPENREVIEW2019} uses targeted attacks in order to maximize confidence, whereas ours is untargeted and, thus, our objective is the maximal confidence over all other classes.

Then, given an adversarial example from \eqnref{eq:conf-attack} during training, \ConfTrain uses the following combination of uniform and one-hot distribution as target for the cross-entropy loss:
\vspace*{-8px}
\begin{align}
	\tilde{y} = \lambda(\delta) \,\,\text{one\_hot}(y) + \big(1-\lambda(\delta)\big) \frac{1}{K},\label{eq:distribution}
\end{align}
\vskip -4px
with $\lambda(\delta) \in [0,1]$ and $\text{one\_hot}(y) \in \{0,1\}^K$ denoting the one-hot vector corresponding to class $y$. Thus, we enforce a convex combination of the original label distribution and the uniform distribution which is controlled by the parameter~$\lambda = \lambda(\delta)$, computed given the perturbation~$\delta$. We choose $\lambda$ to decrease with the distance $\|\delta\|_\infty$ of the adversarial example to the attacked example $x$ with the intention to enforce uniform predictions when $\|\delta\|_\infty = \epsilon$. Then, the network is encouraged to extrapolate this uniform distribution beyond the used $\epsilon$-ball. Even if extrapolation does not work perfectly, the uniform distribution is much more meaningful for extrapolation to arbitrary regions as well as regions between classes compared to high-confidence predictions as encouraged in standard adversarial training, as demonstrated in \figref{fig:interpolation}. For controlling the trade-off $\lambda$ between one-hot and uniform distribution, we consider the following ``power transition'':
\vspace*{0px}
\begin{align}
    \begin{split}
        \lambda(\delta) :=& \Big(1 - \min\Big(1, \frac{\|\delta\|_\infty}{\epsilon}\Big)\Big)^\rho
    \end{split}\label{eq:lambda}
\end{align}
\vskip -4px
This ensures that for $\delta = 0$ we impose the original (one-hot) label. For growing $\delta$, however, the influence of the original label decays proportional to $\|\delta\|_\infty$. The speed of decay is controlled by the parameter $\rho$.
For $\rho=10$, \figref{fig:introduction} (top right) shows the transition as approximated by the network.
The power transition ensures that for $\|\delta\|_\infty \geq \epsilon$, \ie, perturbations larger than encountered during training, a uniform distribution is enforced as $\lambda$ is $0$. We train on $50\%$ clean and $50\%$ adversarial examples in each batch, as in \eqnref{eq:50-50-adversarial-training}, such that the network has an incentive to predict correct labels.

The convex combination of uniform and one-hot distribution in \eqnref{eq:distribution} resembles the label smoothing regularizer introduced in \cite{SzegedyCVPR2016}. In concurrent work, label smoothing has also been used as regularizer for adversarial training \cite{ChengARXIV2020}. However, in our case, $\lambda = \lambda(\delta)$ from \eqnref{eq:lambda} is not a fixed hyper-parameter as in \cite{SzegedyCVPR2016,ChengARXIV2020}. Instead, $\lambda$ depends on the perturbation $\delta$ and reaches zero for $\|\delta\|_\infty = \epsilon$ to encourage low-confidence predictions beyond the $\epsilon$-ball used during training. Thereby, $\lambda$ explicitly models the transition from one-hot to uniform distribution.

\subsection{Confidence-Calibrated Adversarial Training Results in Accurate Models}
\label{sec:main-proposition}

Proposition \ref{prop:toy-example} discusses a problem where standard adversarial training is unable to reconcile robustness and accuracy while \ConfTrain is able to obtain \emph{both} robustness and accuracy:

\begin{proposition}\label{prop:toy-example}
    We consider a classification problem with two points $x=0$ and $x=\epsilon$ in $\mR$ with deterministic labels, \ie,
    $p(y=2|x=0)=1$ and $p(y=1|x=\epsilon)=1$, such that the problem is fully determined by the probability $p_0=p(x=0)$. 
    The Bayes error of this classification problem is zero. Let the predicted probability distribution over classes be $\tilde{p}(y|x)=\frac{e^{g_y(x)}}{e^{g_1(x)}+e^{g_2(x)}}$, where $g:\R^d \rightarrow \R^2$ is the classifier and we assume that the function $\lambda:\mR_+ \rightarrow [0,1]$ used in \ConfTrain is monotonically decreasing and $\lambda(0)=1$. Then, the error of the Bayes optimal classifier (with cross-entropy loss) for
    \vspace*{-8px}
    \begin{itemize}
    \item adversarial training on $100\%$ adversarial examples, \cf \eqnref{eq:adversarial-training}, 
    is $\min\{p_0,1-p_0\}$.
    \vspace*{-3px}
    \item adversarial training on $50\%$/$50\%$ adversarial/clean examples per batch, \cf \eqnref{eq:50-50-adversarial-training}, 
    is $\min\{p_0,1 - p_0\}$.
    \vspace*{-3px}
    \item \ConfTrain on $50\%$ clean and $50\%$ adversarial examples, \cf \algref{alg:main-ccat}, 
    is \emph{zero} 
    if $\lambda(\epsilon) < \min\left\{\nicefrac{p_0}{1-p_0},\nicefrac{1-p_0}{p_0}\right\}$.
    \vspace*{-6px}
    \end{itemize}
\end{proposition}

Here, $100\%$ and $50\%$/$50\%$ standard adversarial training are unable to obtain \emph{both} robustness and accuracy: The $\epsilon$-ball used during training contains examples of different classes such that adversarial training enforces high-confidence predictions in contradicting classes. \ConfTrain addresses this problem by encouraging low-confidence predictions on adversarial examples within the $\epsilon$-ball. Thus, \ConfTrain is able to improve accuracy while preserving robustness.

\section{Detection and Robustness Evaluation with Adaptive Attack}
\label{sec:evaluation-attack}

\ConfTrain allows to reject (adversarial) inputs by confidence-thresholding before classifying them. As we will see, this ``reject option'', is also beneficial for standard adversarial training (\AdvTrain). Thus, evaluation also requires two stages: First, we fix the confidence threshold at $99\%$ true positive rate (TPR), where correctly classified clean examples are positives such that at most $1\%$ (correctly classified) clean examples are rejected.
Second, on the non-rejected examples, we evaluate accuracy and robustness using \emph{confidence-thresholded} (robust) test error. 

\subsection{Adaptive Attack}
\label{subsec:evaluation-attack}

As \ConfTrain encourages low confidence on adversarial examples, we use \PGD to maximize the confidence of adversarial examples, \cf \eqnref{eq:conf-attack}, as effective adaptive attack against \ConfTrain. In order to effectively optimize our objective, we introduce a simple but crucial improvement: after each iteration, the computed update is only applied if the objective is improved; otherwise the learning rate is reduced.
Additionally, we use momentum \citep{DongCVPR2018} and run the attack for exactly $T$ iterations, choosing the perturbation corresponding to the best objective across all iterations. In addition to random initialization, we found that $\delta = 0$ is an effective initialization against \ConfTrain. We applied the same principles for \cite{IlyasICML2018}, \ie, \PGD with approximated gradients, \eqnref{eq:conf-attack} as objective, momentum and backtracking; we also use \eqnref{eq:conf-attack} as objective for the black-box attacks of \cite{AndriushchenkoARXIV2019,NarodytskaCVPRWORK2017,KhouryARXIV2018}.

\subsection{Detection Evaluation}

In the first stage, we consider a detection setting: adversarial example are \emph{negatives} and correctly classified clean examples are \emph{positives}. The confidence threshold $\tau$ is chosen extremely conservatively by requiring a \textbf{$\boldsymbol{99\%}$ true positive rate (TPR)}: at most $1\%$ of correctly classified clean examples can be rejected. As result, the confidence threshold is determined \emph{only} by correctly classified clean examples, independent of adversarial examples. Incorrectly rejecting a significant fraction of correctly classified clean examples is unacceptable. This is also the reason why we do not report the area under the receiver operating characteristic (ROC) curve as related work \citep{LeeNIPS2018,MaICLR2018}.
Instead, we consider the \textbf{false positive rate (FPR)}.
The supplementary material includes a detailed discussion.

\subsection{Robustness Evaluation}

In the second stage, after confidence-thresholding, we consider the widely used robust test error (\RTE) \cite{MadryICLR2018}. It quantifies the model's test error in the case where all test examples are allowed to be attacked, \ie, modified within the chosen threat model, \eg, for $L_p$:
\begin{align}
    \text{``Standard'' }\RTE = \frac{1}{N}\sum_{n = 1}^N \;\max\limits_{\|\delta\|_p\leq \epsilon} \Id_{f(x_n + \delta)\neq y_n}
\end{align}
where $\{(x_n,y_n)\}_{n = 1}^N$ are test examples and labels. In practice, \RTE is computed empirically using adversarial attacks. Unfortunately, standard \RTE does not take into account the option of rejecting (adversarial) examples.

We propose a generalized definition adapted to our confidence-thresholded setting where the model can reject examples. For fixed confidence threshold $\tau$ at $99\%$TPR, the \textbf{confidence-thresholded \RTE} is defined as
\begin{align}
    \RTE(\tau) = \frac{
    	\sum\limits_{n=1}^N \;\max\limits_{\|\delta\|_p\leq \epsilon, c(x_n + \delta)\geq \tau} \Id_{f(x_n + \delta)\neq y_n}
    }
    {
    	\sum\limits_{n=1}^N \;\max\limits_{\|\delta\|_p\leq \epsilon} \Id_{c(x_n + \delta)\geq \tau}
    }\label{eq:conf-rte}
\end{align}
with $c(x) = \max_k f_k(x)$ and $f(x)$ being the model's confidence and predicted class on example $x$, respectively. Essentially, this is the \textbf{test error on test examples that can be modified within the chosen threat model \emph{and} pass confidence thresholding}. For $\tau=0$ (\ie, all examples pass confidence thresholding) this reduces to the standard \RTE, comparable to related work. We stress that our adaptive attack in \eqnref{eq:conf-attack} directly maximizes the numerator of \eqnref{eq:conf-rte} by maximizing the confidence of classes not equal $y$. A (clean) \textbf{confidence-thresholded test error ($\TE(\tau)$)} is obtained similarly. In the following, if not stated otherwise, we report \emph{confidence-thresholded} \RTE and \TE as default and omit the confidence threshold $\tau$ for brevity.

\textbf{FPR and \RTE:}
FPR quantifies how well an adversary can perturb (correctly classified) examples while not being rejected. The confidence-thresholded \RTE is more conservative as it measures \emph{any} non-rejected error (adversarial or not). As result, \RTE implicitly includes FPR \emph{and} \TE. Therefore, we report only \RTE and include FPRs for all our experiments in the supplementary material.

\begin{figure}[t]
    \vspace*{0px}
    \centering
    
    \hspace*{-0.3cm}
    \begin{minipage}[t]{0.23\textwidth}
        \vspace*{0px}
        
        \centering
        \includegraphics[width=1.05\textwidth]{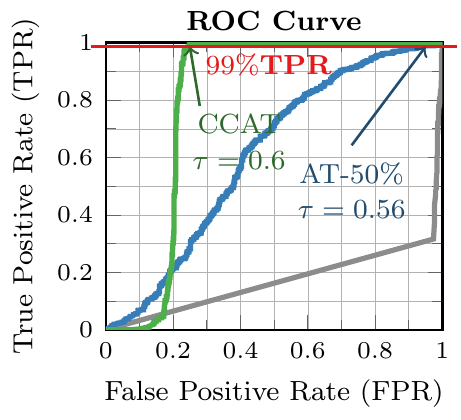}
    \end{minipage}
    \begin{minipage}[t]{0.25\textwidth}
        \vspace*{0px}
        
        \centering			
        \includegraphics[width=1.05\textwidth]{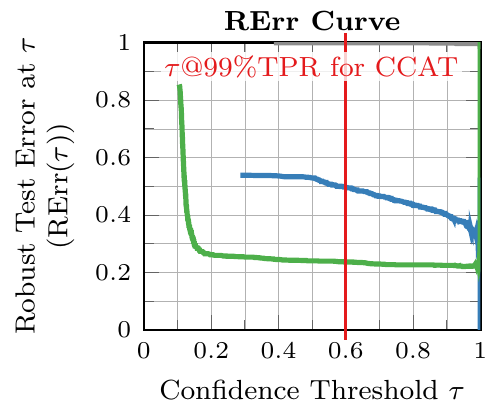}
    \end{minipage}\\
    
    \hspace*{-0.4cm}
    \begin{minipage}[t]{0.45\textwidth}
        \fbox{
            \hspace*{1.6cm}\includegraphics[width=0.575\textwidth]{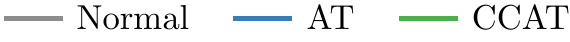}\hspace*{1.6cm}
        }
    \end{minipage}
    \vskip -4px
    \caption{\textbf{ROC and \RTE Curves.} On SVHN, we show ROC curves when distinguishing \emph{correctly classified} test examples from adversarial examples by confidence (left) and (confidence-thresholded) \RTE against confidence threshold $\tau$ (right) for worst-case adversarial examples across $L_\infty$ attacks with $\epsilon = 0.03$. The confidence threshold $\tau$ is chosen exclusively on correctly classified clean examples to obtain $99\%$TPR. For \ConfTrain, this results in $\tau \approx 0.6$. Note that \RTE subsumes both \TE and FPR.}
    \label{fig:experiments-evaluation}
    \vspace*{-2px}
\end{figure}

\textbf{Per-Example Worst-Case Evaluation:}
Instead of reporting average or per-attack results, we use a per-example \emph{worst-case} evaluation scheme: For each individual test example, all adversarial examples from all attacks (and restarts) are accumulated. Subsequently, \emph{per test example}, only the adversarial example with highest confidence is considered, resulting in a significantly stronger robustness evaluation compared to related work.
\section{Experiments}
\label{sec:experiments}

\begin{figure}[t]
    \vspace*{-2px}
    \centering
    
    \begin{minipage}[t]{0.225\textwidth}
        \vspace*{0px}
        
        \centering
        \includegraphics[width=1\textwidth]{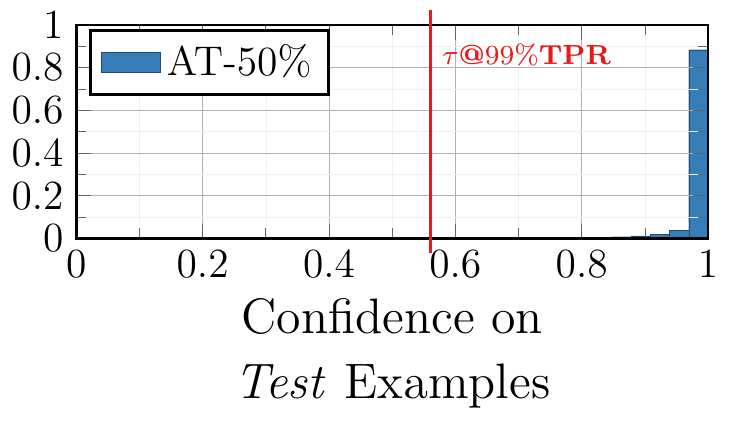}
        \vspace*{-14px}
        
        \includegraphics[width=1\textwidth]{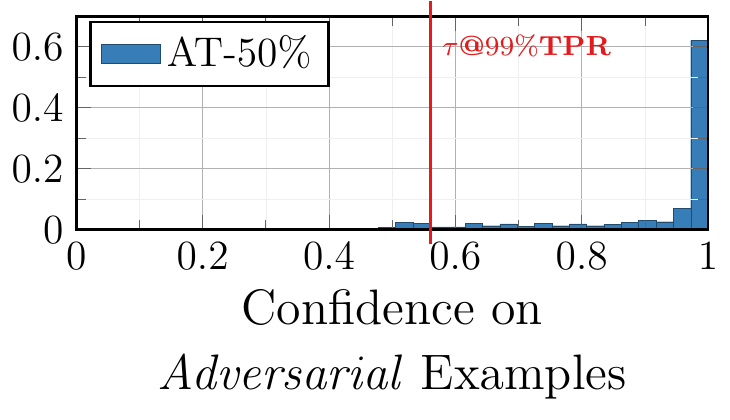}
    \end{minipage}
    \begin{minipage}[t]{0.225\textwidth}
        \vspace*{0px}
        
        \centering
        \includegraphics[width=1\textwidth]{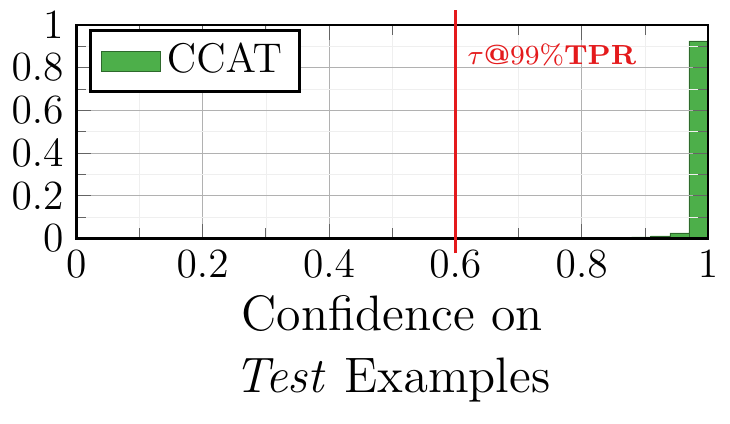}
        \vspace*{-14px}
        
        \includegraphics[width=1\textwidth]{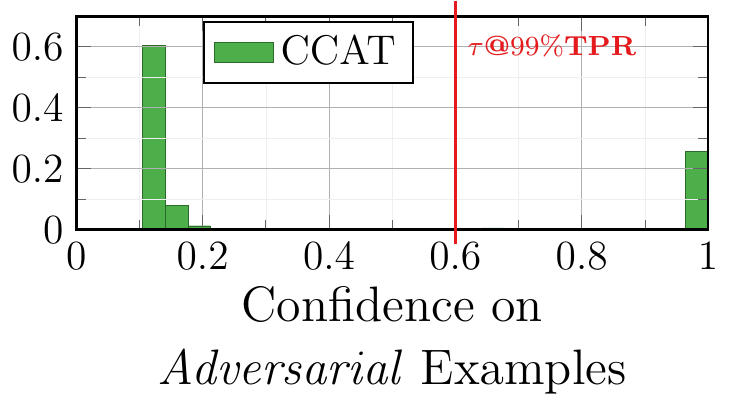}
    \end{minipage}
    
    \vspace*{-8px}
    \caption{\textbf{Confidence Histograms.} On SVHN, for \AdvTrain ($50\%$/$50\%$ adversarial training, left) and \ConfTrain (right), we show confidence histograms corresponding to \emph{correctly classified} test examples (top) and adversarial examples (bottom). We consider the worst-case adversarial examples across all $L_\infty$ attacks for $\epsilon = 0.03$. While the confidence of adversarial examples is reduced slightly for \AdvTrain, \ConfTrain is able to distinguish the majority of adversarial examples from (clean) test examples by confidence thresholding (in \textcolor{colorbrewer1}{red}).}
    \label{fig:experiments-histograms}
    \vspace*{-8px}
\end{figure}

We evaluate \ConfTrain in comparison with \AdvTrain \cite{MadryICLR2018} and related work \citep{MainiICML2020,ZhangICML2019} on MNIST \citep{LecunIEEE1998}, SVHN \citep{NetzerNIPS2011} and Cifar10 \citep{Krizhevsky2009} as well as MNIST-C \cite{MuICMLWORK2019} and Cifar10-C \cite{HendrycksICLR2019} with corrupted examples (\eg, blur, noise, compression, transforms \etc). We report \textbf{\textit{confidence-thresholded} test error (\TE; $\downarrow$ lower is better)} and \textbf{\textit{confidence-thresholded} robust test error (\RTE; $\downarrow$ lower is better)} for a confidence-threshold $\tau$ corresponding to $99\%$ true positive rate (TPR); we omit $\tau$ for brevity. We note that normal and standard adversarial training (\AdvTrain) are also allowed to reject examples by confidence thresholding. \TE is computed on $9000$ test examples. \RTE is computed on $1000$ test examples. The confidence threshold $\tau$ depends \emph{only} on correctly classified clean examples and is fixed at $99\%$TPR on the held-out \emph{last} $1000$ test examples.

\textbf{Attacks:}
For thorough evaluation, we consider $7$ different $L_p$ attacks for $p \in \{\infty, 2, 1, 0\}$. As white-box attacks, we use \PGD to maximize the objectives \eqnref{eq:attack} and \eqref{eq:conf-attack}, referred to as \PGD-\FCE and \PGD-\FConf. We use $T = 1000$ iterations and $10$ random restarts with random initialization plus one restart with zero initialization for \PGD-\FConf, and $T = 200$ with $50$ random restarts for \PGD-\FCE. For $L_\infty$, $L_2$, $L_1$ and $L_0$ attacks, we set \textbf{$\boldsymbol{\epsilon}$ to $\boldsymbol{0.3, 3, 18, 15}$ (MNIST) or $\boldsymbol{0.03, 2, 24, 10}$ (SVHN/Cifar10)}.

As black-box attacks, we additionally use the Query Limited (QL) attack \cite{IlyasICML2018} adapted with momentum and backtracking for $T = 1000$ iterations with $10$ restarts, the Simple attack \citep{NarodytskaCVPRWORK2017} for $T = 1000$ iterations and $10$ restarts, and the Square attack ($L_\infty$ and $L_2$) \citep{AndriushchenkoARXIV2019} with $T = 5000$ iterations. In the case of $L_0$ we also use Corner Search (CS) \citep{CroceICCV2019}. For all $L_p$, $p \in \{\infty, 2, 1, 0\}$, we consider $5000$ (uniform) random samples from the $L_p$-ball and the Geometry attack \cite{KhouryARXIV2018}. Except for CS, all black-box attacks use \eqnref{eq:conf-attack} as objective:

{\footnotesize
\vskip -2px
\begin{tabular}{| l | l | c | c |}
	\hline
	Attack & Objective & $T$ & Restarts\\\hline\hline
	\PGD-\FCE & \eqnref{eq:attack}, random init. & 200 & 50\\
	\PGD-\FConf & \eqnref{eq:conf-attack}, zero + random init. & 1000 & 11\\
	QL$^{\boldsymbol{\dagger}}$ & \eqnref{eq:conf-attack}, zero + random init. & 1000 & 11\\
	Simple$^{\boldsymbol{\dagger}}$ & \eqnref{eq:conf-attack} & 1000 & 10\\
	Square$^{\boldsymbol{\dagger}}$ & \eqnref{eq:conf-attack}, $L_\infty$, $L_2$ only & 5000 & 1\\
	CS$^{\boldsymbol{\dagger}}$ & \eqnref{eq:attack}, $L_0$ only & 200 & 1\\
	Geometry$^{\boldsymbol{\dagger}}$ & \eqnref{eq:conf-attack} & 1000 & 1\\
	Random$^{\boldsymbol{\dagger}}$ & \eqnref{eq:conf-attack} & -- & 5000\\
	\hline
    \multicolumn{4}{l}{\scriptsize $\boldsymbol{\dagger}$ Black-box attacks.}
\end{tabular}
\vskip -2px
}

Additionally, we consider adversarial frames and distal adversarial examples: Adversarial frames \cite{ZajaxAAAIWORK2019} allow a $2$ (MNIST) or $3$ (SVHN/Cifar10) pixel border to be manipulated arbitrarily within $[0,1]$ to maximize \eqnref{eq:conf-attack} using \PGD. Distal adversarial examples start with a (uniform) random image and use \PGD to maximize \eqref{eq:conf-attack} within a $L_\infty$-ball of size $\epsilon = 0.3$ (MNIST) or $\epsilon = 0.03$ (SVHN/Cifar10).

\begin{table}[t]
    \centering
    \footnotesize
    \begin{tabularx}{0.45\textwidth}{| X |c|c|c|c|c|c|}
    \hline
    & \multicolumn{6}{c |}{\textbf{SVHN:} \RTE@$99\%$TPR, $L_\infty$, $\epsilon = 0.03$}\\
    \hline
    & \begin{tabular}{@{}c@{}}worst\\case\end{tabular} & \multicolumn{5}{c |}{\begin{tabular}{@{}c@{}}top-$5$ attacks/restarts\\out of $7$ attacks with $84$ restarts\end{tabular}}\\
    \hline
    \hline
    \AdvTrainHalf & 56.0 & 52.1 & 52.0 & 51.9 & 51.6 & 51.4\\
    \ConfTrain & 39.1 & 23.6 & 13.7 & 13.6 & 12.6 & 12.5\\
    \hline
\end{tabularx}
    \caption{\textbf{Per-Example Worst-Case Evaluation.} We compare confidence-thresholded \RTE with $\tau$@$99\%$TPR for the per-example worst-case and the top-$5$ individual attacks/restarts among $7$ attacks with $84$ restarts in total. Multiple restarts are \emph{necessary} to effectively attack \ConfTrain, while a single attack and restart is nearly sufficient against \AdvTrainHalf. This demonstrates that \ConfTrain is more difficult to ``crack''.}
    \label{tab:experiments-per-attack}
\end{table}

\textbf{Training:}
We train $50\%$/$50\%$ \AdvTrain (\AdvTrainHalf) and \ConfTrain as well as $100\%$ \AdvTrain (\AdvTrainFull) with $L_\infty$ attacks using $T = 40$ iterations for \PGD-\FCE and \PGD-\FConf, respectively, and $\epsilon = 0.3$ (MNIST) or $\epsilon = 0.03$ (SVHN/Cifar10). We use ResNet-20 \citep{HeCVPR2016}, implemented in PyTorch \citep{PaszkeNIPSWORK2017}, trained using stochastic gradient descent. For \ConfTrain, we use $\rho = 10$.

\textbf{Baselines:}
We compare to multi-steepest descent (\Wong) adversarial training \cite{MainiICML2020} using the pre-trained LeNet on MNIST and pre-activation ResNet-18 on Cifar10 trained with $L_\infty$, $L_2$ and $L_1$ adversarial examples and $\epsilon$ set to $0.3, 1.5, 12$ and $0.03, 0.5, 12$, respectively. The $L_2$ and $L_1$ attacks in \tabref{tab:experiments-main} (larger $\epsilon$) are unseen. For \TRADES \cite{ZhangICML2019}, we use the pre-trained convolutional network \cite{CarliniSP2017} on MNIST and WRN-10-28 \cite{ZagoruykoBMVC2016} on Cifar10, trained on $L_\infty$ adversarial examples with $\epsilon=0.3$ and $\epsilon=0.03$, respectively. On Cifar10, we further consider the pre-trained ResNet-50 of \cite{MadryICLR2018} (\MadryAT, $L_\infty$ adversarial examples with $\epsilon = 0.03$). We also consider the Mahalanobis (\Lee) \cite{LeeNIPS2018} and local intrinsic dimensionality (\Ma) detectors \cite{MaICLR2018} using the provided pre-trained ResNet-34 on SVHN/Cifar10.

\begin{figure}[t]
    \vspace*{-1px}
    
    \begin{minipage}[t]{0.235\textwidth}
        \vspace*{-4px}
        
        \includegraphics[width=1\textwidth]{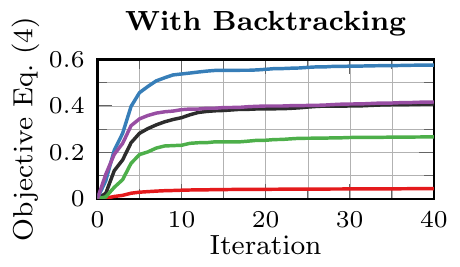}
    \end{minipage}
    \begin{minipage}[t]{0.235\textwidth}
        \vspace*{-4px}
        
        \includegraphics[width=1\textwidth]{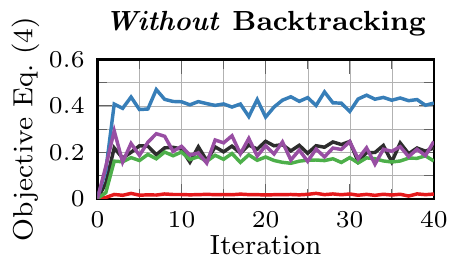}
    \end{minipage}
    \\
    \fbox{
        \hspace*{0.9cm}\includegraphics[width=0.35\textwidth]{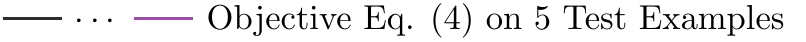}\hspace*{0.9cm}
    }
    \vspace*{-14px}
    \caption{\textbf{Backtracking.} Our $L_\infty$ \PGD-\FConf attack, \ie, \PGD maximizing \eqnref{eq:conf-attack}, using $40$ iterations with momentum and our developed backtracking scheme (left) and without both (right) on SVHN. We plot \eqnref{eq:conf-attack} over iterations for the first $5$ test examples corresponding to different colors. Backtracking avoids oscillation and obtains higher overall objective values within the same number of iterations.}
    \label{fig:experiments-momback}
    \vspace*{-4px}
\end{figure}

\subsection{Ablation Study}
\label{subsec:experiments-ablation}

\textbf{Evaluation Metrics:}
\figref{fig:experiments-evaluation} shows ROC curves, \ie, how well adversarial examples can be rejected by confidence. As marked in \textcolor{colorbrewer1}{red}, we are only interested in the FPR for the conservative choice of $99\%$TPR, yielding the confidence threshold $\tau$. The \RTE curves highlight how robustness is influenced by the threshold: \AdvTrain also benefits from a reject option, however, not as much as \ConfTrain which has been explicitly designed for rejecting adversarial examples.

\begin{table*}[t]
	\footnotesize
	\hspace*{-0.2cm}
	\begin{subfigure}{0.82\textwidth}
		\centering
		\begin{tabularx}{1\textwidth}{|X|@{\hskip 7px}c@{\hskip 7px}|@{\hskip 3px}c@{\hskip 3px}|@{\hskip 7px}c@{\hskip 7px}|@{\hskip 7px}c@{\hskip 7px}|@{\hskip 7px}c@{\hskip 7px}|@{\hskip 7px}c@{\hskip 7px}|@{\hskip 7px}c@{\hskip 7px}|@{\hskip 7px}c@{\hskip 7px}|}
\hline
\textbf{MNIST:} & \multicolumn{2}{c|@{\hskip 7px}}{\textbf{\TE} $\downarrow$ in \%} & \multicolumn{6}{@{\hskip 7px}c@{\hskip 7px}|}{\emph{confidence-thresholed} \textbf{\RTE} $\downarrow$ for $\tau$@$99\%$TPR}\\
\hline
& \begin{tabular}{@{}c@{}}(clean)\\$\tau = 0$\end{tabular}
& \begin{tabular}{@{}c@{}}(clean)\\$99\%$TPR\end{tabular}
& \begin{tabular}{@{\hskip 2.5px}c@{\hskip 2.5px}}$L_\infty$\\$\epsilon = 0.3$\end{tabular}
& \begin{tabular}{@{\hskip 2.5px}c@{\hskip 2.5px}}$L_\infty$\\$\epsilon = 0.4$\end{tabular}
& \begin{tabular}{@{}c@{}}$L_2$\\$\epsilon = 3$\end{tabular}
& \begin{tabular}{@{}c@{}}$L_1$\\$\epsilon = 18$\end{tabular}
& \begin{tabular}{@{}c@{}}$L_0$\\$\epsilon = 15$\end{tabular}
& \begin{tabular}{@{}c@{}}adv.\\frames\end{tabular}
\\\hline
& (seen)
& (seen)
& \textcolor{colorbrewer3}{seen}
& \textcolor{colorbrewer1}{\bfseries unseen}
& \textcolor{colorbrewer1}{\bfseries unseen}
& \textcolor{colorbrewer1}{\bfseries unseen}
& \textcolor{colorbrewer1}{\bfseries unseen}
& \textcolor{colorbrewer1}{\bfseries unseen}
\\\hline\hline
\Normal & 0.4 & 0.1
& 100.0 
& 100.0 
& 100.0 
& 100.0 
& 92.3 
& 87.7 
\\
\AdvTrainHalf & 0.5 & \bfseries 0.0
& \bfseries 1.7 
& 100.0 
& 81.5 
& 24.6 
& 23.9 
& 73.7 
\\
\AdvTrainFull & 0.5 & \bfseries 0.0
& \bfseries 1.7 
& 100.0 
& 84.8 
& 21.3 
& \bfseries 13.9 
& 62.3 
\\
\ConfTrain & \bfseries 0.3 & 0.1
& 7.4 
& \bfseries 11.9 
& \bfseries 0.3 
& \bfseries 1.8 
& 14.8 
& \bfseries 0.2 
\\
\hline
\end{tabularx}
	\end{subfigure}
	\begin{subfigure}{0.085\textwidth}
		\centering
		\begin{tabular}{|@{\hskip 7px}c@{\hskip 7px}|}
\hline
\textbf{FPR} $\downarrow$\\
\hline
\begin{tabular}{@{}c@{}}distal\\\vphantom{t}\end{tabular}\\
\hline
\textcolor{colorbrewer1}{\bfseries unseen}\\
\hline
\hline
100.0 
\\
100.0 
\\
100.0 
\\
\bfseries 0.0 
\\\hline\hline
100.0 
\\
100.0 
\\\hline
\end{tabular}

	\end{subfigure}
	\begin{subfigure}{0.08\textwidth}
		\centering
		\begin{tabular}{|c|}
\hline
\textbf{\TE} $\downarrow$\\
\hline
\begin{tabular}{@{}c@{}}corrupted\\{MNIST-C}\end{tabular}\\
\hline
\textcolor{colorbrewer1}{\bfseries unseen}\\
\hline
\hline
32.8\\ 
12.6\\ 
17.6\\ 
\bfseries 5.7\\ 
\hline\hline
6.0\\ 
7.9\\ 
\hline
\end{tabular}

	\end{subfigure}
	\vskip 2px
	\hspace*{-0.2cm}
	\begin{subfigure}[t]{0.82\textwidth}
		\centering
		\vspace*{0px}
		
		\begin{tabularx}{1\textwidth}{|X|@{\hskip 7px}c@{\hskip 7px}|@{\hskip 3px}c@{\hskip 3px}|@{\hskip 7px}c@{\hskip 7px}|@{\hskip 7px}c@{\hskip 7px}|@{\hskip 7px}c@{\hskip 7px}|@{\hskip 7px}c@{\hskip 7px}|@{\hskip 7px}c@{\hskip 7px}|@{\hskip 7px}c@{\hskip 7px}|}
\hline
\textbf{SVHN:} & \multicolumn{2}{c|@{\hskip 7px}}{\textbf{\TE} $\downarrow$ in \%} & \multicolumn{6}{@{\hskip 7px}c@{\hskip 7px}|}{\emph{confidence-thresholed} \textbf{\RTE} $\downarrow$ for $\tau$@$99\%$TPR}\\
\hline
& \begin{tabular}{@{}c@{}}(clean)\\$\tau = 0$\end{tabular}
& \begin{tabular}{@{}c@{}}(clean)\\$99\%$TPR\end{tabular}
& \begin{tabular}{@{}c@{}}$L_\infty$\\$\epsilon = 0.03$\end{tabular}
& \begin{tabular}{@{}c@{}}$L_\infty$\\$\epsilon = 0.06$\end{tabular}
& \begin{tabular}{@{}c@{}}$L_2$\\$\epsilon = 2$\end{tabular}
& \begin{tabular}{@{}c@{}}$L_1$\\$\epsilon = 24$\end{tabular}
& \begin{tabular}{@{}c@{}}$L_0$\\$\epsilon = 10$\end{tabular}
& \begin{tabular}{@{}c@{}}adv.\\frames\end{tabular}
\\\hline
& (seen)
& (seen)
& \textcolor{colorbrewer3}{seen}
& \textcolor{colorbrewer1}{\bfseries unseen}
& \textcolor{colorbrewer1}{\bfseries unseen}
& \textcolor{colorbrewer1}{\bfseries unseen}
& \textcolor{colorbrewer1}{\bfseries unseen}
& \textcolor{colorbrewer1}{\bfseries unseen}
\\\hline\hline
\Normal & 3.6 & 2.6
& 99.9 
& 100.0 
& 100.0 
& 100.0 
& 83.7 
& 78.7 
\\
\AdvTrainHalf & 3.4 & 2.5
& 56.0 
& 88.4 
& 99.4 
& 99.5 
& 73.6 
& 33.6 
\\
\AdvTrainFull & 5.9 & 4.6
& 48.3 
& 87.1 
& 99.5 
& 99.8 
& 89.4 
& 26.0 
\\
\ConfTrain & \bfseries 2.9 & \bfseries 2.1
& \bfseries 39.1 
& \bfseries 53.1 
& \bfseries 29.0 
& \bfseries 31.7 
& \bfseries 3.5 
& \bfseries 3.7 
\\
\hline\hline
\textbf{*} \Ma & 3.3 & 2.2
& 91.0
& 93.1
& 92.2
& 90.0
& 41.6
& 89.8
\\
\hline
\end{tabularx}
	\end{subfigure}
	\begin{subfigure}[t]{0.08\textwidth}
		\centering
		\vspace*{0px}
		
		\begin{tabular}{|@{\hskip 7px}c@{\hskip 7px}|}
\hline
\textbf{FPR} $\downarrow$\\
\hline
\begin{tabular}{@{}c@{}}distal\\\vphantom{t}\end{tabular}\\
\hline
\textcolor{colorbrewer1}{\bfseries unseen}\\
\hline
\hline
87.1 
\\
86.3 
\\
81.0 
\\
\bfseries 0.0 
\\\hline\hline
8.6\\
\bfseries 0.0\\
\hline
\end{tabular}

	\end{subfigure}
	\hfill
	\vskip 2px
	\hspace*{-0.2cm}
	\begin{subfigure}[t]{0.82\textwidth}
		\centering
		\vspace*{0px}
		
		\begin{tabularx}{1\textwidth}{|X|@{\hskip 7px}c@{\hskip 7px}|@{\hskip 3px}c@{\hskip 3px}|@{\hskip 7px}c@{\hskip 7px}|@{\hskip 7px}c@{\hskip 7px}|@{\hskip 7px}c@{\hskip 7px}|@{\hskip 7px}c@{\hskip 7px}|@{\hskip 7px}c@{\hskip 7px}|@{\hskip 7px}c@{\hskip 7px}|}
\hline
\textbf{CIFAR10:} & \multicolumn{2}{c|@{\hskip 7px}}{\textbf{\TE} $\downarrow$ in \%} & \multicolumn{6}{@{\hskip 7px}c@{\hskip 7px}|}{\emph{confidence-thresholed} \textbf{\RTE} $\downarrow$ for $\tau$@$99\%$TPR}\\
\hline
& \begin{tabular}{@{}c@{}}(clean)\\$\tau = 0$\end{tabular}
& \begin{tabular}{@{}c@{}}(clean)\\$99\%$TPR\end{tabular}
& \begin{tabular}{@{}c@{}}$L_\infty$\\$\epsilon = 0.03$\end{tabular}
& \begin{tabular}{@{}c@{}}$L_\infty$\\$\epsilon = 0.06$\end{tabular}
& \begin{tabular}{@{}c@{}}$L_2$\\$\epsilon = 2$\end{tabular}
& \begin{tabular}{@{}c@{}}$L_1$\\$\epsilon = 24$\end{tabular}
& \begin{tabular}{@{}c@{}}$L_0$\\$\epsilon = 10$\end{tabular}
& \begin{tabular}{@{}c@{}}adv.\\frames\end{tabular}
\\\hline
& (seen)
& (seen)
& \textcolor{colorbrewer3}{seen}
& \textcolor{colorbrewer1}{\bfseries unseen}
& \textcolor{colorbrewer1}{\bfseries unseen}
& \textcolor{colorbrewer1}{\bfseries unseen}
& \textcolor{colorbrewer1}{\bfseries unseen}
& \textcolor{colorbrewer1}{\bfseries unseen}
\\\hline\hline
\Normal & \underline{8.3} & 7.4
& 100.0 
& 100.0 
& 100.0 
& 100.0 
& 84.7 
& 96.7 
\\
\AdvTrainHalf & 16.6 & 15.5
& 62.7 
& 93.7 
& 98.4 
& 98.4 
& 74.4 
& 78.7 
\\
\AdvTrainFull & 19.4 & 18.3
& 59.9 
& 90.3 
& 98.3 
& 98.0 
& 72.3 
& 79.6 
\\
\ConfTrain & 10.1 & \underline{6.7}
& 68.4 
& 92.4 
& \bfseries 52.2 
& \bfseries 58.8 
& \bfseries 23.0 
& \bfseries 66.1 
\\
\hline
\end{tabularx}
	\end{subfigure}
	\begin{subfigure}[t]{0.085\textwidth}
		\centering
		\vspace*{0px}
		
		\begin{tabular}{|@{\hskip 7px}c@{\hskip 7px}|}
\hline
\textbf{FPR} $\downarrow$\\
\hline
\begin{tabular}{@{}c@{}}distal\\\vphantom{t}\end{tabular}\\
\hline
\textcolor{colorbrewer1}{\bfseries unseen}\\
\hline
\hline
83.3 
\\
75.0 
\\
72.5 
\\
\bfseries 0.0 
\\\hline\hline
76.7 
\\
76.2 
\\
78.5 
\\\hline\hline
0.1\\
2.4\\
\hline
\end{tabular}

	\end{subfigure}
	\begin{subfigure}[t]{0.08\textwidth}
		\centering
		\vspace*{0px}
		
		\begin{tabular}{|@{\hskip 7px}c@{\hskip 7px}|}
\hline
\textbf{\TE} $\downarrow$\\
\hline
\begin{tabular}{@{}c@{}}corrupted\\{\scriptsize CIFAR10-C}\end{tabular}\\
\hline
\textcolor{colorbrewer1}{\bfseries unseen}\\
\hline
\hline
12.3\\
16.2\\
19.6\\
\bfseries 8.5\\
\hline\hline
19.3\\
15.0\\
12.9\\
\hline\hline
11.59\\
12.4\\
\hline
\end{tabular}

	\end{subfigure}
	\vskip -6px
	\caption{\textbf{Main Results: Generalizing Robustness.} For $L_\infty$, $L_2$, $L_1$, $L_0$ attacks and adversarial frames, we report per-example worst-case (confidence-thresholded) \TE and \RTE at $99\%$TPR across all attacks; $\epsilon$ is reported in the corresponding columns. For distal adversarial examples and corrupted examples, we report FPR and \TE, respectively. $L_\infty$ attacks with $\epsilon{=}0.3$ on MNIST and $\epsilon = 0.03$ on SVHN/Cifar10 were used for training (\textcolor{colorbrewer3}{seen}). The remaining attacks were not encountered during training (\textbf{\textcolor{colorbrewer1}{unseen}}). \ConfTrain outperforms \AdvTrain and the other baselines regarding robustness against unseen attacks. FPRs included in the supplementary material. \textbf{*}~Pre-trained models with different architecture, \Ma/\Lee use the same model.}
	\label{tab:experiments-main}
    \vspace*{-2px}
\end{table*}

\textbf{Worst-Case Evaluation:} 
\tabref{tab:experiments-per-attack} illustrates the importance of worst-case evaluation on SVHN, showing that \ConfTrain is significantly ``harder'' to attack than \AdvTrain. We show the worst-case \RTE over all $L_\infty$ attacks as well as the top-$5$ individual attacks (each restart treated as separate attack). For \AdvTrainHalf, a single restart of \PGD-\FConf with $T=1000$ iterations is highly successful, with $52.1\%$ \RTE close to the overall worst-case of $56\%$. For \ConfTrain, in contrast, multiple restarts are crucial as the best individual attack, \PGD-\FConf with $T=1000$ iterations and zero initialization obtains only $23.6\%$ \RTE compared to the overall worst-case of $39.1\%$.

\textbf{Backtracking:}
\figref{fig:experiments-momback} illustrates the advantage of backtracking for \PGD-\FConf with $T{=}40$ iterations on $5$ test examples of SVHN. Backtracking results in better objective values and avoids oscillation, \ie, a stronger attack for training and testing. In addition, while $T{=}200$ iterations are sufficient against \AdvTrain, we needed up to $T{=}1000$ iterations for \ConfTrain.

\subsection{Main Results (Table  \ref{tab:experiments-main})}
\label{subsec:experiments-main}

\textbf{Robustness Against \textcolor{colorbrewer3}{seen} $L_\infty$ Attacks:}
Considering \tabref{tab:experiments-main} and $L_\infty$ adversarial examples as \textcolor{colorbrewer3}{seen} during training, \ConfTrain exhibits comparable robustness to \AdvTrain. With $7.4\%$/$67.9\%$ \RTE on MNIST/Cifar10, \ConfTrain lacks behind \AdvTrainHalf ($1.7\%$/$62.7\%$) only slightly. On SVHN, in contrast \ConfTrain outperforms \AdvTrainHalf and \AdvTrainFull significantly with $39.1\%$ vs. $56.0\%$ and $48.3\%$. We note that \ConfTrain and \AdvTrainHalf are trained on $50\%$ clean / $50\%$ adversarial examples. This is in contrast to \AdvTrainFull trained on $100\%$ adversarial examples, which improves robustness slightly, \eg, from $56\%$/$62.7\%$ to $48.3\%$/$59.9\%$ on SVHN/Cifar10.

\textbf{Robustness Against \textcolor{colorbrewer1}{unseen} $L_p$ Attacks:}
Regarding \textcolor{colorbrewer1}{unseen} attacks, \AdvTrain's robustness deteriorates quickly while \ConfTrain is able to generalize robustness to novel threat models. On SVHN, for example, \RTE of \AdvTrainHalf goes up to $88.4\%$, $99.4\%$, $99.5\%$ and $73.6\%$ for larger $L_\infty$, $L_2$, $L_1$ and $L_0$ attacks. In contrast, \ConfTrain's robustness generalizes to these unseen attacks significantly better, with $53.1\%$, $29\%$, $31.7\%$ and $3.5\%$, respectively. The results on MNIST and Cifar10 or for \AdvTrainFull tell a similar story. However, \AdvTrain generalizes better to $L_1$ and $L_0$ attacks on MNIST, possibly due to the large $L_\infty$-ball used during training ($\epsilon = 0.3$). Here, training purely on adversarial examples, \ie, \AdvTrainFull is beneficial. On Cifar10, \ConfTrain has more difficulties with large $L_\infty$ attacks ($\epsilon = 0.06$) with $92\%$ \RTE.
As detailed in the supplementary material, these observations are supported by considering FPR. \AdvTrain benefits from considering FPR as clean \TE is not taken into account. On Cifar10, for example, $47.6\%$ FPR compared to $62.7\%$ \RTE for \AdvTrainHalf. This is less pronounced for \ConfTrain due to the improved \TE compared to \AdvTrain. Overall, \ConfTrain improves robustness against arbitrary (unseen) $L_p$ attacks, demonstrating that \ConfTrain indeed extrapolates near-uniform predictions beyond the $L_\infty$ $\epsilon$-ball used during training.

\textbf{Comparison to \Wong and \TRADES:}
\TRADES is able to outperform \ConfTrain alongside \AdvTrain (including \MadryAT) on Cifar10 with respect to the $L_\infty$ adversarial examples \textcolor{colorbrewer3}{seen} during training: $43.5\%$ \RTE compared to $68.4\%$ for \ConfTrain. This might be a result of training on $100\%$ adversarial examples and using more complex models: \TRADES uses a WRN-10-28 with roughly $46.1\text{M}$ weighs, in contrast to our ResNet-18 with $4.3\text{M}$ (and ResNet-20 with $11.1\text{M}$ for \Wong). However, regarding \textcolor{colorbrewer1}{unseen} $L_2$, $L_1$ and $L_0$ attacks, \ConfTrain outperforms \TRADES with $52.2\%$, $58.8\%$ and $23\%$ compared to $70.9\%$, $96.9\%$ and $36.9\%$ in terms of \RTE. Similarly, \ConfTrain outperforms \Wong. This is surprising, as \Wong trains on both $L_2$ and $L_1$ attacks with smaller $\epsilon$, while \ConfTrain does not. Only against larger $L_\infty$ adversarial examples with $\epsilon = 0.06$, \TRADES reduces \RTE from $92.4\%$ (\ConfTrain) to $81\%$. Similar to \AdvTrain, \TRADES also generalizes better to $L_2$, $L_1$ or $L_0$ on MNIST, while \Wong is not able to compete. Overall, compared to \Wong and \TRADES, the robustness obtained by \ConfTrain generalizes better to previously unseen attacks. We also note that, on MNIST, \ConfTrain outperforms the robust Analysis-by-Synthesis (ABS) approach of \cite{SchottICLR2019} \wrt $L_\infty$, $L_2$, and $L_0$ attacks.

\textbf{Detection Baselines:}
The detection methods \Ma and \Lee are outperformed by \ConfTrain across all datasets and threat models. On SVHN, for example, \Lee obtains $73\%$ \RTE against the seen $L_\infty$ attacks and $79.5\%$, $78.1\%$, $67.5\%$ and $41.5\%$ \RTE for the unseen $L_\infty$, $L_2$, $L_1$ and $L_0$ attacks. \Ma is consistently outperformed by \Lee on SVHN. This is striking, as we \emph{only} used \PGD-\FCE and \PGD-\FConf to attack these approaches and emphasizes the importance of training \emph{adversarially} against an adaptive attack to successfully reject adversarial examples.

\textbf{Robustness Against Unconventional Attacks:}
Against adversarial frames, robustness of \AdvTrain reduces to $73.7\%$ /$62.3\%$ \RTE (AT-$50\%$/$100\%$), even on MNIST, while \ConfTrain achieves $0.2\%$. \Wong, in contrast, is able to preserve robustness better with $8.8\%$ \RTE, which might be due to the $L_2$ and $L_1$ attacks seen during training. \ConfTrain outperforms both approaches with $0.2\%$ \RTE, as does \TRADES. On SVHN and Cifar10, however, \ConfTrain outperforms all approaches, including \TRADES, considering adversarial frames.
Against distal adversarial examples, \ConfTrain outperforms all approaches significantly, with $0\%$ FPR, compared to the second-best of $72.5\%$ for \AdvTrainFull on Cifar10. Only the detection baselines \Ma and \Lee are competitive, reaching close to $0\%$ FPR.
This means that \ConfTrain is able to extrapolate low-confidence distributions to far-away regions of the input space.
Finally, we consider corrupted examples (\eg, blur, noise, transforms \etc) where \ConfTrain also improves results, \ie, mean \TE across all corruptions. On Cifar10-C, for example, \ConfTrain achieves $8.5\%$ compared $12.9\%$ for \MadryAT and $12.3\%$ for normal training. On MNIST-C, only \Wong yields a comparably low \TE: $6\%$ vs. $5.7\%$ for \ConfTrain.

\textbf{Improved Test Error:}
\ConfTrain also outperforms \AdvTrain regarding \TE, coming close to that of normal training. On all datasets, \emph{confidence-thresholded} \TE for \ConfTrain is better or equal than that of normal training. On Cifar10, only \Ma/\Lee achieve a better standard and confidence-thresholded \TE using a ResNet-34 compared to our ResNet-20 for \ConfTrain ($21.2\text{M}$ vs. $4.3\text{M}$ weights). In total the performance of \ConfTrain shows that the robustness-generalization trade-off can be improved significantly.

Our \textbf{supplementary material} includes detailed descriptions of our \PGD-\FConf attack (including pseudo-code), a discussion of our confidence-thresholded \RTE, and more details regarding baselines. We also include results for confidence thresholds at $98\%$ and $95\%$TPR, which improves results only slightly, at the cost of ``throwing away'' significantly more clean examples. Furthermore, we provide ablation studies, qualitative examples and per-attack results.

\section{Conclusion}
\label{sec:conclusion}

Adversarial training results in robust models against the threat model \emph{seen} during training, \eg, $L_\infty$ adversarial examples. However, generalization to \emph{unseen} attacks such as other $L_p$ adversarial examples or larger $L_\infty$ perturbations is insufficient. We propose \textbf{confidence-calibrated adversarial training (\ConfTrain)} which biases the model towards low confidence predictions on adversarial examples and beyond. Then, adversarial examples can easily be rejected based on their confidence. Trained exclusively on $L_\infty$ adversarial examples, \ConfTrain improves robustness against unseen threat models such as larger $L_\infty$, $L_2$, $L_1$ and $L_0$ adversarial examples, adversarial frames, distal adversarial examples and corrupted examples. Additionally, accuracy is improved in comparison to adversarial training. We thoroughly evaluated \ConfTrain using $7$ different white-and black-box attacks with up to $50$ random restarts and $5000$ iterations. These attacks where adapted to \ConfTrain by directly maximizing confidence. We reported worst-case robust test error, extended to our confidence-thresholded setting, across \emph{all} attacks.

\bibliography{bibliography}
\bibliographystyle{icml2020}

\onecolumn
\begin{appendix}
	\icmltitle{Supplementary Material for\\Confidence-Calibrated Adversarial Training: Generalizing to Unseen Attacks}

\vskip 0.3in

\begin{abstract}
    This document provides supplementary material for \textbf{confidence-calibrated adversarial training (\ConfTrain)}. First, in \secref{sec:supp-proof}, we provide the proof of Proposition 1, showing that there exist problems where standard adversarial training (\AdvTrain) is unable to reconcile robustness and accuracy, while \ConfTrain is able to obtain \emph{both} robustness \emph{and} accuracy. In \secref{sec:supp-experiments}, to promote reproducibility and emphasize our thorough evaluation, we discuss details regarding the used attacks, training procedure, baselines and evaluation metrics. Furthermore, \secref{sec:supp-experiments} includes additional experimental results in support of the observations in the main paper. For example, we present results for confidence threshold at $95\%$ and $98\%$ true positive rate (TPR), results for the evaluated detection baselines as well as per-attack and per-corruption results for in-depth analysis. We also include qualitative results highlighting how \ConfTrain obtains robustness through confidence thresholding. Code and pre-trained models are available at \href{http://davidstutz.de/ccat}{davidstutz.de/ccat}.
\end{abstract}

\section{Introduction}

\textbf{Confidence-calibrated adversarial training (\ConfTrain)} biases the network towards low-confidence predictions on adversarial examples. This is achieved by training the network to predict a uniform distribution between (correct) one-hot and uniform distribution which becomes more uniform as the distance to the attacked example increases. In the main paper, we show that \ConfTrain addresses two problems of standard adversarial training (\AdvTrain) as, \eg, proposed in \cite{MadryICLR2018}: the poor generalization of robustness to attacks not employed during training, \eg, other $L_p$ attacks or larger perturbations, and the reduced accuracy. We show that \ConfTrain, trained only on $L_\infty$ adversarial examples, improves robustness against previously unseen attacks through confidence thresholding, \ie, rejecting low-confidence (adversarial) examples. Furthermore, we demonstrate that \ConfTrain is able to improve accuracy compared to adversarial training. In this document, \secref{sec:supp-proof} provides the proof of Proposition 1. Then, \secref{sec:supp-experiments} includes details on our experimental setup, emphasizing our efforts to thoroughly evaluate \ConfTrain, and additional experimental results allowing an in-depth analysis of the robustness obtained through \ConfTrain.

In \secref{sec:supp-proof}, corresponding to the proof of Proposition 1, we show that there exist problems where standard adversarial training is indeed unable to reconcile robustness and accuracy. \ConfTrain, in contrast, is able to obtain \emph{both} robustness and accuracy, given that the ``transition'' between one-hot and uniform distribution used during training is chosen appropriately.

In \secref{sec:supp-experiments}, we discuss our thorough experimental setup to facilitate reproducibility and present additional experimental results in support of the conclusions of the main paper. In the first part of \secref{sec:supp-experiments}, starting with \secref{subsec:supp-experiments-attacks}, we provide a detailed description of the employed projected gradient descent (\PGD) attack with momentum and backtracking, including pseudo-code and used hyper-parameters. Similarly, we discuss details of the used black-box attacks. In \secref{subsec:supp-experiments-training}, we include details on our training procedure, especially for \ConfTrain. In \secref{subsec:supp-experiments-baselines}, we discuss the evaluated baselines, \ie, \cite{MainiICML2020,MadryICLR2018,ZhangICML2019,LeeNIPS2018,MaICLR2018}. Then, in \secref{subsec:supp-experiments-evaluation}, we discuss the employed evaluation metrics, focusing on our \emph{confidence-thresholded} robust test error (\RTE). In the second part of \secref{sec:supp-experiments}, starting with \secref{subsec:supp-experiments-ablation}, we perform ablation studies considering our attack and \ConfTrain. Regarding the attack, we demonstrate the importance of enough iterations, backtracking and appropriate initialization to successfully attack \ConfTrain. Regarding \ConfTrain, we consider various values for the hyper-parameter $\rho$ which controls the transition from one-hot to uniform distribution during training. In \secref{subsec:supp-experiments-analysis}, we analyze how \ConfTrain achieves robustness by considering its behavior in adversarial directions as well as in between clean examples. Finally, we provide additional experimental results in \secref{subsec:supp-experiments-results}: our main results, \ie, robustness against seen and unseen adversarial examples, for a confidence threshold at $95\%$ and $98\%$ true positive rate (TPR), per-attack results on all datasets and per-corruption results on MNIST-C \citep{MuICMLWORK2019} and Cifar10-C \citep{HendrycksARXIV2019}.
	\section{Proof of Proposition 1}
\label{sec:supp-proof}

Adversarial training usually results in reduced accuracy on clean examples. In practice, training on $50\%$ clean and  $50\%$ adversarial examples instead of training \emph{only} on adversarial examples allows to control the robustness-accuracy trade-off to some extent, \ie, increase accuracy while sacrificing robustness. The following proposition shows that there exist problems where adversarial training is unable to reconcile robustness and accuracy, while \ConfTrain is able to obtain \emph{both}:

\begin{proposition}\label{prop:toy-example}
    We consider a classification problem with two points $x=0$ and $x=\epsilon$ in $\mR$ with deterministic labels, \ie,
    $p(y=2|x=0)=1$ and $p(y=1|x=\epsilon)=1$, such that the problem is fully determined by the probability $p_0=p(x=0)$. 
    The Bayes error of this classification problem is zero. Let the predicted probability distribution over classes be $\tilde{p}(y|x)=\frac{e^{g_y(x)}}{e^{g_1(x)}+e^{g_2(x)}}$, where $g:\R^d \rightarrow \R^2$ is the classifier and we assume that the function $\lambda:\mR_+ \rightarrow [0,1]$ used in \ConfTrain is monotonically decreasing and $\lambda(0)=1$. Then, the error of the Bayes optimal classifier (with cross-entropy loss) for
    \vspace*{-8px}
    \begin{itemize}
        \item adversarial training on $100\%$ adversarial examples
        is $\min\{p_0,1-p_0\}$.
        \vspace*{-3px}
        \item adversarial training on $50\%$/$50\%$ adversarial/clean examples per batch
        is $\min\{p_0,1 - p_0\}$.
        \vspace*{-3px}
        \item \ConfTrain on $50\%$ clean and $50\%$ adversarial examples 
        is \emph{zero} 
        if $\lambda(\epsilon) < \min\left\{\nicefrac{p_0}{1-p_0},\nicefrac{1-p_0}{p_0}\right\}$.
        \vspace*{-6px}
    \end{itemize}
\end{proposition}

\begin{proof}
    First, we stress that we are dealing with three different probability distributions over the labels:
	the true one $\Pr(y|x)$, the imposed one during training $\hat{p}(y|x)$ and the predicted one $\tilde{p}(y|x)$. We also note that $\hat{p}$ depends on $\lambda$ as follows:
    \begin{align}
        \hat{p}(k) = \lambda p_y(k) + (1 - \lambda) u(k)
    \end{align}
    where $p_y(k)$ is the original one-hot distribution, \ie, $p_y(k) = 1$ iff $k = y$ and $p_y(k) = 0$ otherwise with $y$ being the true label, and $u(k) = \nicefrac{1}{K}$ is the uniform distribution. We note that this is merely an alternative formulation to the target distribution as outlined in the main paper.
    Also note that $\lambda$ itself is a function of the norm $\|\delta\|$; here, this dependence is made explicit by writing $\hat{p}(\lambda)(y|x)$.
	This makes the expressions for the expected loss of \ConfTrain slightly more complicated. We first derive the Bayes optimal classifier and its loss for \ConfTrain.
    We introduce
    \begin{align}
        a = g_1(0)-g_2(0), \quad b=g_1(\epsilon)-g_2(\epsilon). \label{eq:supp-a-b}
    \end{align}
    and express the logarithm of the predicted probabilities (confidences) of class $1$ and $2$ in terms of these quantities.
	\begin{align*}
		\hspace*{-10px}
		-\log \tilde{p}(y=2|x=x)&= -\log\Big(\frac{e^{g_2(x)}}{e^{g_1(x)} + e^{g_2(x)}}\Big) = \log\big(1 + e^{g_1(x)-g_2(x)}\big)
     	=\begin{cases} \log\big(1+e^a\big) & \textrm{ if } x=0\\ \log(1+e^b) & \textrm{ if } x=\epsilon \end{cases}.\\
    	\hspace*{-10px}
    	-\log \tilde{p}(y=1|x=x)&=-\log\Big(\frac{e^{g_1(x)}}{e^{g_1(x)} + e^{g_2(x)}}\Big) = \log\big(1 + e^{g_2(x)-g_1(x)}\big) = 
    	\begin{cases} \log(1+e^{-a}) & \textrm{ if } x=0\\ \log\big(1+e^{-b}\big) & \textrm{ if } x=\epsilon\end{cases}.
    \end{align*}
    
    We consider the approach by \cite{MadryICLR2018} with $100\%$ adversarial training. The expected loss can be written as
	\begin{align*}
		&\Exp\Big[\max_{\norm{\delta}_\infty\leq \epsilon} L(y,g(x+\delta)\Big]= \Exp\Big[\Exp\Big[\max_{\norm{\delta}_\infty\leq \epsilon} L(y,g(x+\delta)|x\Big]\Big]\\
		=&\Pr(x=0)\Pr(y=2|x=0)\max\left\{-\log\big(\tilde{p}(y=2|x=0)\big),-\log\big(\tilde{p}(y=2|x=\epsilon)\big)\right\}\\
		&+(1-\Pr(x=0))\Pr(y=1|x=\epsilon)\max\left\{-\log\big(\tilde{p}(y=1|x=0)\big),-\log\big(\tilde{p}(y=1|x=\epsilon)\big)\right\}\\
		=&\Pr(x=0)\max\left\{-\log\big(\tilde{p}(y=2|x=0)\big),-\log\big(\tilde{p}(y=2|x=\epsilon)\big)\right\}\\
	 	&+ (1-\Pr(x=0))\max\left\{-\log\big(\tilde{p}(y=1|x=0)\big),-\log\big(\tilde{p}(y=1|x=\epsilon)\right\}
	\end{align*}
	This yields in terms of the parameters $a,b$ the expected loss:
    \begin{align*}
    	L(a,b)=&  \max\Big\{ \log(1+e^a),\log(1+e^b)\Big\}p_0 + \max\Big\{ \log(1+e^{-a}),\log(1+e^{-b})\Big\}(1-p_0)
    \end{align*}
   	The expected loss is minimized if $a=b$ as then both maxima are minimal. This results in the expected loss
   	\begin{align*}
    	L(a)= \log(1+e^a)p_0 + \log(1+e^{-a}) (1-p_0).
    \end{align*}
    The critical point is attained at $a^*=b^*=\log\Big(\frac{1-p_0}{p_0}\Big)$. 
    Thus
    \begin{align*}
        a^*=b^*=\begin{cases} >0 & \textrm{ if } p_0<\frac{1}{2},\\ <0 & \textrm{ if } p_0>\frac{1}{2}.\end{cases}
    \end{align*}
    Thus, we classify $x=0$ correctly, if $p_0>\frac{1}{2}$ and $x=\epsilon$ correctly if $p_0<\frac{1}{2}$. As result, the error of $100\%$ adversarial training is given by
    $\min\{p_0,1-p_0\}$ whereas the Bayes optimal error is zero as the problem is deterministic.
    
    Next we consider $50\%$ adversarial plus $50\%$ clean training. The expected loss 
	\begin{align*}
		\Exp\Big[\max_{\norm{\delta}_\infty\leq \epsilon} L(y,g(x+\delta)\Big] + \Exp\Big[ L(y,g(x+\delta)\Big],
	\end{align*}
	can be written as
    \begin{align*}
	    L(a,b)=&  \max\Big\{ \log(1+e^a),\log(1+e^b)\Big\}p_0 \\
	           &+ \max\Big\{ \log(1+e^{-a}),\log(1+e^{-b})\Big\}(1-p_0)\\
	           &+ \log(1+e^a)p_0 + \log(1+e^{-b})(1-p_0)
    \end{align*}
    We make a case distinction. If $a\geq b$, then the loss reduces to
    \begin{align*}
	    L(a,b)=&    \log(1+e^a)p_0 + \log(1+e^{-b})(1-p_0) \\
	           & + \log(1+e^a)p_0 + \log(1+e^{-b})(1-p_0)\\
	           &\geq L(a,a)\\
	           &=2\log(1+e^a)p_0 + 2\log(1+e^{-a})(1-p_0)
    \end{align*}
    Solving for the critical point yields $a^*=\log\Big(\frac{1-p_0}{p_0}\Big)=b^*$. Next we consider the set $a\leq b$.  This yields the loss
    \begin{align*}
	    L(a,b)=& \log(1+e^b)p_0 + \log(1+e^{-a})(1-p_0)\\
	           &+ \log(1+e^a)p_0 + \log(1+e^{-b})(1-p_0)
    \end{align*}
    Solving for the critical point yields $a^*=\log\Big(\frac{1-p_0}{p_0}\Big)=b^*$ which fulfills $a\leq b$. Actually, it coincides with the solution found already for $100\%$ adversarial training and thus resulting error of $50\%$ adversarial, $50\%$ clean training is again $\min\{p_0,1-p_0\}$ which is not equal to the Bayes optimal error.
    
    For our confidence-calibrated adversarial training one first has to solve
	\begin{align*}
		\delta_x^*(j)=\argmax_{\norm{\delta}_\infty \leq \epsilon} \max\limits_{k\neq j} \tilde{p}(y=k\,|\,x+\delta).
	\end{align*}
	We get with the expressions above (note that we have a binary classification problem):
    \begin{align*}
	    \delta^*_0(1)=\argmax_{\norm{\delta}_\infty \leq \epsilon} \tilde{p}(y=2|{0}+\delta) = \begin{cases} 0 & \textrm{ if } a<b \\ \epsilon & \textrm{ else}.\end{cases}.\\
	    \delta^*_0(2)= \argmax_{\norm{\delta}_\infty \leq \epsilon} \tilde{p}(y=1|{0} +\delta) = \begin{cases} \epsilon & \textrm{ if }  a<b  \\ 0 & \textrm{ else}.\end{cases}.\\
			\delta^*_\epsilon(1)= \argmax_{\norm{\delta}_\infty \leq \epsilon} \tilde{p}(y=2|\epsilon+\delta) = \begin{cases} -\epsilon & \textrm{ if }  a<b  \\0 & \textrm{ else}.\end{cases}.\\
	    \delta^*_\epsilon(2)=\argmax_{\norm{\delta}_\infty \leq \epsilon} \tilde{p}(y=1|\epsilon +\delta) = \begin{cases} 0 & \textrm{ if }  a<b  \\ -\epsilon & \textrm{ else}.\end{cases}.
    \end{align*}
    Note that the imposed distribution $\hat{p}$ over the classes depends on the true label $y$ of $x$ and $\delta^*_x(y)$ and then $\lambda(\norm{\delta^*_x(y)}_\infty)$. Due to the simple structure of the problem it holds that $\norm{\delta^*_x(y)}_\infty$ is either $0$ or $\epsilon$.
    
	In \ConfTrain we use for $50\%$ of the batch the standard cross-entropy loss and for the other $50\%$ we use the following loss:
	\begin{align*}
		L\left(\hat{p}_y\big(\lambda\big(\norm{\delta^*_x(y)}\big)\big)(x),\tilde{p}(x)\right)= - \sum_{j=1}^2 \hat{p}_y\big(\lambda\big(\norm{\delta^*_x(y)}\big)\big)(y=j\,|\,x=x) \log\left(\tilde{p}\big(y=j\,|\,x=x+\delta^*_x(j)\big)\right).
	\end{align*}
	The corresponding expected loss is then given by
	\begin{align*}
		&\Exp\Big[ L\Big(\hat{p}_y\big(\lambda\big(\norm{\delta^*_x(y)}\big)\big)(x),\tilde{p}(x)\Big)\Big]
		=\Exp\Big[\,\Exp\Big[ L\Big(\hat{p}_y\big(\lambda\big(\norm{\delta^*_x(y)}\big)\big)(x),\tilde{p}(x)\Big) \Big| \,x\Big]\Big] \\
		=&\Pr(x=0)\,\Exp\Big[ L\Big(\hat{p}_y\big(\lambda\big(\norm{\delta^*_0(y)}\big)\big)(0),\tilde{p}(0)\Big)\Big|\,x=0\Big] + \Pr(x=\epsilon)\,\Exp\Big[ L\Big(\hat{p}_Y\big(\lambda\big(\norm{\delta^*_\epsilon(y)}\big)\big)(\epsilon),\tilde{p}(\epsilon)\Big)\Big|\,x=\epsilon\Big],
	\end{align*}
	where $\Pr(x=0)=p_0$ and $\Pr(x=\epsilon)=1-\Pr(x=0)=1-p_0$. With the true conditional probabilities $\Pr(y=s\,|\,x)$ we get 
	\begin{align*}
		&\Exp\Big[ L\Big(\hat{p}_y\big(\lambda(\delta)\big)(x),\tilde{p}(x)\Big)\Big|\,x\Big]
		=\sum_{s=1}^2 \Pr(y=s\,|\,x) \, L\Big(\hat{p}_s\big(\lambda\big(\norm{\delta^*_x(s)}\big)\big)(x),\tilde{p}(x)\Big)\\
		=&- \sum_{s=1}^2 \Pr(y=s\,|\,x) \sum_{j=1}^2 \hat{p}_s\big(\lambda\big(\norm{\delta^*_x(s)}\big)\big)(y=j\,|\,x) \log\big(\tilde{p}(y=j\,|\,x=x+\delta^*_x(s))\big)
	\end{align*}
	For our problem it holds with $\Pr(y=2\,|\,x=0)=\Pr(y=1\,|\,x=\epsilon) =1$ (by assumption). Thus,
	\begin{align*}
		\Exp\Big[ L\Big(\hat{p}_y\big(\lambda(\delta)\big)(x),\tilde{p}(x)\Big)\Big|x=0\Big]=-\sum_{j=1}^2 \hat{p}_2\big(\lambda\big(\norm{\delta^*_0(2)}\big)\big)(y=j\,|\,x=0) \log\big(\tilde{p}(y=j\,|\,x=x+\delta^*_0(2))\big)\\
		\Exp\Big[ L\Big(\hat{p}_y\big(\lambda(\delta)\big)(x),\tilde{p}(x))\Big|\,x=\epsilon\Big]=-\sum_{j=1}^2 \hat{p}_1\big(\lambda\big(\norm{\delta^*_\epsilon(1)}\big)\big)(y=j\,|\,x=\epsilon) \log\big(\tilde{p}(y=j\,|\,x=x+\delta^*_\epsilon(1))\big)
	\end{align*}
	As $\norm{\delta^*_x(y)}_\infty$ is either $0$ or $\epsilon$ and $\lambda(0)=1$ we use in the following for simplicity the notation $\lambda=\lambda(\norm{\epsilon}_\infty)$. Moreover, note that 
	\begin{align*}
		\hat{p}_y(\lambda)(y=j|x=x)=\begin{cases} \lambda + \frac{(1-\lambda)}{K} & \textrm{ if } y = j,\\ \frac{(1-\lambda)}{K} & \textrm{ else }\end{cases},
	\end{align*}
	where $K$ is the number of classes. Thus, $K=2$ in our example and we note that $\lambda + \frac{(1-\lambda)}{2}=\frac{1+\lambda}{2}$.
	
    With this we can write the total loss (remember that we have half normal cross-entropy loss and half the loss for the adversarial part with the modified ``labels'') as
    \begin{align*}
    	L(a,b)=&  p_0\Big[\log(1+e^a) \Id_{a\geq b} 
           + \Id_{a<b}\Big( \frac{(1+\lambda)}{2}\log(1+e^b)+\frac{(1-\lambda)}{2}\log(1+e^{-b})\Big)\Big]\\ 
					 &+ (1-p_0)\Big[\log(1+e^{-b})   \Id_{a\geq b} +\Id_{a<b}\Big( \frac{(1+\lambda)}{2}\log(1+e^{-a})+\frac{(1-\lambda)}{2}\log(1+e^a)\Big)\Big]\\
           &+ \log(1+e^a)p_0 + \log(1+e^{-b})(1-p_0),
    \end{align*}
    where we have omitted a global factor $\frac{1}{2}$ for better readability (the last row is the cross-entropy loss).  We distinguish two sets in the optimization. First we consider the case $a\geq b$. Then it is easy to see that in order to minimize the loss we have $a=b$. 
    \begin{align*}
    	\partial_a L &= 2\frac{e^a}{1+e^a}p_0-\frac{e^{-a}}{1+e^{-a}}(1-p_0)
    \end{align*}
    This yields $e^a=\frac{1-p_0}{p_0}$ or $a=\log\big(\frac{1-p_0}{p_0}\big)$ and the minimum for $a\geq b$ is attained on the boundary of the domain of $a\leq b$. The other case is $a\leq b$. We get
    \begin{align*}
    	\partial_a L =& \Big[\frac{(1+\lambda)}{2}\frac{-e^{-a}}{1+e^{-a}}+\frac{(1-\lambda)}{2}\frac{e^a}{1+e^a}\Big](1-p_0)
                   + p_0 \frac{e^a}{1+e^a}\\
    	\partial_b L =& \Big[\frac{(1+\lambda)}{2}\frac{e^{b}}{1+e^{b}}+\frac{(1-\lambda)}{2}\frac{-e^b}{1+e^{-b}}\Big]p_0 
                  + (1-p_0) \frac{-e^{-b}}{1+e^{-b}}
    \end{align*}
    This yields the solution
    \begin{align*}
    	a^* = \log\left(\frac{\frac{1+\lambda}{2}(1-p_0)}{p_0 + \frac{1-\lambda}{2}(1-p_0)}\right), \qquad
    	b^* = \log\left(\frac{\frac{1-\lambda}{2}p_0 + (1-p_0)}{\frac{1+\lambda}{2}p_0}\right)
    \end{align*}
    It is straightforward to check that $a^*<b^*$ for all $0<p_0<1$, indeed we have
	\begin{align*}
		\frac{\frac{1+\lambda}{2}(1-p_0)}{p_0 + \frac{1-\lambda}{2}(1-p_0)} &= \frac{\frac{1+\lambda}{2}(1-p_0)}{p_0 \frac{1+\lambda}{2} +\frac{1-\lambda}{2}}
		=\frac{1-p_0 - \frac{(1-\lambda)}{2}(1-p_0)}{p_0 \frac{1+\lambda}{2} +\frac{1-\lambda}{2}} < \frac{\frac{1-\lambda}{2}p_0 + (1-p_0)}{\frac{1+\lambda}{2}p_0}
	\end{align*}
	if $0<p_0<1$ and note that $\lambda<1$ by assumption. We have $a^*<0$ and thus $g_2(0)>g_1(0)$ (Bayes optimal decision for $x=0$) if 
	\begin{align*}
    	1 > \frac{1-p_0}{p_0}\lambda,
	\end{align*}
   	and $b^*>0$ and thus $g_1(\epsilon)>g_2(\epsilon)$ (Bayes optimal decision for $x=\epsilon$) if
   	\begin{align*}
	    1 > \frac{p_0}{1-p_0}\lambda.
   	\end{align*}
    Thus we recover the Bayes classifier if
    \begin{align*}
    	\lambda < \min\Big\{\frac{1-p_0}{p_0},\frac{p_0}{1-p_0}\Big\}.
   	\end{align*}
\end{proof}
	\section{Experiments}
\label{sec:supp-experiments}

In the following, we provide additional details on our experimental setup regarding (a) the used attacks, especially our \PGD-\FConf attack including pseudo-code in \secref{subsec:supp-experiments-attacks}, (b) training of \AdvTrain and \ConfTrain in \secref{subsec:supp-experiments-training}, (c) the evaluated baselines in \secref{subsec:supp-experiments-baselines} and (d) the used evaluation metrics in \secref{subsec:supp-experiments-evaluation}. Afterwards, we include additional experimental results, including ablation studies in \secref{subsec:supp-experiments-ablation}, qualitative results for analysis in \secref{subsec:supp-experiments-analysis}, and further results for $95\%$ and $98\%$ true positive rate (TPR), results per attack and results per corruption on MNIST-C \citep{MuICMLWORK2019} and Cifar10-C \citep{HendrycksARXIV2019} in \secref{subsec:supp-experiments-results}.

\subsection{Attacks}
\label{subsec:supp-experiments-attacks}

\textbf{Projected Gradient Descent (\PGD):}
Complementary to the description of the projected gradient descent (\PGD) attack by \cite{MadryICLR2018} and our adapted attack, we provide a detailed algorithm in \algref{alg:supp-pgd}. We note that the objective maximized in \citep{MadryICLR2018} is
\begin{align}
    \mathcal{F}(x + \delta, y) = \cL(f(x + \delta; w), y)\label{eq:supp-attack}
\end{align}
where $\mathcal{L}$ denotes the cross-entropy loss, $f(\cdot;w)$ denotes the model and $(x,y)$ is an input-label pair from the test set. Our adapted attack, in contrast, maximizes
\begin{align}
    \mathcal{F}(x + \delta, y) = \max_{k\neq y}f_k(x + \delta;w)\label{eq:supp-conf-attack}
\end{align}
where $f_k$ denotes the confidence of $f$ in class $k$. Note that the maximum over labels, \ie, $\max_{k\neq y}$, is explicitly computed during optimization; this means that in contrast to \citep{GoodfellowOPENREVIEW2019}, we do not run $(K - 1)$ targeted attacks and subsequently take the maximum-confidence one, where $K$ is the number of classes. We denote these two variants as \PGD-\FCE and \PGD-\FConf, respectively. Deviating from \citep{MadryICLR2018}, we initialize $\delta$ uniformly over directions and norm (instead of uniform initialization over the volume of the $\epsilon$-ball):
\begin{align}
    \delta = u\epsilon \frac{\delta'}{\|\delta'\|_\infty},\quad\delta' \sim \mathcal{N}(0, I),u\sim U(0,1)\label{eq:supp-initialization}
\end{align}
where $\delta'$ is sampled from a standard Gaussian and $u \in [0,1]$ from a uniform distribution. We also consider zero initialization, \ie, $\delta = 0$. For random initialization we always consider multiple restarts, $10$ for \PGD-\FConf and $50$ for \PGD-\FCE; for zero initialization, we use $1$ restart. Finally, in contrast to \cite{MadryICLR2018}, we run \PGD for exactly $T$ iterations, taking the perturbation corresponding to the best objective value obtained throughout the optimization.

\begin{algorithm*}[t]
    \caption{\textbf{Projected Gradient Descent (\PGD) with Backtracking.} Pseudo-code for the used \PGD procedure to maximize \eqnref{eq:supp-attack} or \eqnref{eq:supp-conf-attack} using momentum and backtracking subject to the constraints $\tilde{x}_i = x_i + \delta_i \in [0,1]$ and $\|\delta\|_\infty \leq \epsilon$; in practice, the procedure is applied on batches of inputs. The algorithm is easily adapted to work with arbitrary $L_p$-norm; only the projections on Line \ref{line:projection1} and \ref{line:projection2} as well as the normalized gradient in Line \ref{line:normalized-grad} need to be adapted.}
    \label{alg:supp-pgd}
    \begin{algorithmic}[1]
        \small
        \STATEx \textbf{input:} example $x$ with label $y$
        \STATEx \textbf{input:} number of iterations $T$
        \STATEx \textbf{input:} learning rate $\gamma$, momentum $\beta$, learning rate factor $\alpha$
        \STATEx \textbf{input:} initial $\delta^{(0)}$, \eg, \eqnref{eq:supp-initialization} or $\delta^{(0)} = 0$
        \STATE $v := 0$ \COMMENT{saves the best objective achieved}
        \STATE $\tilde{x} := x + \delta^{(0)}$ \COMMENT{best adversarial example obtained}
        \STATE $g^{(-1)} := 0$ \COMMENT{accumulated gradients}
        \FOR{$t = 0,\ldots,T$}
        \STATE \COMMENT{projection onto $L_\infty$ $\epsilon$-ball and on $[0,1]$:}
        \STATE clip $\delta^{(t)}_i$ to $[-\epsilon, \epsilon]$\label{line:projection1}
        \STATE clip $x_i + \delta^{(t)}_i$ to $[0,1]$
        \STATE \COMMENT{forward and backward pass to get objective and gradient:}
        \STATE $v^{(t)} := \mathcal{F}(x + \delta^{(t)}, y)$ \COMMENT{see \eqnref{eq:supp-attack} or \eqnref{eq:supp-conf-attack}}
        \STATE $g^{(t)} := \text{sign}\left(\nabla_{\delta^{(t)}} \mathcal{F}(x + \delta^{(t)}, y)\right)$\label{line:normalized-grad}
        \STATE \COMMENT{keep track of adversarial example resulting in best objective:}
        \IF{$v^{(t)} > v$}
        \STATE $v := v^{(t)}$
        \STATE $\tilde{x} := x + \delta^{(t)}$
        \ENDIF
        \STATE \COMMENT{iteration $T$ is only meant to check whether last update improved objective:}
        \IF{$t = T$}
        \STATE \textbf{break}
        \ENDIF
        \STATE \COMMENT{integrate momentum term:}
        \STATE $g^{(t)} := \beta g^{(t - 1)} + (1 - \beta)g^{(t)}$
        \STATE \COMMENT{``try'' the update step and see if objective increases:}
        \STATE $\hat{\delta}^{(t)} := \delta^{(t)} + \gamma g^{(t)}$
        \STATE clip $\hat{\delta}^{(t)}_i$ to $[-\epsilon, \epsilon]$\label{line:projection2}
        \STATE clip $x_i + \hat{\delta}^{(t)}_i$ to $[0,1]$
        \STATE $\hat{v}^{(t)} := \mathcal{F}(x + \hat{\delta}^{(t)}, y)$
        \STATE \COMMENT{only keep the update if the objective increased; otherwise decrease learning rate:}
        \IF{$\hat{v}^{(t)} \geq v^{(t)}$}
        \STATE $\delta^{(t + 1)} := \hat{\delta}^{(t)}$
        \ELSE
        \STATE $\gamma := \nicefrac{\gamma}{\alpha}$
        \ENDIF
        \ENDFOR
        \STATE \textbf{return} $\tilde{x}$, $\tilde{v}$
    \end{algorithmic}
\end{algorithm*}
\begin{table*}[t]
    \centering
    \begin{subfigure}{1\textwidth}
        \centering
        \footnotesize
        \begin{tabularx}{1\textwidth}{| X ||c|c|c|c||c||c|c||c|c|}
    \hline
    \multicolumn{10}{|c|}{\textbf{MNIST:} Attack Ablation with confidence-thresholded \RTE in \% for $\tau$@$99\%$TPR}\\
    \multicolumn{10}{|c|}{($L_\infty$ attack with $\epsilon = 0.3$ for training and testing)}\\
    \hline
    Optimization & \multicolumn{5}{c||}{momentum+backtrack} & \multicolumn{2}{c||}{mom} & \multicolumn{2}{c|}{--}\\
    \hline
    Initialization & \multicolumn{4}{c||}{zero} & rand & \multicolumn{2}{c||}{zero} & \multicolumn{2}{c|}{zero}\\
    \hline
    Iterations & $40$ & $200$ & $1000$ & $2000$ & $2000$ & $60$ & $300$ & $60$ & $300$\\
    \hline
    \hline
    \AdvTrain & 0.4 & 0.4 & 0.4 & 0.3 & 0.6 & 0.4 & 0.6 & 0.4 & 0.4\\
    \AdvTrain\FConf (\AdvTrain trained with \PGD-\FConf) & 0.8 & 0.8 & 0.8 & 0.8 & 1.1 & 1.0 & 1.1 & 1.0 & 1.0\\
    \ConfTrain & 0.2 & 1.4 & 4.6 & 4.6 & 3.7 & 0.7 & 0.7 & 0.2 & 0.2\\
    \hline
\end{tabularx}
    \end{subfigure}
    \vskip 2px
    \begin{subfigure}{1\textwidth}
        \centering
        \footnotesize
        \begin{tabularx}{1\textwidth}{| X ||c|c|c|c||c||c|c||c|c|}
\hline
\multicolumn{10}{|c|}{\textbf{SVHN:} Attack Ablation with confidence-thresholded \RTE in \% for $\tau$@$99\%$TPR}\\
\multicolumn{10}{|c|}{($L_\infty$ attack with $\epsilon = 0.03$ for training and testing)}\\
\hline
Optimization & \multicolumn{5}{c||}{momentum+backtrack} & \multicolumn{2}{c||}{mom} & \multicolumn{2}{c|}{--}\\
\hline
Initialization & \multicolumn{4}{c||}{zero} & rand & \multicolumn{2}{c||}{zero} & \multicolumn{2}{c|}{zero}\\
\hline
Iterations & $40$ & $200$ & $1000$ & $2000$ & $2000$ & $60$ & $300$ & $60$ & $300$\\
\hline
\hline
\AdvTrain & 38.4 & 46.2 & 49.9 & 50.1 & 51.8 & 37.7 & 38.1 & 29.9 & 30.8\\
\AdvTrain\FConf (\AdvTrain trained with \PGD-\FConf) & 27.4 & 40.5 & 46.9 & 47.3 & 48.1 & 27.1 & 28.5 & 21.1 & 23.8\\
\ConfTrain & 4.0 & 5.0 & 22.8 & 23.3 & 5.2 & 2.6 & 2.6 & 2.6 & 2.6\\
\hline
\end{tabularx}
    \end{subfigure}
    \vskip 2px
    \begin{subfigure}{1\textwidth}
        \centering
        \footnotesize
        \begin{tabularx}{1\textwidth}{| X ||c|c|c|c||c||c|c||c|c|}
\hline
\multicolumn{10}{|c|}{\textbf{CIFAR10:} Attack Ablation with confidence-thresholded \RTE in \% for $\tau$@$99\%$TPR}\\
\multicolumn{10}{|c|}{($L_\infty$ attack with $\epsilon = 0.03$ for training and testing)}\\
\hline
Optimization & \multicolumn{5}{c||}{momentum+backtrack} & \multicolumn{2}{c||}{mom} & \multicolumn{2}{c|}{--}\\
\hline
Initialization & \multicolumn{4}{c||}{zero} & rand & \multicolumn{2}{c||}{zero} & \multicolumn{2}{c|}{zero}\\
\hline
Iterations & $40$ & $200$ & $1000$ & $2000$ & $2000$ & $60$ & $300$ & $60$ & $300$\\
\hline
\hline
\AdvTrain & 60.9 & 60.8 & 60.8 & 60.8 & 60.9 & 60.9 & 60.9 & 57.4 & 57.6\\
\AdvTrain\FConf (\AdvTrain trained with \PGD-\FConf) & 60.4 & 60.6 & 60.5 & 60.5 & 60.9 & 60.4 & 60.6 & 56.2 & 56.6\\
\ConfTrain & 14.8 & 16.2 & 40.2 & 41.3 & 34.9 & 7.2 & 7.2 & 7.2 & 7.2\\
\hline
\end{tabularx}
    \end{subfigure}
    \vskip -6px
    \caption{\textbf{Detailed Attack Ablation Studies.} We compare our $L_\infty$ \PGD-\FConf attack with $T$ iterations and different combinations of momentum, backtracking and initialization on all three datasets. We consider \AdvTrain, \AdvTrain trained with \PGD-\FConf (\AdvTrain\FConf), and \ConfTrain; we report \RTE for confidence threshold $\tau$@$99\%$TPR. As backtracking requires an additional forward pass per iteration, we use $T = 60$ and $T = 300$ for attacks without backtracking to be comparable to attacks with $T = 40$ and $T = 200$ with backtracking. Against \ConfTrain, $T=1000$ iterations or more are required and backtracking is essential to achieve high \RTE. \AdvTrain, in contrast, is ``easier'' to attack, requiring less iterations and less sophisticated optimization (\ie, without momentum and/or backtracking).}
    \label{tab:supp-experiments-attack}
\end{table*}
\begin{table}[t]
    \centering
    \begin{subfigure}[t]{1\textwidth}
        \vspace*{0px}
        
        \centering
        \footnotesize
\begin{tabularx}{1\textwidth}{| X ||c|c|c||c|c|}
\hline
\multicolumn{6}{|c|}{\textbf{SVHN:} Training Ablation for Detection ($\tau$@$99\%$TPR) \underline{and} Standard Settings ($\tau = 0$)}\\
\multicolumn{6}{|c|}{($L_\infty$ attack with $\epsilon = 0.03$ during training and testing)}\\
\hline
\hline
& \multicolumn{3}{c||}{\begin{tabular}{c}\textbf{Detection Setting}\\$\tau$@$99\%$TPR\end{tabular}} & \multicolumn{2}{c|}{\begin{tabular}{c}\textbf{Standard Setting}\\$\tau = 0$\end{tabular}}\\
\hline
& \begin{tabular}{c}ROC\\AUC\end{tabular} &
\begin{tabular}{c}\TE\\ in \%\end{tabular} &
\begin{tabular}{c}\RTE\\ in \%\end{tabular} &
\begin{tabular}{c}\TE\\ in \%\end{tabular} & \begin{tabular}{c}\RTE\\ in \%\end{tabular}\\
\hline
\hline
\Normal & 0.17 & 2.6 & 99.9 & 3.6 & 99.9\\
\AdvTrain & 0.55 & 2.5 & 54.9 & 3.4 & 56.9\\
\AdvTrain\FConf (\AdvTrain trained with \PGD-\FConf) & 0.61 & 2.8 & 52.5 & 3.7 & 58.7\\
\hline
\ConfTrain, $\rho = 1$ & 0.74 & 2.2 & 43.0 & 2.7 & 82.4\\
\ConfTrain, $\rho = 2$ & 0.68 & 2.1 & 44.2 & 2.9 & 79.6\\
\ConfTrain, $\rho = 4$ & 0.68 & 1.8 & 35.8 & 2.7 & 80.4\\
\ConfTrain, $\rho = 6$ & 0.64 & 1.8 & 32.8 & 2.9 & 72.1\\
\ConfTrain, $\rho = 8$ & 0.63 & 2.2 & 42.3 & 2.9 & 84.6\\
\ConfTrain, $\rho = 10$ & 0.67 & 2.1 & 38.5 & 2.9 & 91.0\\
\ConfTrain, $\rho = 12$ & 0.67 & 1.9 & 36.3 & 2.8 & 81.8\\
\hline
\end{tabularx}
    \end{subfigure}
    \vskip 2px
    \begin{subfigure}[t]{1\textwidth}
        \vspace*{0px}
        
        \centering
        \footnotesize
        \begin{tabularx}{1\textwidth}{| X ||c|c|c||c|c|}
\hline
\multicolumn{6}{|c|}{\textbf{CIFAR10:} Training Ablation for Detection ($\tau$@$99\%$TPR) \underline{and} Standard Settings ($\tau = 0$)}\\
\multicolumn{6}{|c|}{($L_\infty$ attack with $\epsilon = 0.03$ during training and testing)}\\
\hline
\hline
& \multicolumn{3}{c||}{\begin{tabular}{c}\textbf{Detection Setting}\\$\tau$@$99\%$TPR\end{tabular}} & \multicolumn{2}{c|}{\begin{tabular}{c}\textbf{Standard Setting}\\$\tau = 0$\end{tabular}}\\
\hline
& \begin{tabular}{c}ROC\\AUC\end{tabular} &
\begin{tabular}{c}\TE\\ in \%\end{tabular} &
\begin{tabular}{c}\RTE\\ in \%\end{tabular} &
\begin{tabular}{c}\TE\\ in \%\end{tabular} &
\begin{tabular}{c}\RTE\\ in \%\end{tabular}\\
\hline
\hline
\Normal & 0.20 & 7.4 & 100.0 & 8.3 & 100.0\\
\AdvTrain & 0.65 & 15.1 & 60.9 & 16.6 & 61.3\\
\AdvTrain\FConf (\AdvTrain trained with \PGD-\FConf) & 0.63 & 15.1 & 61.5 & 16.1 & 61.7\\
\hline
\ConfTrain, $\rho = 1$ & 0.63 & 8.7 & 72.4 & 9.7 & 95.3\\
\ConfTrain, $\rho = 2$ & 0.60 & 8.4 & 70.6 & 9.7 & 95.1\\
\ConfTrain, $\rho = 4$ & 0.61 & 8.6 & 66.3 & 9.8 & 93.5\\
\ConfTrain, $\rho = 6$ & 0.54 & 8.0 & 69.8 & 9.2 & 94.1\\
\ConfTrain, $\rho = 8$ & 0.58 & 8.5 & 65.3 & 9.4 & 93.2\\
\ConfTrain, $\rho = 10$ & 0.60 & 8.7 & 63.0 & 10.1 & 95.0\\
\ConfTrain, $\rho = 12$ & 0.62 & 9.4 & 63.0 & 10.1 & 96.6\\
\hline
\end{tabularx}
    \end{subfigure}
    \vskip -6px
    \caption{\textbf{Training Ablation Studies.} We report unthresholded \RTE and \TE, \ie, $\tau = 0$ (``Standard Setting''), and $\tau$@$99\%$TPR as well as ROC AUC (``Detection Setting'') for \ConfTrain with various values for $\rho$. The models are tested against our $L_\infty$ \PGD-\FConf attack with $T = 1000$ iterations and zero as well as random initialization. On Cifar10, $\rho = 10$ works best and performance stagnates for $\rho > 10$. On SVHN, we also use $\rho = 10$, although $\rho = 6$ shows better results.}
    \label{tab:supp-experiments-training}
\end{table}
\begin{figure*}[t]
    \begin{subfigure}[t]{0.485\textwidth}
        \vspace*{0px}
        
        \centering
        \textbf{MNIST} (worst-case of $L_\infty$ attacks with $\epsilon = 0.3$)
        
        \begin{subfigure}[t]{0.49\textwidth}
            \vspace*{0px}
            
            \centering
            \includegraphics[width=1\textwidth]{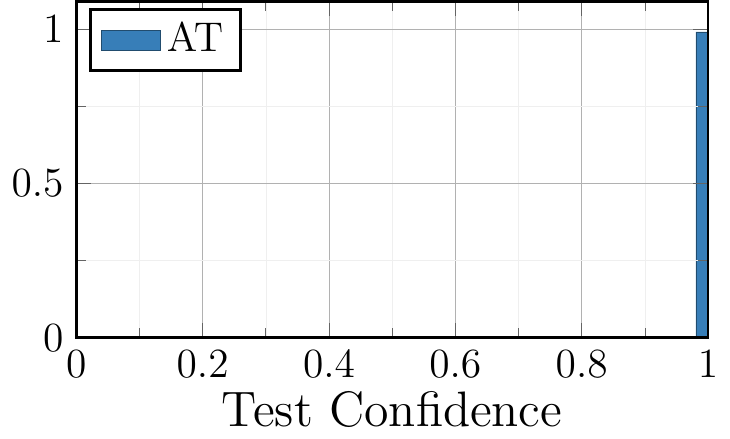}
            
            \includegraphics[width=1\textwidth]{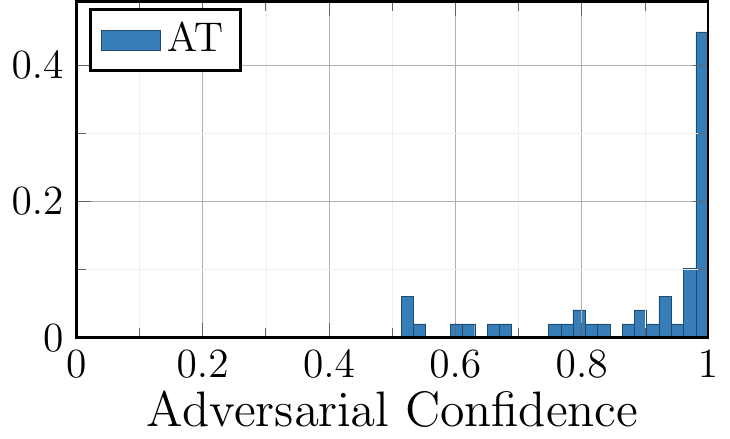}
        \end{subfigure}
        \begin{subfigure}[t]{0.49\textwidth}
            \vspace*{0px}
            
            \centering
            \includegraphics[width=1\textwidth]{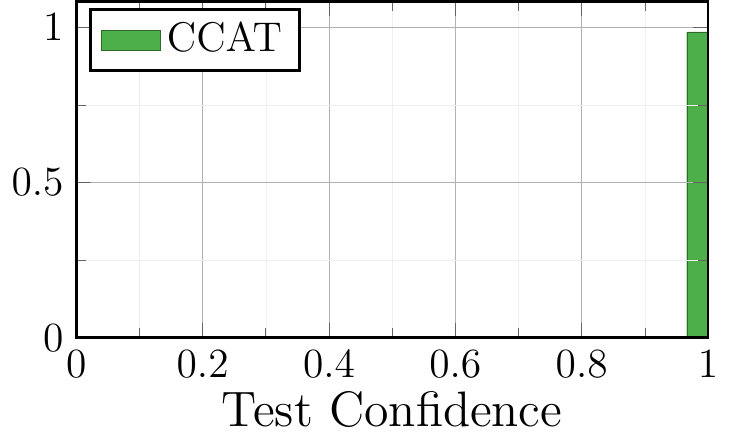}
            
            \includegraphics[width=1\textwidth]{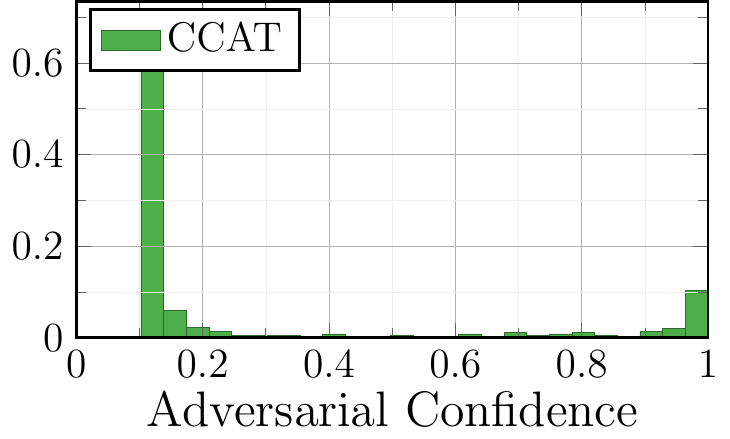}
        \end{subfigure}
    \end{subfigure}
    \hfill\vrule\hfill
    \begin{subfigure}[t]{0.485\textwidth}
        \vspace*{0px}
        
        \centering
        \textbf{Cifar10} (worst-case of $L_\infty$ attacks with $\epsilon = 0.03$)
        
        \begin{subfigure}[t]{0.49\textwidth}
            \vspace*{0px}
            
            \centering
            \includegraphics[width=1\textwidth]{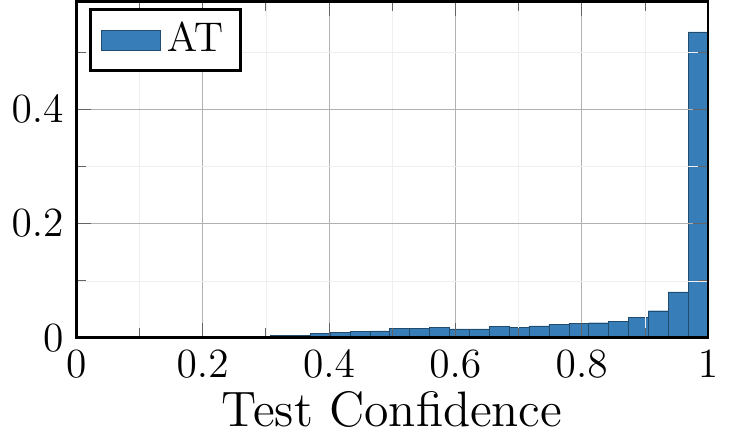}
            
            \includegraphics[width=1\textwidth]{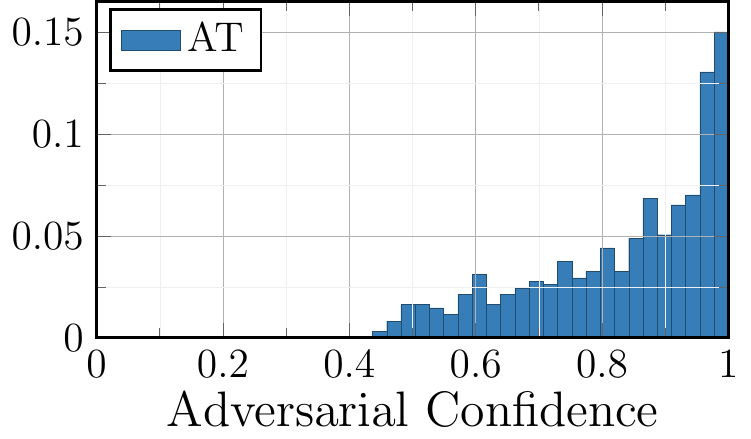}
        \end{subfigure}
        \begin{subfigure}[t]{0.49\textwidth}
            \vspace*{0px}
            
            \centering
            \includegraphics[width=1\textwidth]{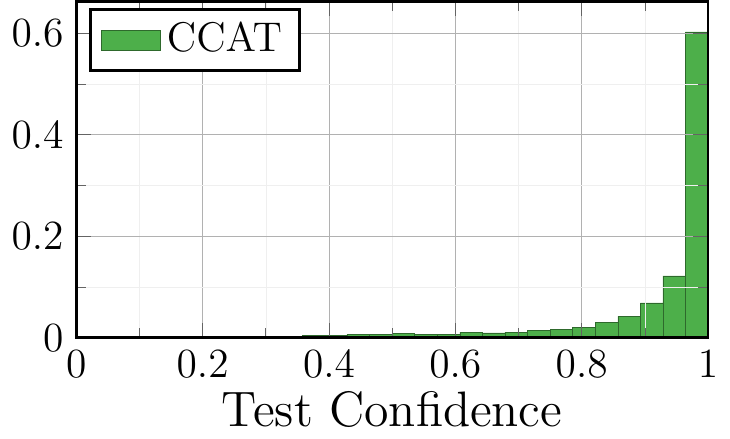}
            
            \includegraphics[width=1\textwidth]{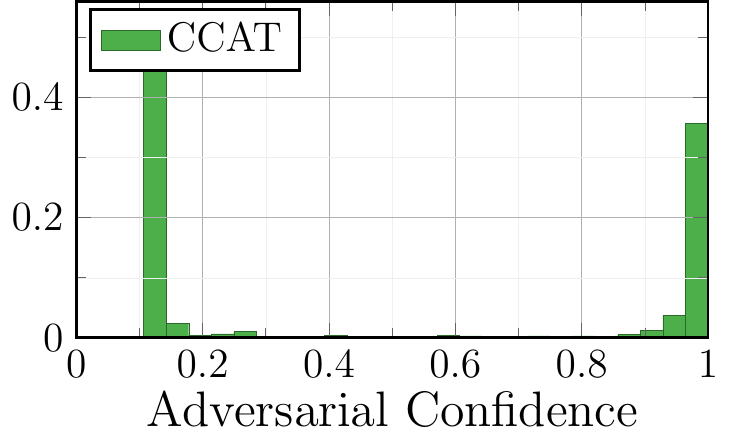}
        \end{subfigure}
    \end{subfigure}
    \vskip -6px
    \caption{\textbf{Confidence Histograms.} We show histograms of confidences on correctly classified test examples (top) and on adversarial examples (bottom) for both \AdvTrain and \ConfTrain. Note that for \AdvTrain, the number of successful adversarial examples is usually lower than for \ConfTrain. For \ConfTrain in contrast, nearly all adversarial examples are successful, while only a part has high confidence. Histograms obtained for the worst-case adversarial examples across all  tested $L_\infty$ attacks with $\epsilon = 0.3$ and $\epsilon = 0.03$ on MNIST and Cifar10, respectively.}
    \label{fig:supp-experiments-histograms}
\end{figure*}

\textbf{\PGD for $\boldsymbol{L_p}$, $\boldsymbol{p \in \{0, 1, 2\}}$:}
Both \PGD-\FCE and \PGD-\FConf can also be applied using the $L_2$, $L_1$ and $L_0$ norms following the description above. Then, gradient normalization in Line \ref{line:normalized-grad} of \algref{alg:supp-pgd}, the projection in Line \ref{line:projection1}, and the initialization in \eqnref{eq:supp-initialization} need to be adapted. For the $L_2$ norm, the gradient is normalized by dividing by the $L_2$ norm; for the $L_1$ norm only the $1\%$ largest values (in absolute terms) of the gradient are kept and normalized by their $L_1$ norm; and for the $L_0$ norm, the gradient is normalized by dividing by the $L_1$ norm. We follow the algorithm of \cite{DuchiICML2008} for the $L_1$ projection; for the $L_0$ projection (onto the $\epsilon$-ball for $\epsilon \in \mathbb{N}_0$), only the $\epsilon$ largest values are kept. Similarly, initialization for $L_2$ and $L_1$ are simple by randomly choosing a direction (as in \eqnref{eq:supp-initialization}) and then normalizing by their norm. For $L_0$, we randomly choose pixels with probability $(\frac{2}{3}\epsilon)/(HWD)$ and set them to a uniformly random values $u \in [0,1]$, where $H \times W \times D$ is the image size. In experiments, we found that tuning the learning rate for \PGD with $L_1$ and $L_0$ constraints (independent of the objective, \ie, \eqnref{eq:supp-attack} or \eqnref{eq:supp-conf-attack}) is much more difficult. Additionally, \PGD using the $L_0$ norm seems to get easily stuck in sub-optimal local optima.

\textbf{Backtracking:}
\algref{alg:supp-pgd} also gives more details on the employed momentum and backtracking scheme. These two ``tricks'' add two additional hyper-parameters to the number of iterations $T$ and the learning rate $\gamma$, namely the momentum parameter $\beta$ and the learning rate factor $\alpha$. After each iteration, the computed update, already including the momentum term, is only applied if this improves the objective. This is checked through an additional forward pass. If not, the learning rate is divided by $\alpha$, and the update is rejected. \algref{alg:supp-pgd} includes this scheme as an algorithm for an individual test example $x$ with label $y$ for brevity; however, extending it to work on batches is straight-forward. However, it is important to note that the learning rate is updated per test example individually. In practice, for \PGD-\FCE, with $T = 200$ iterations, we use $\gamma = 0.05$, $\beta = 0.9$ and $\alpha = 1.25$; for \PGD-\FConf, with $T = 1000$ iterations, we use $\gamma = 0.001$, $\beta = 0.9$ and $\alpha = 1.1$.

\textbf{Black-Box Attacks:}
We also give more details on the used black-box attacks. For random sampling, we apply \eqnref{eq:supp-initialization} $T = 5000$ times. We also implemented the Query-Limited (QL) black-box attack of \cite{IlyasICML2018} using a population of $50$ and variance of $0.1$ for estimating the gradient in Line \ref{line:normalized-grad} of \algref{alg:supp-pgd}; a detailed algorithm is provided in \citep{IlyasICML2018}. We use a learning rate of $0.001$ (note that the gradient is signed, as in \citep{MadryICLR2018}) and also integrated a momentum with $\beta = 0.9$ and backtracking with $\alpha = 1.1$ and $T = 1000$ iterations. We use zero and random initialization; in the latter case we allow $10$ random restarts. For the Simple black-box attack we follow the algorithmic description in \citep{NarodytskaCVPRWORK2017} considering only axis-aligned perturbations of size $\epsilon$ per pixel. We run the attack for $T = 1000$ iterations and allow $10$ random restarts. Following, \cite{KhouryARXIV2018}, we further use the Geometry attack for $T = 1000$ iterations. Random sampling, QL, Simple and Geometry attacks are run for arbitrary $L_p$, $p \in \{\infty, 2, 1, 0\}$. For $L_\infty$, we also use the Square attack proposed in \citep{AndriushchenkoARXIV2019} with $T = 5000$ iterations with a probability of change of $0.05$. For all attacks, we use \eqnref{eq:supp-conf-attack} as objective. Finally, for $L_0$, we also use Corner Search \cite{CroceICCV2019} with the cross-entropy loss as objective, for $T = 200$ iterations. We emphasize that, except for QL, these attacks are not gradient-based and do not approximate the gradient. Furthermore, we note that all attacks except Corner Search are adapted to explicitly attack \ConfTrain by maximizing \eqnref{eq:supp-conf-attack}.

\subsection{Training}
\label{subsec:supp-experiments-training}

We follow the ResNet-20 architecture by \cite{HeCVPR2016} implemented in PyTorch \citep{PaszkeNIPSWORK2017}. For training we use a batch size of $100$ and train for $100$ and $200$ epochs on MNIST and SVHN/Cifar10, respectively: this holds for normal training, adversarial training (\AdvTrain) and confidence-calibrated adversarial training (\ConfTrain). For the latter two, we use \PGD-\FCE and \PGD\FConf, respectively, for $T = 40$ iterations including momentum and backtracking ($\beta = 0.9$, $\alpha = 1.5$). For \PGD-\FCE we use a learning rate of $0.05$, $0.01$ and $0.005$ on MNIST, SVHN and Cifar10. For \PGD-\FConf we use a learning rate of $0.005$. For \ConfTrain, we randomly switch between the initialization in \eqnref{eq:supp-initialization} and zero initialization. For training, we use standard stochastic gradient descent, starting with a learning rate of $0.1$ on MNIST/SVHN and $0.075$ on Cifar10. The learning rate is multiplied by $0.95$ after each epoch. We do not use weight decay; but the network includes batch normalization \citep{IoffeICML2015}. On SVHN and Cifar10, we use random cropping, random flipping (only Cifar10) and contrast augmentation during training. We always train on $50\%$ clean and $50\%$ adversarial examples per batch, \ie, each batch contains both clean and adversarial examples which is important when using batch normalization.

\subsection{Baselines}
\label{subsec:supp-experiments-baselines}

As baseline, we use the multi-steepest descent (\Wong) adversarial training of \cite{MainiICML2020}, using the code and models provided in the official repository\footnote{\url{https://github.com/locuslab/robust_union}}. The models correspond to a LeNet-like \cite{LecunIEEE1998} architecture on MNIST, and the pre-activation version of ResNet-18 \cite{HeCVPR2016} on Cifar10. The models were trained with $L_\infty$, $L_2$ and $L_1$ adversarial examples and $\epsilon$ set to $0.3, 1.5, 12$ and $0.03, 0.5, 12$, respectively. We attacked these models using the same setup as used for standard \AdvTrain and our \ConfTrain. 

Additionally, we compare to \TRADES \cite{ZhangICML2019} using the code and pre-trained models from the official repository\footnote{\url{https://github.com/yaodongyu/TRADES}}. The models correspond to a convolutional architecture with four convolutional and three fully-connected layers \cite{CarliniSP2017} on MNIST, and a wide ResNet, specifically WRN-10-28 \cite{ZagoruykoBMVC2016}, on Cifar10. Both are trained using \emph{only} $L_\infty$ adversarial examples with $\epsilon = 0.3$ and $\epsilon = 0.03$, respectively. The evaluation protocol follows the same setup as used for standard \AdvTrain and \ConfTrain.

On Cifar10, we also use the pre-trained ResNet-50 from \cite{MadryICLR2018} obtained from the official repository\footnote{\url{https://github.com/MadryLab/robustness}}. The model was trained on $L_\infty$ adversarial examples with $\epsilon = 0.03$. The same evaluation as for \ConfTrain applies.

Furthermore, we evaluate two detection baseline: the Mahalanobis detector (\Lee) of \cite{MaICLR2018} and the local intrinsic dimensionality (\Ma) detector of \cite{LeeNIPS2018}.
We used the code provided by \cite{LeeNIPS2018} from the official repository\footnote{\url{https://github.com/pokaxpoka/deep_Mahalanobis_detector}}. For evaluation, we used the provided setup, adding \emph{only} \PGD-\FCE and \PGD-\FConf with $T=1000$, $T=200$ and $T = 40$. For $T=1000$, we used $5$ random restarts, for $T=200$, we used $25$ restarts, and for $T = 40$, we used one restart. These were run for $L_\infty$, $L_2$, $L_1$ and $L_0$. We also evaluated distal adversarial examples as in the main paper. While the hyper-parameters were chosen considering our $L_\infty$ \PGD-\FCE attack ($T = 40$, one restart) and kept fixed for other threat models, the logistic regression classifier trained on the computed statistics (\eg, the Mahalanobis statistics) is trained for each threat model individually, resulting in an advantage over \AdvTrain and \ConfTrain. For worst-case evaluation, where we keep the highest-confidence adversarial example per test example for \ConfTrain, we use the obtained detection score instead. This means, for each test example individually, we consider the adversarial example with worst detection score for evaluation.

\subsection{Evaluation Metrics}
\label{subsec:supp-experiments-evaluation}

Complementing the discussion in the main paper, we describe the used evaluation metrics and evaluation procedure in more detail. Adversarial examples are computed on the first 1000 examples of the test set; the used confidence threshold is computed on the last 1000 examples of the test set; test errors are computed on all test examples minus the last 1000. As we consider multiple attacks, and some attacks allow multiple random restarts, we always consider the worst case adversarial example per test example and across all attacks/restarts; the worst-case is selected based on confidence.

\textbf{FPR and ROC AUC:}
To compute receiver operating characteristic (ROC) curves, and the area under the curve, \ie, ROC AUC, we define negatives as \emph{successful} adversarial examples (corresponding to correctly classified test examples) and positives as the corresponding \emph{correctly classified} test examples. The ROC AUC as well as the curve itself can easily be calculated using scikit-learn \citep{PedregosaJMLR2011}. Practically, the generated curve could be used to directly estimate a threshold corresponding to a pre-determined true positive rate (TPR). However, this requires interpolation; after trying several interpolation schemes, we concluded that the results are distorted significantly, especially for TPRs close to $100\%$. Thus, we follow a simpler scheme: on a held out validation set of size $1000$ (the last 1000 samples of the test set), we sorted the corresponding confidences, and picked the confidence threshold in order to obtain (at least) the desired TPR, \eg, $99\%$.

\begin{figure*}[t]
    \begin{subfigure}[t]{0.485\textwidth}
        \vspace*{0px}
        
        \centering
        \textbf{MNIST} (worst-case of $L_\infty$ attacks with $\epsilon = 0.3$)
        
        \begin{subfigure}[t]{0.47\textwidth}
            \vspace*{0px}
            
            \centering
            \includegraphics[height=3.5cm]{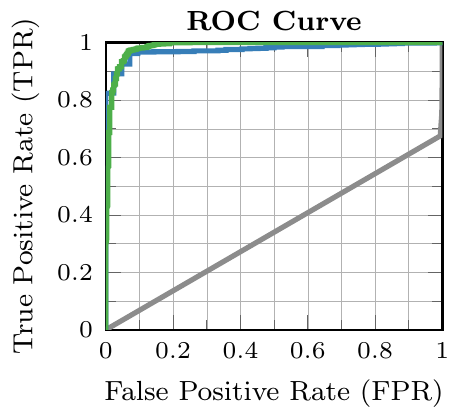}
        \end{subfigure}
        \begin{subfigure}[t]{0.49\textwidth}
            \vspace*{0px}
            
            \centering			
            \includegraphics[height=3.5cm]{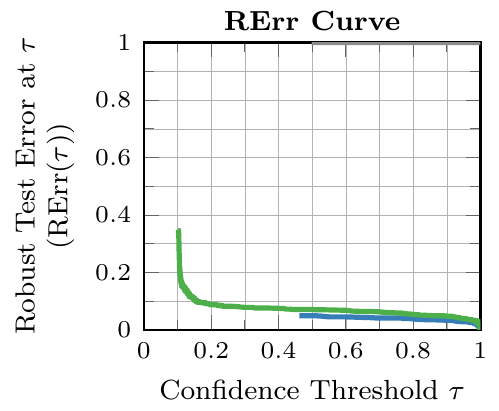}
        \end{subfigure}
        \begin{subfigure}{0.925\textwidth}
            \fbox{
                \hspace*{1.5cm}\includegraphics[width=0.6\textwidth]{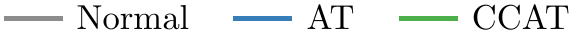}\hspace*{1.5cm}
            }
        \end{subfigure}
    \end{subfigure}
    \hfill
    \vrule
    \hfill
    \begin{subfigure}[t]{0.485\textwidth}
        \vspace*{0px}
        
        \centering
        \textbf{Cifar10} (worst-case of $L_\infty$ attacks with $\epsilon = 0.03$)
        
        \begin{subfigure}[t]{0.47\textwidth}
            \vspace*{0px}
            
            \centering
            \includegraphics[height=3.5cm]{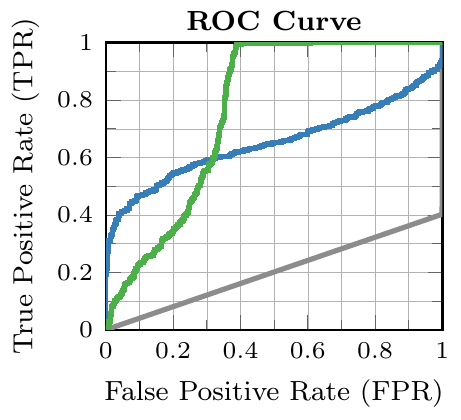}
        \end{subfigure}
        \begin{subfigure}[t]{0.49\textwidth}
            \vspace*{0px}
            
            \centering			
            \includegraphics[height=3.5cm]{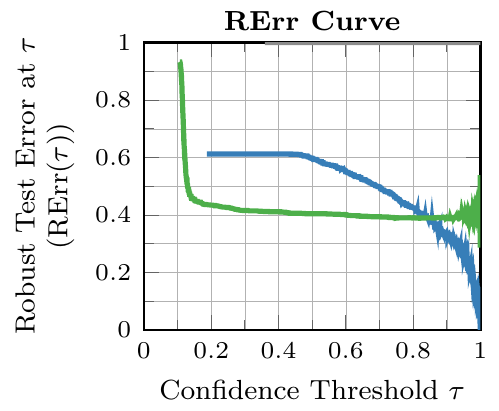}
        \end{subfigure}
        \begin{subfigure}{0.925\textwidth}
            \fbox{
                \hspace*{1.5cm}\includegraphics[width=0.6\textwidth]{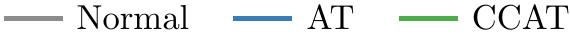}\hspace*{1.5cm}
            }
        \end{subfigure}
    \end{subfigure}
    \vskip -6px
    \caption{\textbf{ROC and \RTE curves.} ROC curves, \ie FPR plotted against TPR for all possible confidence thresholds $\tau$, and (confidence-thresholded) \RTE curves, \ie, \RTE over confidence threshold $\tau$ for \AdvTrain and \ConfTrain, including different $\rho$ parameters. Worst-case adversarial examples across all $L_\infty$ attacks with $\epsilon = 0.3$ (MNIST) and $\epsilon = 0.03$ (Cifar10) were tested. For evaluation, the confidence threshold $\tau$ is fixed at $99\%$TPR, allowing to reject at most $1\%$ correctly classified clean examples. Thus, we also do not report the area under the ROC curve in the main paper.}
    \label{fig:supp-experiments-evaluation}
\end{figure*}
\begin{figure*}[t]
    \centering
    \footnotesize
    \begin{subfigure}{1\textwidth}
        \centering
        \textbf{SVHN:} \textbf{\AdvTrain} with $L_\infty$ \PGD-\FConf, $\epsilon = 0.03$ for training \emph{and} testing
    \end{subfigure}
    \\[2px]
    \begin{subfigure}{0.19\textwidth}
        \includegraphics[height=2.2cm]{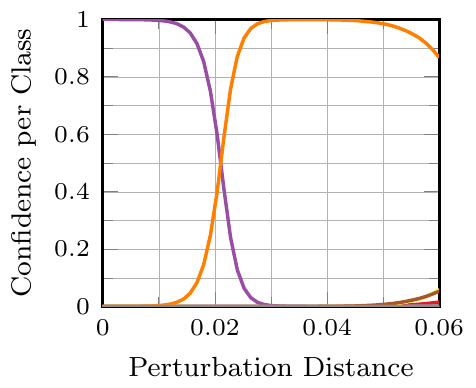}
    \end{subfigure}
    \begin{subfigure}{0.19\textwidth}
        \includegraphics[height=2.2cm]{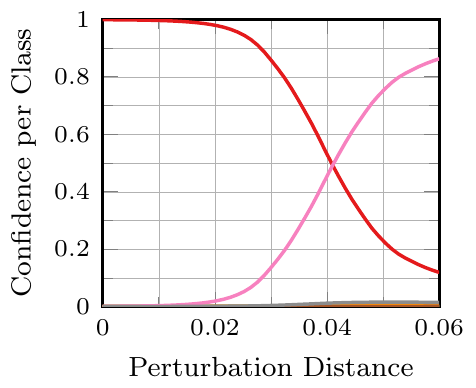}
    \end{subfigure}
    \begin{subfigure}{0.19\textwidth}
        \includegraphics[height=2.2cm]{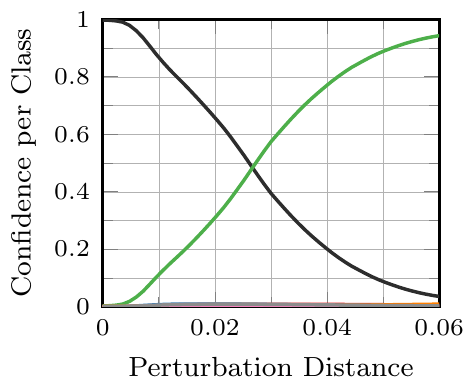}
    \end{subfigure}
    \begin{subfigure}{0.19\textwidth}
        \includegraphics[height=2.2cm]{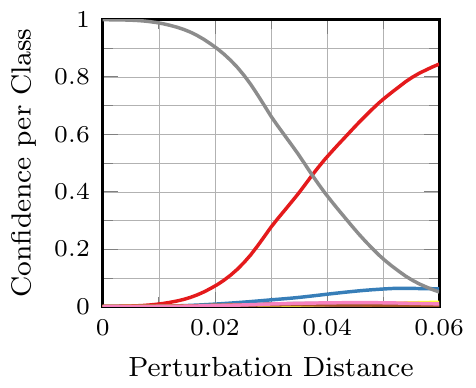}
    \end{subfigure}
    \begin{subfigure}{0.19\textwidth}
        \includegraphics[height=2.2cm]{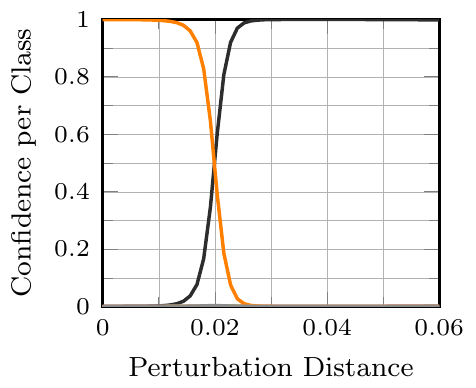}
    \end{subfigure}
    \\
    \begin{subfigure}{0.19\textwidth}
        \includegraphics[height=2.2cm]{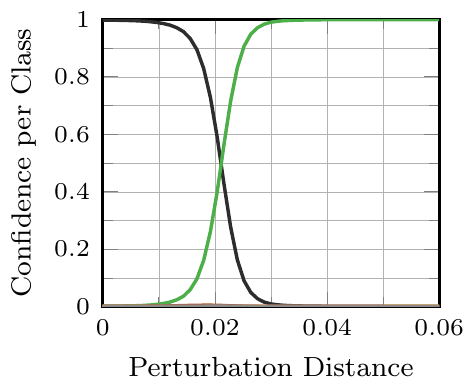}
    \end{subfigure}
    \begin{subfigure}{0.19\textwidth}
        \includegraphics[height=2.2cm]{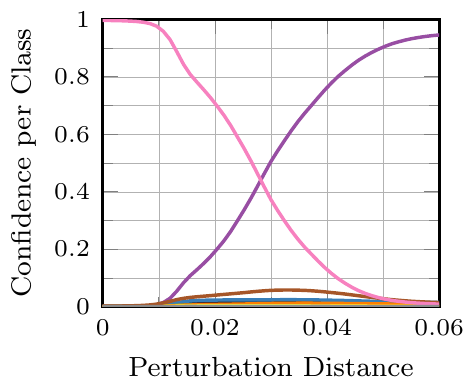}
    \end{subfigure}
    \begin{subfigure}{0.19\textwidth}
        \includegraphics[height=2.2cm]{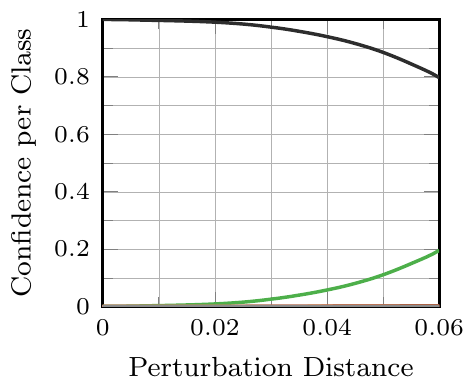}
    \end{subfigure}
    \begin{subfigure}{0.19\textwidth}
        \includegraphics[height=2.2cm]{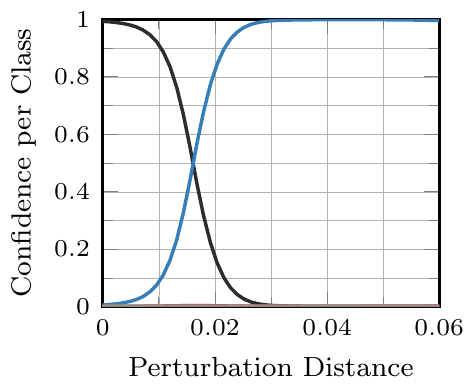}
    \end{subfigure}
    \begin{subfigure}{0.19\textwidth}
        \includegraphics[height=2.2cm]{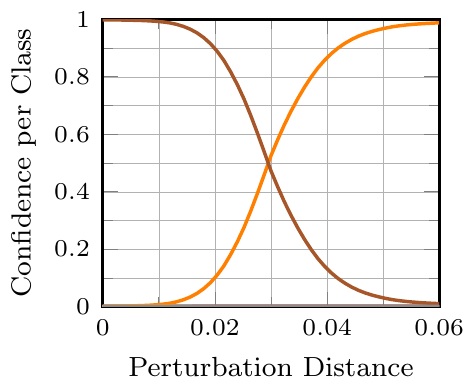}
    \end{subfigure}
    \\[4px]
    \begin{subfigure}{1\textwidth}
        \centering
        \textbf{SVHN:} \textbf{\ConfTrain} with $L_\infty$ \PGD-\FConf, $\epsilon = 0.03$ for training \emph{and} testing
    \end{subfigure}\\[4px]
    \begin{subfigure}{0.19\textwidth}
        \includegraphics[height=2.2cm]{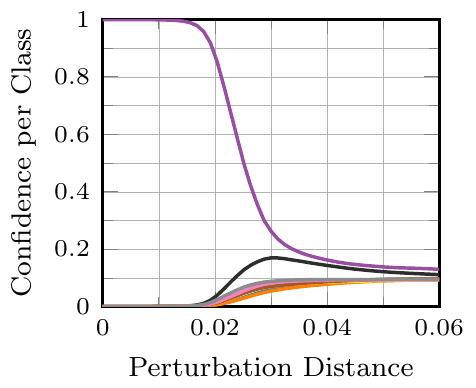}
    \end{subfigure}
    \begin{subfigure}{0.19\textwidth}
        \includegraphics[height=2.2cm]{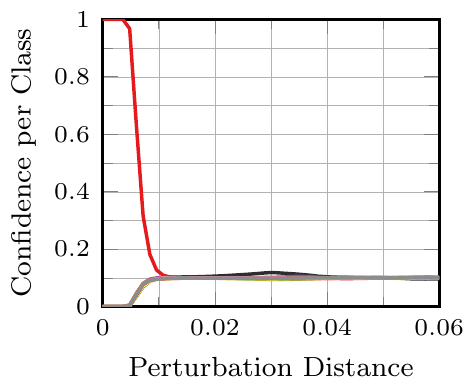}
    \end{subfigure}
    \begin{subfigure}{0.19\textwidth}
        \includegraphics[height=2.2cm]{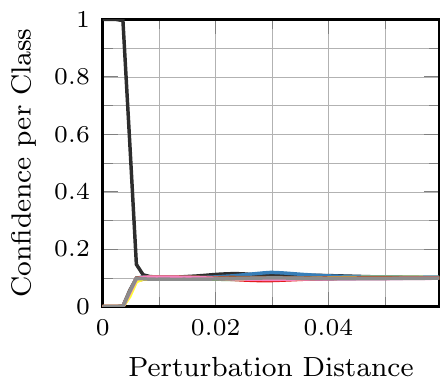}
    \end{subfigure}
    \begin{subfigure}{0.19\textwidth}
        \includegraphics[height=2.2cm]{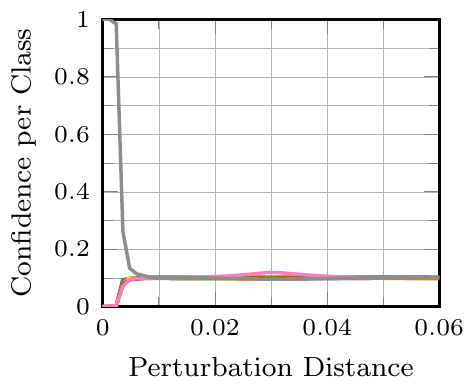}
    \end{subfigure}
    \begin{subfigure}{0.19\textwidth}
        \includegraphics[height=2.2cm]{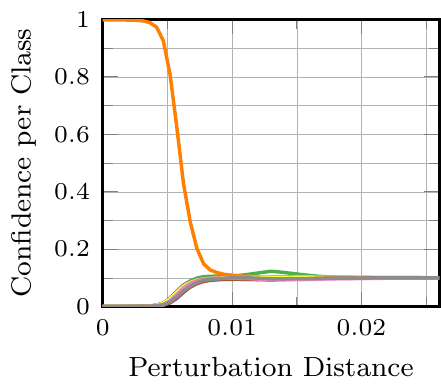}
    \end{subfigure}
    \\
    \begin{subfigure}{0.19\textwidth}
        \includegraphics[height=2.2cm]{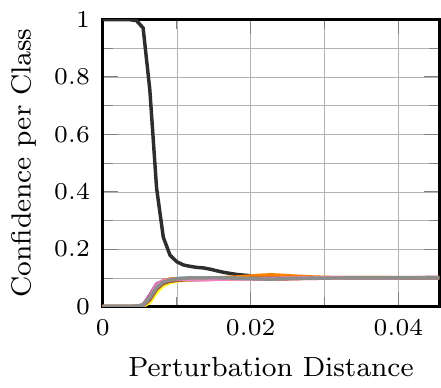}
    \end{subfigure}
    \begin{subfigure}{0.19\textwidth}
        \includegraphics[height=2.2cm]{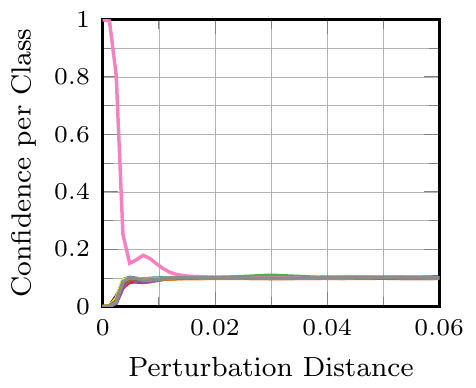}
    \end{subfigure}
    \begin{subfigure}{0.19\textwidth}
        \includegraphics[height=2.2cm]{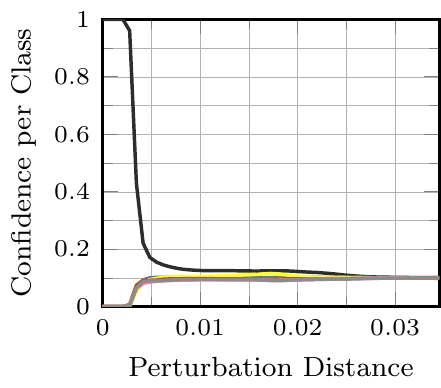}
    \end{subfigure}
    \begin{subfigure}{0.19\textwidth}
        \includegraphics[height=2.2cm]{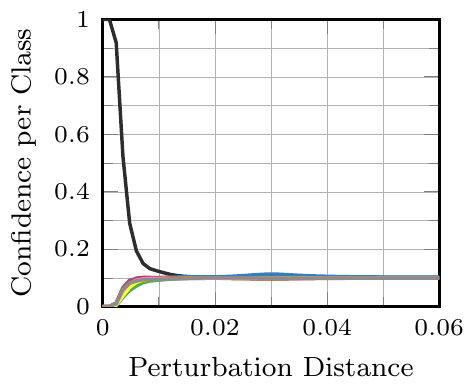}
    \end{subfigure}
    \begin{subfigure}{0.19\textwidth}
        \includegraphics[height=2.2cm]{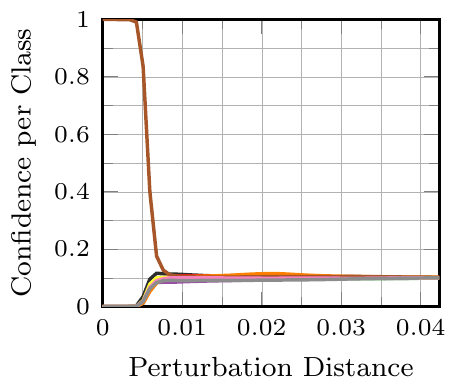}
    \end{subfigure}
    \\[4px]
    \begin{subfigure}{1\textwidth}
        \centering
        \textbf{Cifar10:} \textbf{\AdvTrain} with $L_\infty$ \PGD-\FConf, $\epsilon = 0.03$ for training \emph{and} testing
    \end{subfigure}
    \\[2px]
    \begin{subfigure}{0.19\textwidth}
        \includegraphics[height=2.2cm]{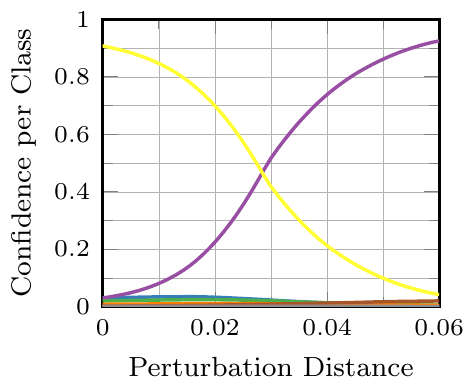}
    \end{subfigure}
    \begin{subfigure}{0.19\textwidth}
        \includegraphics[height=2.2cm]{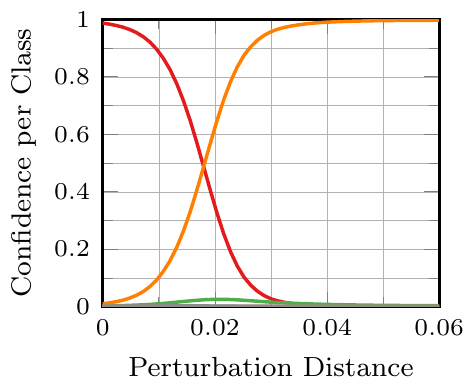}
    \end{subfigure}
    \begin{subfigure}{0.19\textwidth}
        \includegraphics[height=2.2cm]{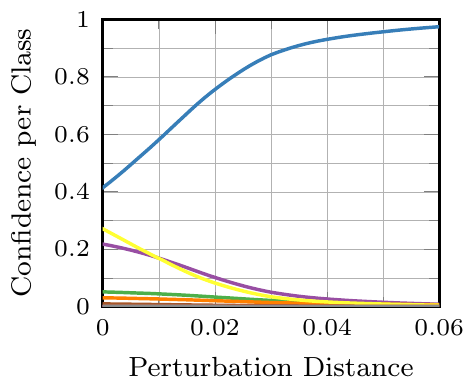}
    \end{subfigure}
    \begin{subfigure}{0.19\textwidth}
        \includegraphics[height=2.2cm]{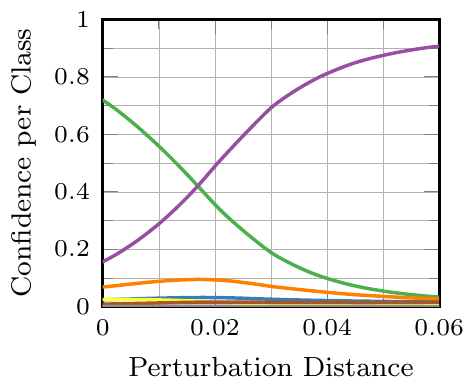}
    \end{subfigure}
    \begin{subfigure}{0.19\textwidth}
        \includegraphics[height=2.2cm]{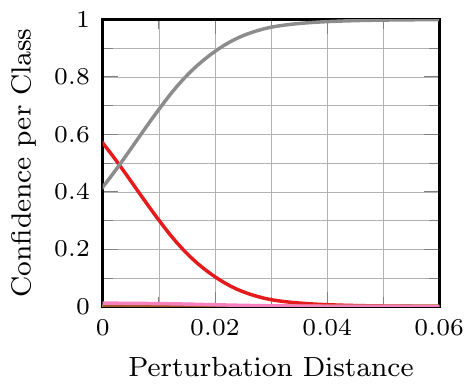}
    \end{subfigure}
    \\[4px]
    \begin{subfigure}{1\textwidth}
        \centering
        \textbf{Cifar10:} \textbf{\ConfTrain} with $L_\infty$ \PGD-\FConf, $\epsilon = 0.03$ for training \emph{and} testing
    \end{subfigure}
    \\[2px]
    \begin{subfigure}{0.19\textwidth}
        \includegraphics[height=2.2cm]{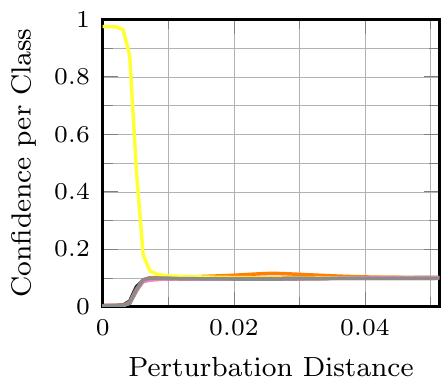}
    \end{subfigure}
    \begin{subfigure}{0.19\textwidth}
        \includegraphics[height=2.2cm]{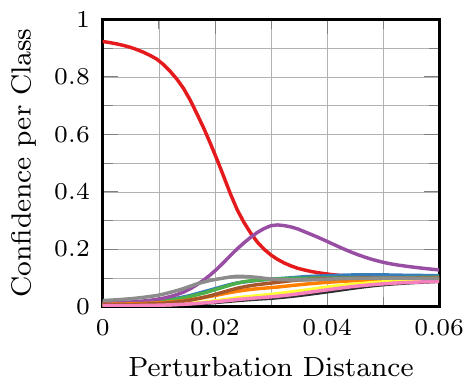}
    \end{subfigure}
    \begin{subfigure}{0.19\textwidth}
        \includegraphics[height=2.2cm]{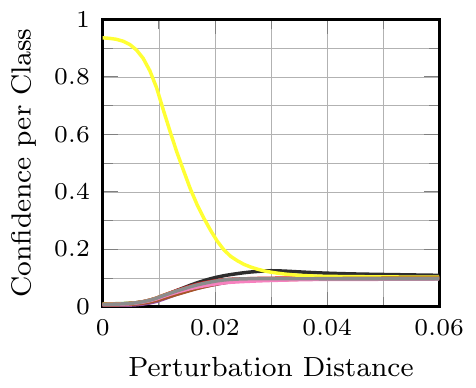}
    \end{subfigure}
    \begin{subfigure}{0.19\textwidth}
        \includegraphics[height=2.2cm]{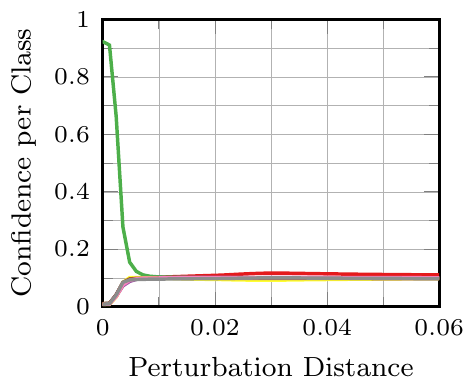}
    \end{subfigure}
    \begin{subfigure}{0.19\textwidth}
        \includegraphics[height=2.2cm]{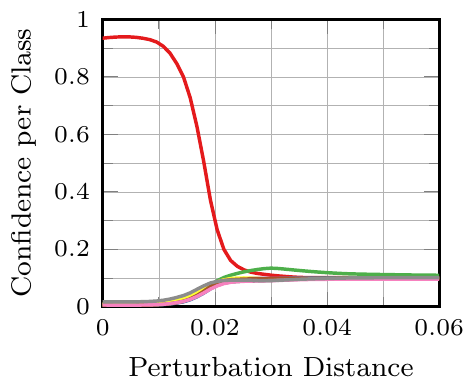}
    \end{subfigure}
    \vskip 4px
    \fbox{
        \hspace*{3.75cm}\includegraphics[width=0.5\textwidth]{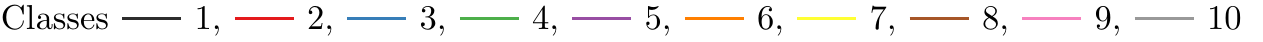}\hspace*{3.75cm}
    }
    \caption{\textbf{Effect of Confidence Calibration.} Confidences for classes along adversarial directions for \AdvTrain and \ConfTrain. Adversarial examples were computed using \PGD-\FConf with $T=1000$ iterations and zero initialization. For both \AdvTrain and \ConfTrain, we show the first ten examples of the test set on SVHN, and the first five examples of the test set on Cifar10. As can be seen, \ConfTrain biases the network to predict uniform distributions beyond the $\epsilon$-ball used during training ($\epsilon = 0.03$). For \AdvTrain, in contrast, adversarial examples can usually be found right beyond the $\epsilon$-ball.}
    \label{fig:supp-experiments-analysis}
\end{figure*}
\begin{figure*}[t]
    \centering
    \footnotesize
    \begin{subfigure}{1\textwidth}
        \centering
        \textbf{MNIST:} \textbf{\AdvTrain} with $L_\infty$ \PGD-\FConf, $\epsilon = 0.3$ for training
    \end{subfigure}
    \\[4px]
    \begin{subfigure}{0.225\textwidth}
        \includegraphics[height=2cm]{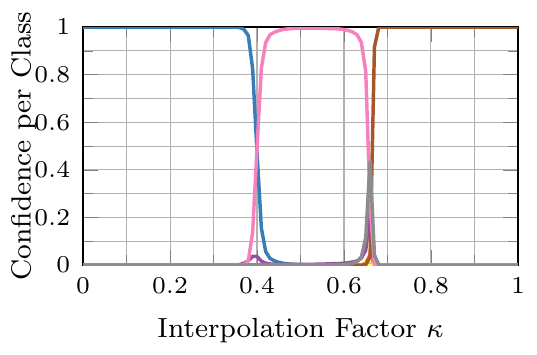}
    \end{subfigure}
    \begin{subfigure}{0.085\textwidth}
        \centering
        \includegraphics[height=1.25cm]{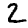}\\
        \tiny $y{=}2$\\
        $f_y{=}1$
    \end{subfigure}
    \begin{subfigure}{0.085\textwidth}
        \centering
        \includegraphics[height=1.25cm]{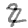}\\
        \tiny $\tilde{y}{=}8$\\
        $f_{\tilde{y}}{=}0.995$
    \end{subfigure}
    \begin{subfigure}{0.085\textwidth}
        \centering
        \includegraphics[height=1.25cm]{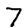}\\
        \tiny $y{=}7$\\
        $f_y{=}1$
    \end{subfigure}
    \begin{subfigure}{0.225\textwidth}
        \includegraphics[height=2cm]{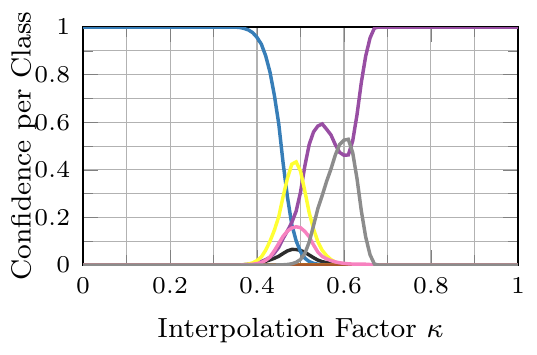}
    \end{subfigure}
    \begin{subfigure}{0.085\textwidth}
        \centering
        \includegraphics[height=1.25cm]{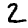}\\
        \tiny $y{=}2$\\
        $f_y{=}1$
    \end{subfigure}
    \begin{subfigure}{0.085\textwidth}
        \centering
        \includegraphics[height=1.25cm]{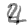}\\
        \tiny $\tilde{y}{=}6$\\
        $f_{\tilde{y}}{=}0.4$
    \end{subfigure}
    \begin{subfigure}{0.085\textwidth}
        \centering
        \includegraphics[height=1.25cm]{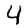}\\
        \tiny $y{=}4$\\
        $f_y{=}1$
    \end{subfigure}
    \\[4px]
    \begin{subfigure}{1\textwidth}
        \centering
        \textbf{MNIST:} \textbf{\ConfTrain} with $L_\infty$ \PGD-\FConf, $\epsilon = 0.3$ for training
    \end{subfigure}
    \\[4px]
    \begin{subfigure}{0.225\textwidth}
        \includegraphics[height=2cm]{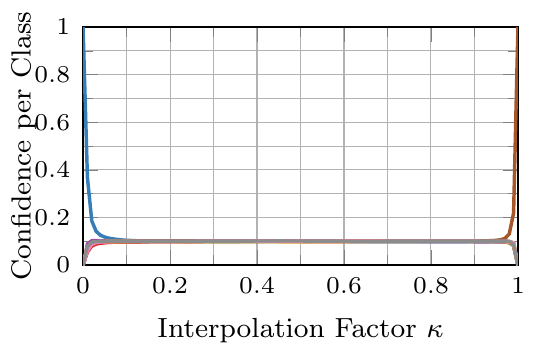}
    \end{subfigure}
    \begin{subfigure}{0.085\textwidth}
        \centering
        \includegraphics[height=1.25cm]{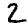}\\
        \tiny $y{=}2$\\
        $f_y{=}1$
    \end{subfigure}
    \begin{subfigure}{0.085\textwidth}
        \centering
        \includegraphics[height=1.25cm]{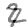}\\
        \tiny $\tilde{y}{=}4$\\
        {\color{red}$f_{\tilde{y}}{=}0.11$}
    \end{subfigure}
    \begin{subfigure}{0.085\textwidth}
        \centering
        \includegraphics[height=1.25cm]{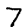}\\
        \tiny $y{=}7$\\
        $f_y{=}1$
    \end{subfigure}
    \begin{subfigure}{0.225\textwidth}
        \includegraphics[height=2cm]{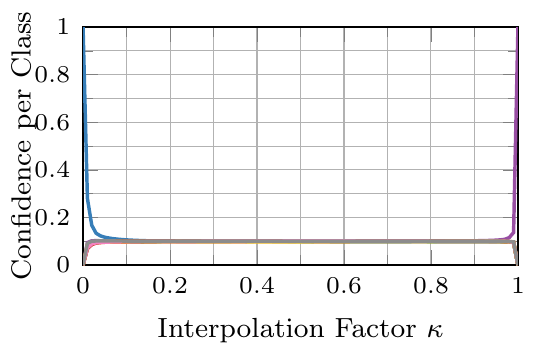}
    \end{subfigure}
    \begin{subfigure}{0.085\textwidth}
        \centering
        \includegraphics[height=1.25cm]{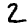}\\
        \tiny $y{=}2$\\
        $f_y{=}1$
    \end{subfigure}
    \begin{subfigure}{0.085\textwidth}
        \centering
        \includegraphics[height=1.25cm]{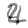}\\
        \tiny $\tilde{y}{=}0$\\
        {\color{red}$f_{\tilde{y}}{=}0.12$}
    \end{subfigure}
    \begin{subfigure}{0.085\textwidth}
        \centering
        \includegraphics[height=1.25cm]{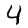}\\
        \tiny $y{=}4$\\
        $f_y{=}1$
    \end{subfigure}
    \vskip 4px
    \fbox{
        \hspace*{1.9cm}\includegraphics[width=0.7\textwidth]{fig_supp_class_legend}\hspace*{1.9cm}
    }
    
    \caption{\textbf{Confidence Calibration Between Test Examples.} We plot the confidence for all classes when interpolating linearly between test examples: $(1 - \kappa) x_1 + \kappa x_2$ for two test examples $x_1$ and $x_2$ with $\kappa \in [0,1]$; $x_1$ is fixed and we show two examples corresponding to different $x_2$. Additionally, we show the corresponding images for $\kappa = 0$, \ie, $x_1$, $\kappa = 0.5$, \ie, the mean image, and $\kappa = 1$, \ie, $x_2$, with the corresponding labels and confidences. As can be seen, \ConfTrain is able to perfectly predict a uniform distribution between test examples. \AdvTrain, in contrast, enforces high-confidence predictions, resulting in conflicts if $x_1$ and $x_2$ are too close together (\ie, within one $\epsilon$-ball) or in sudden changes of the predicted class in between, as seen above.}
    \label{fig:supp-experiments-interpolation}
\end{figure*}

In the main paper, instead of reporting ROC AUC, we reported only confidence-thresholded robust test error (\RTE), which implicitly subsumes the false positive rate (FPR), at a confidence threshold of $99\%$TPR. Again, we note that this is an extremely conservative choice, allowing to reject at most $1\%$ correctly classified clean examples. In addition, comparison to other approaches is fair as the corresponding confidence threshold only depends on correctly classified clean examples, not on adversarial examples. As also seen in \figref{fig:supp-experiments-evaluation}, ROC AUC is not a practical metric to evaluate the detection/rejection of adversarial examples. This is because rejecting a significant part of correctly classified clean examples is not acceptable. In this sense, ROC AUC measures how well positives and negatives can be distinguished in general, while we are only interested in the performance for very high TPR, \eg, $99\%$TPR as in the main paper. In this document, we also report FPR to complement our evaluation using (confidence-thresholded) \RTE.

\textbf{Robust Test Error:} 
The standard robust test error \cite{MadryICLR2018} is the model's test error in the case where all test examples are allowed to be attacked, \ie, modified within the chosen threat model, \eg, for $L_p$:
\begin{align}
\text{``Standard'' }\RTE = \frac{1}{N}\sum_{n = 1}^N \;\max\limits_{\|\delta\|_p\leq \epsilon} \Id_{f(x_n + \delta)\neq y_n}\label{eq:supp-rte}
\end{align}
where $\{(x_n,y_n)\}_{n = 1}^N$ are test examples and labels. In practice, \RTE is computed empirically using several adversarial attacks, potentially with multiple restarts as the inner maximization problem is generally non-convex. 

As standard \RTE does not account for a reject option, we propose a generalized definition adapted to our confidence-thresholded setting. For fixed confidence threshold $\tau$, \eg, at $99\%$TPR, the confidence-thresholded \RTE is defined as
\begin{align}
\RTE(\tau) = \frac{
    \sum\limits_{n=1}^N \;\max\limits_{\|\delta\|_p\leq \epsilon, c(x_n + \delta)\geq \tau} \Id_{f(x_n + \delta)\neq y_n}
}
{
    \sum\limits_{n=1}^N \;\max\limits_{\|\delta\|_p\leq \epsilon} \Id_{c(x_n + \delta)\geq \tau}
}\label{eq:supp-conf-rte}
\end{align}
with $c(x) = \max_k f_k(x)$ and $f(x)$ being the model's confidence and predicted class on example $x$, respectively. This is the test error on test examples that can be modified within the chosen threat model \emph{and} pass confidence thresholding. It reduces to standard \RTE for $\tau = 0$, is in $[0,1]$ and, thus, fully comparable to related work.

As both \eqnref{eq:supp-rte} and \eqnref{eq:supp-conf-rte} cannot be computed exactly, we compute
\begin{align}
\frac{
    \sum_{n = 1}^N \max\{\Id_{f(x_n) \neq y_n}\Id_{c(x_n)\geq \tau}, \Id_{f(\tilde{x}_n) \neq y_n}\Id_{c(\tilde{x}_n)\geq \tau}\}
}{
    \sum_{n = 1}^N \max\{\Id_{c(x_n)\geq \tau}, \Id_{c(\tilde{x}_n)\geq \tau}\}
}\label{eq:supp-appr-rte}
\end{align}
which is an upper bound assuming that our attack is perfect. Essentially, this counts the test examples $x_n$ that are either classified incorrectly with confidence $c(x_n) \geq \tau$ or that can be attacked successfully $\tilde{x}_n = x_n + \delta$ with confidence $c(\tilde{x}_n) \geq \tau$. This is normalized by the total number of test examples $x_n$ that have $c(x_n) \geq \tau$ or where the corresponding adversarial example $\tilde{x}_n$ has $c(\tilde{x}_n) \geq \tau$. It can easily be seen that $\tau = 0$ reduces \eqnref{eq:supp-appr-rte} to its unthresholded variant, \ie, standard \RTE, ensuring full comparability to related work.

In the following, we also highlight two special cases that are (correctly) taken into account by \eqnref{eq:supp-appr-rte}: (a) if a correctly classified test example $x_n$, \ie, $f(x_n) = y_n$, has confidence $c(x_n) < \tau$, \ie, is rejected, but the corresponding adversarial example $\tilde{x}_n$ with $f(\tilde{x}_n) \neq y$ has $c(\tilde{x}_n) \geq \tau$, \ie, is \emph{not} rejected, this is counted both in the numerator and denominator; (b) if an incorrectly classified test example $x_n$, \ie, $f(x_n) \neq y$, with $c(x_n) < \tau$, \ie, rejected, has a corresponding adversarial example $\tilde{x}_n$, \ie, also $f(\tilde{x}_n) \neq y$, but with $x(x_n) \geq \tau$, \ie, \emph{not} rejected, this is also counted in the numerator as well as denominator. Note that these cases are handled differently in our detection evaluation following related work \cite{MaICLR2018,LeeNIPS2018}: negatives are adversarial examples corresponding to correctly classified clean examples that are successful, \ie, change the label. For example, case (b) would not contribute towards the FPR since the original test example is already mis-classified. Thus, while \RTE implicitly includes FPR as well as \TE, it is even more conservative than just considering ``$\text{FPR} + \TE$''.

\subsection{Ablation Study}
\label{subsec:supp-experiments-ablation}

In the following, we include ablation studies for our attack \PGD-\FConf, in \tabref{tab:supp-experiments-attack}, and for \ConfTrain, in \tabref{tab:supp-experiments-training}.

\textbf{Attack.}
Regarding the proposed attack \PGD-\FConf using momentum and backtracking, \tabref{tab:supp-experiments-attack} shows that backtracking and sufficient iterations are essential to attack \ConfTrain. On SVHN, for \AdvTrain, the difference in \RTE between $T=200$ and $T=1000$ iterations is only $3.7\%$, specifically, $46.2\%$ and $49.9\%$. For \ConfTrain, in contrast, using $T=200$ iterations is not sufficient, with merely $5\%$ \RTE. However, $T=1000$ iterations with zero initialization increases \RTE to $22.8\%$. For more iterations, \ie, $T = 2000$, \RTE stagnates with $23.3\%$. When using random initialization (one restart), \RTE drops to $5.2\%$, even when using $T = 2000$ iterations. Similar significant drops are observed without backtracking. These observations generalize to MNIST and Cifar10.

\textbf{Training:}
\tabref{tab:supp-experiments-training} reports results for \ConfTrain with different values for $\rho$. We note that $\rho$ controls the (speed of the) transition from (correct) one-hot distribution to uniform distribution depending on the distance of adversarial example to the corresponding original training example. Here, higher $\rho$ results in a sharper (\ie, faster) transition from one-hot to uniform distribution. It is also important to note that the power transition does not preserve a bias towards the true label, \ie, for the maximum possible perturbation ($\|\delta\|_\infty = \epsilon$), the network is forced to predict a purely uniform distribution. As can be seen, both on SVHN and Cifar10, higher $\rho$ usually results in better robustness. Thus, for the main paper, we chose $\rho = 10$. Only on SVHN, $\rho = 6$ of $\rho = 12$ perform slightly better. However, we found that $\rho = 10$ generalizes better to previously unseen attacks.

\begin{table*}[t]
	\centering
	\footnotesize
	\begin{subfigure}{0.84\textwidth}
		\centering
		\begin{tabularx}{1\textwidth}{|X|@{ }c@{ }|@{ }c@{ }|@{ }c@{ }|@{ }c@{ }|@{ }c@{ }|@{ }c@{ }|@{ }c@{ }|@{ }c@{ }|@{ }c@{ }|@{ }c@{ }|@{ }c@{ }|@{ }c@{ }|}
\hline
\multicolumn{13}{|c|}{\textbf{MNIST:} \textbf{FPR} and \textbf{\RTE} in \% for $\tau$@$99\%$TPR}\\
\hline
& \multicolumn{2}{@{}c@{}|@{ }}{\begin{tabular}{@{}c@{}}$L_\infty$\\$\epsilon = 0.3$\end{tabular}}
& \multicolumn{2}{@{}c@{}|@{ }}{\begin{tabular}{@{}c@{}}$L_\infty$\\$\epsilon = 0.4$\end{tabular}}
& \multicolumn{2}{@{}c@{}|@{ }}{\begin{tabular}{@{}c@{}}$L_2$\\$\epsilon = 3$\end{tabular}}
& \multicolumn{2}{@{}c@{}|@{ }}{\begin{tabular}{@{}c@{}}$L_1$\\$\epsilon = 18$\end{tabular}}
& \multicolumn{2}{@{}c@{}|@{ }}{\begin{tabular}{@{}c@{}}$L_0$\\$\epsilon = 15$\end{tabular}}
& \multicolumn{2}{c|}{\begin{tabular}{@{}c@{}}adv.\\frames\end{tabular}}
\\\hline
& \multicolumn{2}{@{}c@{}|@{ }}{\textcolor{colorbrewer3}{seen}}
& \multicolumn{2}{@{}c@{}|@{ }}{\textcolor{colorbrewer1}{\bfseries unseen}}
& \multicolumn{2}{@{}c@{}|@{ }}{\textcolor{colorbrewer1}{\bfseries unseen}}
& \multicolumn{2}{@{}c@{}|@{ }}{\textcolor{colorbrewer1}{\bfseries unseen}}
& \multicolumn{2}{@{}c@{}|@{ }}{\textcolor{colorbrewer1}{\bfseries unseen}}
& \multicolumn{2}{c|}{\textcolor{colorbrewer1}{\bfseries unseen}}
\\\hline
& FPR$\downarrow$ & \RTE$\downarrow$
& FPR$\downarrow$ & \RTE$\downarrow$
& FPR$\downarrow$ & \RTE$\downarrow$
& FPR$\downarrow$ & \RTE$\downarrow$
& FPR$\downarrow$ & \RTE$\downarrow$
& FPR$\downarrow$ & \RTE$\downarrow$
\\\hline\hline
\Normal & 99.3 & 100.0 
& 99.3 & 100.0 
& 99.3 & 100.0 
& 99.3 & 100.0 
& 91.6 & 92.3 
& 87.0 & 87.7 
\\
\AdvTrainHalf & 1.0 & 1.7 
& 99.3 & 100.0 
& 80.7 & 81.5 
& 23.8 & 24.6 
& 23.0 & 23.9 
& 72.9 & 73.7 
\\
\hline
\end{tabularx}
	\end{subfigure}
	\vskip 2px
	\begin{subfigure}{0.84\textwidth}
		\centering
		\begin{tabularx}{1\textwidth}{|X|@{ }c@{ }|@{ }c@{ }|@{ }c@{ }|@{ }c@{ }|@{ }c@{ }|@{ }c@{ }|@{ }c@{ }|@{ }c@{ }|@{ }c@{ }|@{ }c@{ }|@{ }c@{ }|@{ }c@{ }|}
\hline
\multicolumn{13}{|c|}{\textbf{SVHN:} \textbf{FPR} and \textbf{\RTE} in \% for $\tau$@$99\%$TPR}\\
\hline
& \multicolumn{2}{@{}c@{}|@{ }}{\begin{tabular}{@{}c@{}}$L_\infty$\\$\epsilon = 0.03$\end{tabular}}
& \multicolumn{2}{@{}c@{}|@{ }}{\begin{tabular}{@{}c@{}}$L_\infty$\\$\epsilon = 0.06$\end{tabular}}
& \multicolumn{2}{@{}c@{}|@{ }}{\begin{tabular}{@{}c@{}}$L_2$\\$\epsilon = 2$\end{tabular}}
& \multicolumn{2}{@{}c@{}|@{ }}{\begin{tabular}{@{}c@{}}$L_1$\\$\epsilon = 24$\end{tabular}}
& \multicolumn{2}{@{}c@{}|@{ }}{\begin{tabular}{@{}c@{}}$L_0$\\$\epsilon = 10$\end{tabular}}
& \multicolumn{2}{c|}{\begin{tabular}{@{}c@{}}adv.\\frames\end{tabular}}
\\\hline
& \multicolumn{2}{@{}c@{}|@{ }}{\textcolor{colorbrewer3}{seen}}
& \multicolumn{2}{@{}c@{}|@{ }}{\textcolor{colorbrewer1}{\bfseries unseen}}
& \multicolumn{2}{@{}c@{}|@{ }}{\textcolor{colorbrewer1}{\bfseries unseen}}
& \multicolumn{2}{@{}c@{}|@{ }}{\textcolor{colorbrewer1}{\bfseries unseen}}
& \multicolumn{2}{@{}c@{}|@{ }}{\textcolor{colorbrewer1}{\bfseries unseen}}
& \multicolumn{2}{c|}{\textcolor{colorbrewer1}{\bfseries unseen}}
\\\hline
& FPR$\downarrow$ & \RTE$\downarrow$
& FPR$\downarrow$ & \RTE$\downarrow$
& FPR$\downarrow$ & \RTE$\downarrow$
& FPR$\downarrow$ & \RTE$\downarrow$
& FPR$\downarrow$ & \RTE$\downarrow$
& FPR$\downarrow$ & \RTE$\downarrow$
\\\hline\hline
\Normal & 95.8 & 99.9 
& 95.9 & 100.0 
& 95.9 & 100.0 
& 95.9 & 100.0 
& 79.6 & 83.7 
& 74.6 & 78.7 
\\
\AdvTrainHalf & 52.3 & 56.0 
& 84.7 & 88.4 
& 95.7 & 99.4 
& 95.8 & 99.5 
& 70.0 & 73.6 
& 30.0 & 33.6 
\\
\hline
\end{tabularx}
	\end{subfigure}
	\vskip 2px
	\begin{subfigure}{0.84\textwidth}
		\centering
		\begin{tabularx}{1\textwidth}{|X|@{ }c@{ }|@{ }c@{ }|@{ }c@{ }|@{ }c@{ }|@{ }c@{ }|@{ }c@{ }|@{ }c@{ }|@{ }c@{ }|@{ }c@{ }|@{ }c@{ }|@{ }c@{ }|@{ }c@{ }|}
\hline
\multicolumn{13}{|c|}{\textbf{CIFAR10:} \textbf{FPR} and \textbf{\RTE} in \% for $\tau$@$99\%$TPR}\\
\hline
& \multicolumn{2}{@{}c@{}|@{ }}{\begin{tabular}{@{}c@{}}$L_\infty$\\$\epsilon = 0.03$\end{tabular}}
& \multicolumn{2}{@{}c@{}|@{ }}{\begin{tabular}{@{}c@{}}$L_\infty$\\$\epsilon = 0.06$\end{tabular}}
& \multicolumn{2}{@{}c@{}|@{ }}{\begin{tabular}{@{}c@{}}$L_2$\\$\epsilon = 2$\end{tabular}}
& \multicolumn{2}{@{}c@{}|@{ }}{\begin{tabular}{@{}c@{}}$L_1$\\$\epsilon = 24$\end{tabular}}
& \multicolumn{2}{@{}c@{}|@{ }}{\begin{tabular}{@{}c@{}}$L_0$\\$\epsilon = 10$\end{tabular}}
& \multicolumn{2}{c|}{\begin{tabular}{@{}c@{}}adv.\\frames\end{tabular}}
\\\hline
& \multicolumn{2}{@{}c@{}|@{ }}{\textcolor{colorbrewer3}{seen}}
& \multicolumn{2}{@{}c@{}|@{ }}{\textcolor{colorbrewer1}{\bfseries unseen}}
& \multicolumn{2}{@{}c@{}|@{ }}{\textcolor{colorbrewer1}{\bfseries unseen}}
& \multicolumn{2}{@{}c@{}|@{ }}{\textcolor{colorbrewer1}{\bfseries unseen}}
& \multicolumn{2}{@{}c@{}|@{ }}{\textcolor{colorbrewer1}{\bfseries unseen}}
& \multicolumn{2}{c|}{\textcolor{colorbrewer1}{\bfseries unseen}}
\\\hline
& FPR$\downarrow$ & \RTE$\downarrow$
& FPR$\downarrow$ & \RTE$\downarrow$
& FPR$\downarrow$ & \RTE$\downarrow$
& FPR$\downarrow$ & \RTE$\downarrow$
& FPR$\downarrow$ & \RTE$\downarrow$
& FPR$\downarrow$ & \RTE$\downarrow$
\\\hline\hline
\Normal & 93.0 & 100.0 
& 93.0 & 100.0 
& 93.0 & 100.0 
& 93.0 & 100.0 
& 77.7 & 84.7 
& 89.7 & 96.7 
\\
\AdvTrainHalf & 47.6 & 62.7 
& 78.6 & 93.7 
& 83.3 & 98.4 
& 83.3 & 98.4 
& 59.3 & 74.4 
& 63.6 & 78.7 
\\
\hline
\end{tabularx}
	\end{subfigure}
	\vskip -6px
	\caption{\textbf{Main Results: FPR for $99\%$TPR.} For $\boldsymbol{99}\%$TPR, we report confidence-thresholded \RTE \emph{and} FPR for the results from the main paper. We emphasize that only \PGD-\FCE and \PGD-\FConf were used against \Ma and \Lee. In general, the observations of the main paper can be confirmed considering FPR. Due to the poor \TE of \AdvTrain, \Wong or \TRADES on Cifar10, these methods benefit most from considering FPR instead of (confidence-thresholded) \RTE. \textbf{*} Pre-trained models with different architectures.}
	\label{tab:supp-experiments-99}
\end{table*}
\begin{table*}[t]
    \footnotesize
    \begin{subfigure}{0.83\textwidth}
        \centering
        \begin{tabularx}{1\textwidth}{|X|@{ }c@{ }|@{ }c@{ }|@{ }c@{ }|@{ }c@{ }|@{ }c@{ }|@{ }c@{ }|@{ }c@{ }|@{ }c@{ }|@{ }c@{ }|@{ }c@{ }|@{ }c@{ }|@{ }c@{ }|}
\hline
\multicolumn{13}{|c|}{\textbf{MNIST:} \textbf{FPR} and confidence-thresholded \textbf{\RTE} in \% for $\tau$@{\color{colorbrewer1}$\mathbf{98}\%$TPR}}\\
\hline
& \multicolumn{2}{@{}c@{}|@{ }}{\begin{tabular}{@{}c@{}}$L_\infty$\\$\epsilon = 0.3$\end{tabular}}
& \multicolumn{2}{@{}c@{}|@{ }}{\begin{tabular}{@{}c@{}}$L_\infty$\\$\epsilon = 0.4$\end{tabular}}
& \multicolumn{2}{@{}c@{}|@{ }}{\begin{tabular}{@{}c@{}}$L_2$\\$\epsilon = 3$\end{tabular}}
& \multicolumn{2}{@{}c@{}|@{ }}{\begin{tabular}{@{}c@{}}$L_1$\\$\epsilon = 18$\end{tabular}}
& \multicolumn{2}{@{}c@{}|@{ }}{\begin{tabular}{@{}c@{}}$L_0$\\$\epsilon = 15$\end{tabular}}
& \multicolumn{2}{c|}{\begin{tabular}{@{}c@{}}adv.\\frames\end{tabular}}
\\\hline
& \multicolumn{2}{@{}c@{}|@{ }}{\textcolor{colorbrewer3}{seen}}
& \multicolumn{2}{@{}c@{}|@{ }}{\textcolor{colorbrewer1}{\bfseries unseen}}
& \multicolumn{2}{@{}c@{}|@{ }}{\textcolor{colorbrewer1}{\bfseries unseen}}
& \multicolumn{2}{@{}c@{}|@{ }}{\textcolor{colorbrewer1}{\bfseries unseen}}
& \multicolumn{2}{@{}c@{}|@{ }}{\textcolor{colorbrewer1}{\bfseries unseen}}
& \multicolumn{2}{c|}{\textcolor{colorbrewer1}{\bfseries unseen}}
\\\hline
& FPR$\downarrow$ & \RTE$\downarrow$
& FPR$\downarrow$ & \RTE$\downarrow$
& FPR$\downarrow$ & \RTE$\downarrow$
& FPR$\downarrow$ & \RTE$\downarrow$
& FPR$\downarrow$ & \RTE$\downarrow$
& FPR$\downarrow$ & \RTE$\downarrow$
\\\hline\hline
\Normal & 99.3 & 100.0 
& 99.3 & 100.0 
& 99.3 & 100.0 
& 99.3 & 100.0 
& 87.3 & 88.1 
& 79.8 & 80.5 
\\
\AdvTrainHalf & 0.5 & 0.8 
& 99.3 & 100.0 
& 66.7 & 67.9 
& 16.3 & 17.5 
& 16.1 & 17.2 
& 61.3 & 62.5 
\\
\hline
\end{tabularx}
    \end{subfigure}
    \begin{subfigure}{0.08\textwidth}
        \centering
        \begin{tabular}{|c|}
\hline
\\
\hline
\begin{tabular}{@{}c@{}}distal\\\vphantom{t}\end{tabular}\\
\hline
\textcolor{colorbrewer1}{\bfseries unseen}\\
\hline
FPR$\downarrow$\\
\hline
\hline
 100.0 
\\
100.0 
\\
100.0 
\\
0.0 
\\\hline\hline
100.0 
\\
100.0 
\\\hline
\end{tabular}

    \end{subfigure}
    \begin{subfigure}{0.075\textwidth}
        \centering
        \begin{tabular}{|c|}
\hline
\\
\hline
\begin{tabular}{@{}c@{}}corr.\\\tiny MNIST-C\end{tabular}\\
\hline
\textcolor{colorbrewer1}{\bfseries unseen}\\
\hline
\TE$\downarrow$\\
\hline\hline
31.0\\ 
12.3\\ 
15.4\\ 
5.3\\ 
\hline\hline
5.6\\ 
5.7\\ 
\hline
\end{tabular}

    \end{subfigure}
    \vskip 2px
    \begin{subfigure}{0.83\textwidth}
        \centering
        \begin{tabularx}{1\textwidth}{|X|@{ }c@{ }|@{ }c@{ }|@{ }c@{ }|@{ }c@{ }|@{ }c@{ }|@{ }c@{ }|@{ }c@{ }|@{ }c@{ }|@{ }c@{ }|@{ }c@{ }|@{ }c@{ }|@{ }c@{ }|}
\hline
\multicolumn{13}{|c|}{\textbf{SVHN:} \textbf{FPR} and confidence-thresholded \textbf{\RTE} in \% for $\tau$@{\color{colorbrewer1}$\mathbf{98}\%$TPR}}\\
\hline
& \multicolumn{2}{@{}c@{}|@{ }}{\begin{tabular}{@{}c@{}}$L_\infty$\\$\epsilon = 0.03$\end{tabular}}
& \multicolumn{2}{@{}c@{}|@{ }}{\begin{tabular}{@{}c@{}}$L_\infty$\\$\epsilon = 0.06$\end{tabular}}
& \multicolumn{2}{@{}c@{}|@{ }}{\begin{tabular}{@{}c@{}}$L_2$\\$\epsilon = 2$\end{tabular}}
& \multicolumn{2}{@{}c@{}|@{ }}{\begin{tabular}{@{}c@{}}$L_1$\\$\epsilon = 24$\end{tabular}}
& \multicolumn{2}{@{}c@{}|@{ }}{\begin{tabular}{@{}c@{}}$L_0$\\$\epsilon = 10$\end{tabular}}
& \multicolumn{2}{c|}{\begin{tabular}{@{}c@{}}adv.\\frames\end{tabular}}
\\\hline
& \multicolumn{2}{@{}c@{}|@{ }}{\textcolor{colorbrewer3}{seen}}
& \multicolumn{2}{@{}c@{}|@{ }}{\textcolor{colorbrewer1}{\bfseries unseen}}
& \multicolumn{2}{@{}c@{}|@{ }}{\textcolor{colorbrewer1}{\bfseries unseen}}
& \multicolumn{2}{@{}c@{}|@{ }}{\textcolor{colorbrewer1}{\bfseries unseen}}
& \multicolumn{2}{@{}c@{}|@{ }}{\textcolor{colorbrewer1}{\bfseries unseen}}
& \multicolumn{2}{c|}{\textcolor{colorbrewer1}{\bfseries unseen}}
\\\hline
& FPR$\downarrow$ & \RTE$\downarrow$
& FPR$\downarrow$ & \RTE$\downarrow$
& FPR$\downarrow$ & \RTE$\downarrow$
& FPR$\downarrow$ & \RTE$\downarrow$
& FPR$\downarrow$ & \RTE$\downarrow$
& FPR$\downarrow$ & \RTE$\downarrow$
\\\hline\hline
\Normal & 95.7 & 99.8 
& 95.9 & 100.0 
& 95.9 & 100.0 
& 95.9 & 100.0 
& 72.7 & 76.8 
& 72.4 & 76.5 
\\
\AdvTrainHalf & 50.0 & 53.7 
& 83.4 & 87.1 
& 95.5 & 99.2 
& 95.7 & 99.4 
& 54.0 & 57.9 
& 26.8 & 30.6 
\\
\hline
\end{tabularx}
    \end{subfigure}
    \begin{subfigure}{0.08\textwidth}
        \centering
        \begin{tabular}{|c|}
\hline
\\
\hline
\begin{tabular}{@{}c@{}}distal\\\vphantom{t}\end{tabular}\\
\hline
\textcolor{colorbrewer1}{\bfseries unseen}\\
\hline
FPR$\downarrow$\\
\hline
\hline
87.1 
\\
86.3 
\\
81.0 
\\
0.0 
\\\hline
\end{tabular}

    \end{subfigure}
    \hfill
    \vskip 2px
    \begin{subfigure}{0.83\textwidth}
        \centering
        \begin{tabularx}{1\textwidth}{|X|@{ }c@{ }|@{ }c@{ }|@{ }c@{ }|@{ }c@{ }|@{ }c@{ }|@{ }c@{ }|@{ }c@{ }|@{ }c@{ }|@{ }c@{ }|@{ }c@{ }|@{ }c@{ }|@{ }c@{ }|}
\hline
\multicolumn{13}{|c|}{\textbf{CIFAR10:} \textbf{FPR} and confidence-thresholded \textbf{\RTE} in \% for $\tau$@{\color{colorbrewer1}$\mathbf{98}\%$TPR}}\\
\hline
& \multicolumn{2}{@{}c@{}|@{ }}{\begin{tabular}{@{}c@{}}$L_\infty$\\$\epsilon = 0.03$\end{tabular}}
& \multicolumn{2}{@{}c@{}|@{ }}{\begin{tabular}{@{}c@{}}$L_\infty$\\$\epsilon = 0.06$\end{tabular}}
& \multicolumn{2}{@{}c@{}|@{ }}{\begin{tabular}{@{}c@{}}$L_2$\\$\epsilon = 2$\end{tabular}}
& \multicolumn{2}{@{}c@{}|@{ }}{\begin{tabular}{@{}c@{}}$L_1$\\$\epsilon = 24$\end{tabular}}
& \multicolumn{2}{@{}c@{}|@{ }}{\begin{tabular}{@{}c@{}}$L_0$\\$\epsilon = 10$\end{tabular}}
& \multicolumn{2}{c|}{\begin{tabular}{@{}c@{}}adv.\\frames\end{tabular}}
\\\hline
& \multicolumn{2}{@{}c@{}|@{ }}{\textcolor{colorbrewer3}{seen}}
& \multicolumn{2}{@{}c@{}|@{ }}{\textcolor{colorbrewer1}{\bfseries unseen}}
& \multicolumn{2}{@{}c@{}|@{ }}{\textcolor{colorbrewer1}{\bfseries unseen}}
& \multicolumn{2}{@{}c@{}|@{ }}{\textcolor{colorbrewer1}{\bfseries unseen}}
& \multicolumn{2}{@{}c@{}|@{ }}{\textcolor{colorbrewer1}{\bfseries unseen}}
& \multicolumn{2}{c|}{\textcolor{colorbrewer1}{\bfseries unseen}}
\\\hline
& FPR$\downarrow$ & \RTE$\downarrow$
& FPR$\downarrow$ & \RTE$\downarrow$
& FPR$\downarrow$ & \RTE$\downarrow$
& FPR$\downarrow$ & \RTE$\downarrow$
& FPR$\downarrow$ & \RTE$\downarrow$
& FPR$\downarrow$ & \RTE$\downarrow$
\\\hline\hline
\Normal & 93.0 & 100.0 
& 93.0 & 100.0 
& 93.0 & 100.0 
& 93.0 & 100.0 
& 70.1 & 77.1 
& 89.6 & 96.6 
\\
\AdvTrainHalf & 47.6 & 62.7 
& 78.6 & 93.7 
& 83.3 & 98.4 
& 83.3 & 98.4 
& 57.2 & 72.4 
& 63.6 & 78.7 
\\
\hline
\end{tabularx}
    \end{subfigure}
    \begin{subfigure}{0.08\textwidth}
        \centering
        \begin{tabular}{|c|}
\hline
\\
\hline
\begin{tabular}{@{}c@{}}corr.\\\vphantom{t}\end{tabular}\\
\hline
\textcolor{colorbrewer1}{\bfseries unseen}\\
\hline
FPR$\downarrow$\\
\hline
\hline
83.3 
\\
75.0 
\\
72.5 
\\
0.0 
\\\hline\hline
76.7 
\\
76.2 
\\
78.5 
\\\hline
\end{tabular}

    \end{subfigure}
    \begin{subfigure}{0.075\textwidth}
        \centering
        \begin{tabular}{|c|}
\hline
\\
\hline
\begin{tabular}{@{}c@{}}corr.\\\tiny CIFAR10-C\end{tabular}\\
\hline
\textcolor{colorbrewer1}{\bfseries unseen}\\
\hline
\TE$\downarrow$\\
\hline\hline
11.4\\ 
15.1\\ 
17.8\\ 
8.1\\ 
\hline\hline
17.1\\ 
14.1\\ 
11.4\\ 
\hline
\end{tabular}

    \end{subfigure}
    \vskip -6px
    \caption{\textbf{Main Results: Generalizable Robustness for {\color{colorbrewer1}$\boldsymbol{98\%}$TPR}.} While reporting results for $99\%$TPR in the main paper, reducing the TPR requirement for confidence-thresholding to {\color{colorbrewer1}$98\%$TPR} generally improves results, but only slightly. We report FPR and confidence-thresholded \RTE for $98\%$TPR. For MNIST-C and Cifar10-C, we report mean \TE across all corruptions. $L_\infty$ attacks with $\epsilon{=}0.3$ on MNIST and $\epsilon = 0.03$ on SVHN/Cifar10 were used for training (\textbf{\textcolor{colorbrewer3}{seen}}). All other attacks were not used during training (\textbf{\textcolor{colorbrewer1}{unseen}}). \textbf{*} Pre-trained models with different architectures.}
    \label{tab:supp-experiments-98}
\end{table*}
\begin{table*}[t]
    \footnotesize
    \begin{subfigure}{0.83\textwidth}
        \centering
        \begin{tabularx}{1\textwidth}{|X|@{ }c@{ }|@{ }c@{ }|@{ }c@{ }|@{ }c@{ }|@{ }c@{ }|@{ }c@{ }|@{ }c@{ }|@{ }c@{ }|@{ }c@{ }|@{ }c@{ }|@{ }c@{ }|@{ }c@{ }|}
\hline
\multicolumn{13}{|c|}{\textbf{MNIST:} \textbf{FPR} and confidence-thresholded \textbf{\RTE} in \% for $\tau$@{\color{colorbrewer1}$\mathbf{95}\%$TPR}}\\
\hline
& \multicolumn{2}{@{}c@{}|@{ }}{\begin{tabular}{@{}c@{}}$L_\infty$\\$\epsilon = 0.3$\end{tabular}}
& \multicolumn{2}{@{}c@{}|@{ }}{\begin{tabular}{@{}c@{}}$L_\infty$\\$\epsilon = 0.4$\end{tabular}}
& \multicolumn{2}{@{}c@{}|@{ }}{\begin{tabular}{@{}c@{}}$L_2$\\$\epsilon = 3$\end{tabular}}
& \multicolumn{2}{@{}c@{}|@{ }}{\begin{tabular}{@{}c@{}}$L_1$\\$\epsilon = 18$\end{tabular}}
& \multicolumn{2}{@{}c@{}|@{ }}{\begin{tabular}{@{}c@{}}$L_0$\\$\epsilon = 15$\end{tabular}}
& \multicolumn{2}{c|}{\begin{tabular}{@{}c@{}}adv.\\frames\end{tabular}}
\\\hline
& \multicolumn{2}{@{}c@{}|@{ }}{\textcolor{colorbrewer3}{seen}}
& \multicolumn{2}{@{}c@{}|@{ }}{\textcolor{colorbrewer1}{\bfseries unseen}}
& \multicolumn{2}{@{}c@{}|@{ }}{\textcolor{colorbrewer1}{\bfseries unseen}}
& \multicolumn{2}{@{}c@{}|@{ }}{\textcolor{colorbrewer1}{\bfseries unseen}}
& \multicolumn{2}{@{}c@{}|@{ }}{\textcolor{colorbrewer1}{\bfseries unseen}}
& \multicolumn{2}{c|}{\textcolor{colorbrewer1}{\bfseries unseen}}
\\\hline
& FPR$\downarrow$ & \RTE$\downarrow$
& FPR$\downarrow$ & \RTE$\downarrow$
& FPR$\downarrow$ & \RTE$\downarrow$
& FPR$\downarrow$ & \RTE$\downarrow$
& FPR$\downarrow$ & \RTE$\downarrow$
& FPR$\downarrow$ & \RTE$\downarrow$
\\\hline\hline
\Normal & 99.3 & 100.0 
& 99.3 & 100.0 
& 99.3 & 100.0 
& 99.0 & 99.7 
& 75.6 & 76.7 
& 65.9 & 67.0 
\\
\AdvTrainHalf & 0.2 & 0.4 
& 99.3 & 100.0 
& 54.6 & 56.7 
& 12.3 & 13.7 
& 11.1 & 12.3 
& 47.3 & 49.2 
\\
\hline
\end{tabularx}
    \end{subfigure}
    \begin{subfigure}{0.08\textwidth}
        \centering
        \begin{tabular}{|c|}
\hline
\\
\hline
\begin{tabular}{@{}c@{}}distal\\\vphantom{t}\end{tabular}\\
\hline
\textcolor{colorbrewer1}{\bfseries unseen}\\
\hline
FPR$\downarrow$\\
\hline
\hline
100.0 
\\
100.0 
\\
100.0 
\\
0.0 
\\\hline\hline
100.0 
\\
100.0 
\\\hline
\end{tabular}

    \end{subfigure}
    \begin{subfigure}{0.075\textwidth}
        \centering
        \begin{tabular}{|c|}
\hline
\\
\hline
\begin{tabular}{@{}c@{}}corr.\\\tiny MNIST-C\end{tabular}\\
\hline
\textcolor{colorbrewer1}{\bfseries unseen}\\
\hline
\TE$\downarrow$\\
\hline\hline
27.5\\ 
8.6\\ 
10.5\\ 
5.5\\ 
\hline\hline
4.4\\ 
3.2\\ 
\hline
\end{tabular}

    \end{subfigure}
    \vskip 2px
    \begin{subfigure}{0.83\textwidth}
        \centering
        \begin{tabularx}{1\textwidth}{|X|@{ }c@{ }|@{ }c@{ }|@{ }c@{ }|@{ }c@{ }|@{ }c@{ }|@{ }c@{ }|@{ }c@{ }|@{ }c@{ }|@{ }c@{ }|@{ }c@{ }|@{ }c@{ }|@{ }c@{ }|}
\hline
\multicolumn{13}{|c|}{\textbf{SVHN:} \textbf{FPR} and confidence-thresholded \textbf{\RTE} in \% for $\tau$@{\color{colorbrewer1}$\mathbf{95}\%$TPR}}\\
\hline
& \multicolumn{2}{@{}c@{}|@{ }}{\begin{tabular}{@{}c@{}}$L_\infty$\\$\epsilon = 0.03$\end{tabular}}
& \multicolumn{2}{@{}c@{}|@{ }}{\begin{tabular}{@{}c@{}}$L_\infty$\\$\epsilon = 0.06$\end{tabular}}
& \multicolumn{2}{@{}c@{}|@{ }}{\begin{tabular}{@{}c@{}}$L_2$\\$\epsilon = 2$\end{tabular}}
& \multicolumn{2}{@{}c@{}|@{ }}{\begin{tabular}{@{}c@{}}$L_1$\\$\epsilon = 24$\end{tabular}}
& \multicolumn{2}{@{}c@{}|@{ }}{\begin{tabular}{@{}c@{}}$L_0$\\$\epsilon = 10$\end{tabular}}
& \multicolumn{2}{c|}{\begin{tabular}{@{}c@{}}adv.\\frames\end{tabular}}
\\\hline
& \multicolumn{2}{@{}c@{}|@{ }}{\textcolor{colorbrewer3}{seen}}
& \multicolumn{2}{@{}c@{}|@{ }}{\textcolor{colorbrewer1}{\bfseries unseen}}
& \multicolumn{2}{@{}c@{}|@{ }}{\textcolor{colorbrewer1}{\bfseries unseen}}
& \multicolumn{2}{@{}c@{}|@{ }}{\textcolor{colorbrewer1}{\bfseries unseen}}
& \multicolumn{2}{@{}c@{}|@{ }}{\textcolor{colorbrewer1}{\bfseries unseen}}
& \multicolumn{2}{c|}{\textcolor{colorbrewer1}{\bfseries unseen}}
\\\hline
& FPR$\downarrow$ & \RTE$\downarrow$
& FPR$\downarrow$ & \RTE$\downarrow$
& FPR$\downarrow$ & \RTE$\downarrow$
& FPR$\downarrow$ & \RTE$\downarrow$
& FPR$\downarrow$ & \RTE$\downarrow$
& FPR$\downarrow$ & \RTE$\downarrow$
\\\hline\hline
\Normal & 95.6 & 99.7 
& 95.9 & 100.0 
& 95.9 & 100.0 
& 95.9 & 100.0 
& 65.3 & 69.5 
& 66.8 & 71.1 
\\
\AdvTrainHalf & 45.5 & 49.2 
& 79.7 & 83.4 
& 95.4 & 99.1 
& 95.5 & 99.2 
& 34.8 & 38.9 
& 21.5 & 25.5 
\\
\hline
\end{tabularx}
    \end{subfigure}
    \begin{subfigure}{0.08\textwidth}
        \centering
        \begin{tabular}{|c|}
\hline
\\
\hline
\begin{tabular}{@{}c@{}}distal\\\vphantom{t}\end{tabular}\\
\hline
\textcolor{colorbrewer1}{\bfseries unseen}\\
\hline
FPR$\downarrow$\\
\hline
\hline
87.1 
\\
59.3 
\\
75.1 
\\
0.0 
\\\hline
\end{tabular}

    \end{subfigure}
    \hfill
    \vskip 2px
    \begin{subfigure}{0.83\textwidth}
        \centering
        \begin{tabularx}{1\textwidth}{|X|@{ }c@{ }|@{ }c@{ }|@{ }c@{ }|@{ }c@{ }|@{ }c@{ }|@{ }c@{ }|@{ }c@{ }|@{ }c@{ }|@{ }c@{ }|@{ }c@{ }|@{ }c@{ }|@{ }c@{ }|}
\hline
\multicolumn{13}{|c|}{\textbf{CIFAR10:} \textbf{FPR} and confidence-thresholded \textbf{\RTE} in \% for $\tau$@{\color{colorbrewer1}$\mathbf{95}\%$TPR}}\\
\hline
& \multicolumn{2}{@{}c@{}|@{ }}{\begin{tabular}{@{}c@{}}$L_\infty$\\$\epsilon = 0.03$\end{tabular}}
& \multicolumn{2}{@{}c@{}|@{ }}{\begin{tabular}{@{}c@{}}$L_\infty$\\$\epsilon = 0.06$\end{tabular}}
& \multicolumn{2}{@{}c@{}|@{ }}{\begin{tabular}{@{}c@{}}$L_2$\\$\epsilon = 2$\end{tabular}}
& \multicolumn{2}{@{}c@{}|@{ }}{\begin{tabular}{@{}c@{}}$L_1$\\$\epsilon = 24$\end{tabular}}
& \multicolumn{2}{@{}c@{}|@{ }}{\begin{tabular}{@{}c@{}}$L_0$\\$\epsilon = 10$\end{tabular}}
& \multicolumn{2}{c|}{\begin{tabular}{@{}c@{}}adv.\\frames\end{tabular}}
\\\hline
& \multicolumn{2}{@{}c@{}|@{ }}{\textcolor{colorbrewer3}{seen}}
& \multicolumn{2}{@{}c@{}|@{ }}{\textcolor{colorbrewer1}{\bfseries unseen}}
& \multicolumn{2}{@{}c@{}|@{ }}{\textcolor{colorbrewer1}{\bfseries unseen}}
& \multicolumn{2}{@{}c@{}|@{ }}{\textcolor{colorbrewer1}{\bfseries unseen}}
& \multicolumn{2}{@{}c@{}|@{ }}{\textcolor{colorbrewer1}{\bfseries unseen}}
& \multicolumn{2}{c|}{\textcolor{colorbrewer1}{\bfseries unseen}}
\\\hline
& FPR$\downarrow$ & \RTE$\downarrow$
& FPR$\downarrow$ & \RTE$\downarrow$
& FPR$\downarrow$ & \RTE$\downarrow$
& FPR$\downarrow$ & \RTE$\downarrow$
& FPR$\downarrow$ & \RTE$\downarrow$
& FPR$\downarrow$ & \RTE$\downarrow$
\\\hline\hline
\Normal & 93.0 & 100.0 
& 93.0 & 100.0 
& 93.0 & 100.0 
& 93.0 & 100.0 
& 52.0 & 59.0 
& 88.8 & 95.9 
\\
\AdvTrainHalf & 46.6 & 61.7 
& 78.6 & 93.7 
& 83.3 & 98.4 
& 83.3 & 98.4 
& 43.9 & 59.6 
& 62.8 & 77.9 
\\
\hline
\end{tabularx}
    \end{subfigure}
    \begin{subfigure}{0.08\textwidth}
        \centering
        \begin{tabular}{|c|}
\hline
\\
\hline
\begin{tabular}{@{}c@{}}corr.\\\vphantom{t}\end{tabular}\\
\hline
\textcolor{colorbrewer1}{\bfseries unseen}\\
\hline
FPR$\downarrow$\\
\hline
\hline
83.3 
\\
75.0 
\\
72.5 
\\
0.0 
\\\hline\hline
76.7 
\\
76.2 
\\
78.5 
\\\hline
\end{tabular}

    \end{subfigure}
    \begin{subfigure}{0.075\textwidth}
        \centering
        \begin{tabular}{|c|}
\hline
\\
\hline
\begin{tabular}{@{}c@{}}corr.\\\tiny CIFAR10-C\end{tabular}\\
\hline
\textcolor{colorbrewer1}{\bfseries unseen}\\
\hline
\TE$\downarrow$\\
\hline\hline
7.7\\ 
12.1\\ 
15.6\\ 
6.0\\ 
\hline\hline
13.7\\ 
10.7\\ 
9.1\\ 
\hline
\end{tabular}

    \end{subfigure}
    \vskip -6px
    \caption{\textbf{Main Results: Generalizable Robustness for {\color{colorbrewer1}$\boldsymbol{95\%}$TPR}.} We report FPR and \RTE for $\boldsymbol{95}\%$TPR, in comparison with $98\%$ in \tabref{tab:supp-experiments-98} and $99\%$ in the main paper. For MNIST-C and Cifar10-C, we report mean \TE across all corruptions. $L_\infty$ attacks with $\epsilon{=}0.3$ on MNIST and $\epsilon = 0.03$ on SVHN/Cifar10 \textbf{\textcolor{colorbrewer3}{seen}} during training; all other attacks \textbf{\textcolor{colorbrewer1}{unseen}} during training. Results improve slightly in comparison with $98\%$TPR. However, the improvements are rather small and do not justify the significantly increased fraction of ``thrown away'' (correctly classified) clean examples. \textbf{*} Pre-trained models with different architectures.}
    \label{tab:supp-experiments-95}
\end{table*}
\begin{table}
    \centering
    \begin{subfigure}{0.235\textwidth}
        \begin{tabular}{|@{ }l@{ }|@{ }c@{ }|@{ }c@{ }|}
    \hline
    & \multicolumn{2}{@{}c@{}|}{\textbf{MNIST:}}\\
    \hline\hline
    & \multicolumn{2}{@{}c@{}|}{\begin{tabular}{@{}c@{}}\textbf{all}\\\textbf{\textcolor{colorbrewer1}{unseen}}\\\vphantom{$L_\infty$}\end{tabular}}\\\hline
    & FPR$\downarrow$ & \RTE$\downarrow$\\
    \hline
    \hline
    \Normal & 99.3 & 100.0 
    \\
    \AdvTrainHalf & 99.3 & 100.0 
    \\
    \AdvTrainFull & 99.2 & 100.0 
    \\
    \ConfTrain & 23.4 & 23.9 
    \\\hline\hline
    \textbf{*} \Wong & 97.0 & 99.2 
    \\
    \textbf{*} \TRADES & 99.3 & 99.9 
    \\
    \textbf{*} \MadryAT & \textcolor{gray}{--} & \textcolor{gray}{--}\\
    \hline
\end{tabular}
    \end{subfigure}
    \begin{subfigure}{0.125\textwidth}
        \begin{tabular}{|@{ }c@{ }|@{ }c@{ }|}
    \hline
    \multicolumn{2}{|@{}c@{}|}{\textbf{SVHN:}}\\
    \hline\hline
    \multicolumn{2}{|@{}c@{}|}{\begin{tabular}{@{}c@{}}\textbf{all}\\\textbf{\textcolor{colorbrewer1}{unseen}}\\\vphantom{$L_\infty$}\end{tabular}}\\\hline
    FPR$\downarrow$ & \RTE$\downarrow$\\
    \hline
    \hline
    95.9 & 100.0 
    \\
    96.2 & 99.9 
    \\
    93.7 & 99.9 
    \\
    57.5 & 61.1 
    \\
    \hline\hline
    \textcolor{gray}{--} & \textcolor{gray}{--}\\ 
    \textcolor{gray}{--} & \textcolor{gray}{--}\\ 
    \textcolor{gray}{--} & \textcolor{gray}{--}\\
    \hline
\end{tabular}
    \end{subfigure}
    \begin{subfigure}{0.125\textwidth}
        \begin{tabular}{|@{ }c@{ }|@{ }c@{ }|}
    \hline
    \multicolumn{2}{|@{}c@{}|}{\textbf{CIFAR10:}}\\
    \hline\hline
    \multicolumn{2}{|@{}c@{}|}{\begin{tabular}{@{}c@{}}\textbf{all}\\\textbf{\textcolor{colorbrewer1}{unseen}}\\\vphantom{$L_\infty$}\end{tabular}}\\\hline
    FPR$\downarrow$ & \RTE$\downarrow$\\
    \hline
    \hline
    93.0 & 100.0 
    \\
    84.1 & 99.2 
    \\
    81.0 & 98.6 
    \\
    86.3 & 94.8 
    \\\hline\hline
    76.2 & 94.1 
    \\
    82.8 & 97.4 
    \\
    87.6 & 98.9 
    \\
    \hline
\end{tabular}
    \end{subfigure}
    \hspace*{0px}
    {\color{gray}\rule[-2.75cm]{1px}{5.75cm}}
    \hspace*{0.5px}
    \begin{subfigure}{0.125\textwidth}
        \begin{tabular}{|@{ }c@{ }|@{ }c@{ }|}
    \hline
    \multicolumn{2}{|@{}c@{}|}{\textbf{CIFAR10:}}\\
    \hline\hline
    \multicolumn{2}{|@{}c@{}|}{\begin{tabular}{@{}c@{}}\textbf{\textcolor{colorbrewer1}{unseen}}\\\textcolor{colorbrewer2}{\textbf{except} $\mathbf{L_\infty}$}\\\textcolor{colorbrewer2}{with $\mathbf{\epsilon{=}0.06}$}\\\end{tabular}}\\\hline
    FPR$\downarrow$ & \RTE$\downarrow$\\
    \hline
    \hline
    93.0 & 100.0 
    \\
    84.1 & 99.2 
    \\
    81.1 & 98.7 
    \\
    69.1 & 77.6 
    \\\hline\hline
    75.6 & 93.5 
    \\
    82.7 & 97.3 
    \\
    87.6 & 98.9 
    \\
    \hline
\end{tabular}
    \end{subfigure}
    \caption{\textbf{Worst-Case Results Across \textcolor{colorbrewer1}{Unseen} Attacks.} We report the (per-example) worst-case, confidence-thresholded \RTE and FPR across \textbf{all} unseen attacks on MNIST, SVHN and Cifar10. On Cifar10, we additionally present results for all attacks except $L_\infty$ adversarial examples with larger $\epsilon=0.06$ (indicated in \textcolor{colorbrewer2}{blue}). \ConfTrain is able to outperform all baselines, including \Wong and \TRADES, significantly on MNIST and SVHN. On Cifar10, \ConfTrain performs poorly on $L_\infty$ adversarial examples with larger $\epsilon = 0.06$. However, excluding these adversarial examples, \ConfTrain also outperforms all baselines on Cifar10. \textbf{*} Pre-trained models with different architectures.}
    \label{tab:supp-experiments-all}
\end{table}

\subsection{Analysis}
\label{subsec:supp-experiments-analysis}

\textbf{Confidence Histograms:}
For further analysis, \figref{fig:supp-experiments-histograms} shows confidence histograms for \AdvTrain and \ConfTrain on MNIST and Cifar10. The confidence histograms for \ConfTrain reflect the expected behavior: adversarial examples are mostly successful in changing the label, which is supported by high \RTE values for confidence threshold $\tau = 0$, but their confidence is pushed towards uniform distributions. For \AdvTrain, in contrast, successful adversarial examples -- fewer in total -- generally obtain high confidence. As a result, while confidence thresholding generally benefits \AdvTrain, the improvement is not as significant as for \ConfTrain.

\textbf{Confidence Along Adversarial Directions:}
In \figref{fig:supp-experiments-analysis}, we plot the probabilities for all ten classes along an adversarial direction. We note that these directions do not necessarily correspond to successful or high-confidence adversarial examples. Instead, we chose the first 10 test examples on SVHN and Cifar10. The adversarial examples were obtained using our $L_\infty$ \PGD-\FConf attack with $T = 1000$ iterations and zero initialization for $\epsilon = 0.03$. For \AdvTrain, we usually observe a change in predictions along these directions; some occur within $\|\delta\|_\infty \leq \epsilon$, corresponding to successful adversarial examples (within $\epsilon$), some occur for $\|\delta\|_\infty > \epsilon$, corresponding to unsuccessful adversarial examples (within $\epsilon$). However, \AdvTrain always assigns high confidence. Thus, when allowing larger adversarial perturbations at test time, robustness of \AdvTrain reduces significantly. For \ConfTrain, in contrast, there are only few such cases; more often, the model achieves a near uniform prediction for small $\|\delta\|_\infty$ and extrapolates this behavior beyond the $\epsilon$-ball used for training. On SVHN, this behavior successfully allows to generalize the robustness to larger adversarial perturbations. Furthermore, these plots illustrate why using more iterations at test time, and using techniques such as momentum and backtracking, are necessary to find adversarial examples as the objective becomes more complex compared to \AdvTrain.

\textbf{Confidence Along Interpolation:}
In \figref{fig:supp-experiments-interpolation}, on MNIST, we additionally illustrate the advantage of \ConfTrain with respect to the toy example in Proposition \ref{prop:toy-example}. Here, we consider the case where the $\epsilon$-balls of two training or test examples (in different classes) overlap. As we show in Proposition \ref{prop:toy-example}, adversarial training is not able to handle such cases, resulting in the  trade-off between accuracy in robustness reported in the literature \citep{TsiprasARXIV2018,StutzCVPR2019,RaghunathanARXIV2019,ZhangICML2019}. This is because adversarial training enforces high-confidence predictions on both $\epsilon$-balls (corresponding to different classes), resulting in an obvious conflict. \ConfTrain, in contrast, enforces uniform predictions throughout the largest parts of both $\epsilon$-balls, resolving the conflict.

\subsection{Results}
\label{subsec:supp-experiments-results}

\textbf{Main Results for $98\%$ and $95\%$ TPR:}
\tabref{tab:supp-experiments-98} reports our main results requiring only $98\%$TPR; \tabref{tab:supp-experiments-95} shows results for $95\%$TPR. This implies, that compared to $99\%$TPR, up to $1\%$ (or $4\%$) more correctly classified test examples can be rejected, increasing the confidence threshold and potentially improving robustness. For relatively simple tasks such as MNIST and SVHN, where \TE is low, this is a significant ``sacrifice''. However, as can be seen, robustness in terms of \RTE only improves slightly. We found that the same holds for $95\%$TPR, however, rejecting more than $2\%$ of correctly classified examples seems prohibitive large for the considered datasets.

\textbf{Worst-Case Across \textbf{\textcolor{colorbrewer1}{Unseen}} Attacks:}
\tabref{tab:supp-experiments-all} reports \emph{per-example} worst-case \RTE and FPR for $99\%$TPR considering \textbf{all} \textcolor{colorbrewer1}{unseen} attacks. On MNIST and SVHN, \RTE increases to nearly $100\%$ for \AdvTrain, both \AdvTrainHalf and \AdvTrainFull. \ConfTrain, in contrast, is able to achieve considerably lower \RTE: $23.9\%$ on MNIST and $61.1\%$ on SVHN. Only on Cifar10, \ConfTrain does not result in a significant improvement; all methods, including related work such as \Wong and \TRADES yield \RTE of $94\%$ or higher. However, this is mainly due to the poor performance of \ConfTrain against large $L_\infty$ adversarial examples with $\epsilon = 0.06$. Excluding these adversarial examples (right most table, indicated in \textcolor{colorbrewer2}{blue}) shows that \RTE improves to $77.6\%$ for \ConfTrain, while \RTE for the remaining methods remains nearly unchanged. Overall, these experiments emphasize that \ConfTrain is able to generalize robustness to previously unseen attacks.

\textbf{Per-Attack Results:}
In \tabref{tab:supp-experiments-main-mnist-1} to \ref{tab:supp-experiments-main-cifar10-4}, we break down our main results regarding all used $L_p$ attacks for $p \in \{\infty, 2, 1, 0\}$. For simplicity we focus on \PGD-\FCE and \PGD-\FConf while reporting the used black-box attacks together, \ie, taking the per-example worst-case adversarial examples across all black-box attacks. For comparison, we also include the area under the ROC curve (ROC AUC), non-thresholded \TE and non-thresholded \RTE. On MNIST, where \AdvTrain performs very well in practice, it is striking that for $\nicefrac{4}{3}\epsilon = 0.4$ even black-box attacks are able to reduce robustness completely, resulting in high \RTE. This observation also transfers to SVHN and Cifar10. For \ConfTrain, black-box attacks are only effective on Cifar10, where they result in roughly $87\%$ \RTE with $\tau$@$99\%$TPR. For the $L_2$, $L_1$ and $L_0$ attacks we can make similar observations. Across all $L_p$ norms, it can also be seen that \PGD-\FCE performs significantly worse against our \ConfTrain compared to \AdvTrain, which shows that it is essential to optimize the right objective to evaluate the robustness of defenses and adversarially trained models, \ie, maximize confidence against \ConfTrain.

\textbf{Results on Corrupted MNIST/Cifar10:}
We also conducted experiments on MNIST-C \citep{MuICMLWORK2019} and Cifar10-C \citep{HendrycksARXIV2019}. These datasets are variants of MNIST and Cifar10 that contain common perturbations of the original images obtained from various types of noise, blur or transformations; examples include zoom or motion blue, Gaussian and shot noise, rotations, translations and shear. \tabref{tab:supp-experiments-corruption-mnist-1} to \ref{tab:supp-experiments-corruption-cifar10-2} presents the per-corruption results on MNIST-C and Cifar10-C, respectively. Here, \texttt{all} includes all corruptions and \texttt{mean} reports the average results across all corruptions. We note that, due to the thresholding, different numbers of corrupted examples are left after detection for different corruptions. Thus, the distinction between \texttt{all} and \texttt{mean} is meaningful. Striking is the performance of \ConfTrain on noise corruptions such as \texttt{gaussian\_noise} or \texttt{shot\_noise}. Here, \ConfTrain is able to reject $100\%$ of the corrupted examples, resulting in a thresholded \TE of $0\%$. This is in stark contrast to \AdvTrain, exhibiting a \TE of roughly $15\%$ after rejection on Cifar10-C. On the remaining corruptions, \ConfTrain is able to perform slightly better than \AdvTrain, which is often due to higher detection rate, \ie, higher ROC AUC. On, Cifar10, the generally lower \TE of \ConfTrain also contributes to the results. Overall, this illustrates that \ConfTrain is able to preserve the inductive bias of predicting near-uniform distribution on noise similar to $L_\infty$ adversarial examples as seen during training.

\begin{table*}[t]
	\centering
	\scriptsize
    \begin{tabularx}{1\textwidth}{| r | X ||c|c|c|c|c||c|c|}
\hline
\multicolumn{9}{|c|}{\textbf{MNIST:} Supplementary Results for $\mathbf{L_\infty}$ Adversarial Examples}\\
\hline
&& \multicolumn{5}{c||}{\begin{tabular}{c}\textbf{Detection Setting}\\$\tau$@$99\%$TPR\end{tabular}} & \multicolumn{2}{c|}{\begin{tabular}{c}\textbf{Standard Setting}\\$\tau{=}0$\end{tabular}}\\
\hline
Attack & Training & \begin{tabular}{c}ROC\\AUC\end{tabular} & \begin{tabular}{c}FPR\\ in \%\end{tabular} & \begin{tabular}{c}\TE\\ in \%\end{tabular} & \begin{tabular}{c}\RTE\\ in \%\end{tabular} & $\tau$ & \begin{tabular}{c}\TE\\ in \%\end{tabular} & \begin{tabular}{c}\RTE\\ in \%\end{tabular}\\
\hline
\hline
\multirow{7}{*}{Worst-Case ($L_\infty$, $\epsilon = 0.30$)} & \Normal & 0.34 & 99.3 & 0.1 & 100.0 & 0.98 & 0.4 & 100.0\\
& \AdvTrainHalf & 0.97 & 1.0 & 0.0 & 1.0 & 1.00 & 0.5 & 7.2\\
& \AdvTrainFull & 0.98 & 1.0 & 0.0 & 1.0 & 1.00 & 0.5 & 7.1\\
& \ConfTrain & 0.99 & 6.9 & 0.2 & 7.1 & 0.85 & 0.4 & 48.9\\
& \Wong & 0.86 & 32.1 & 0.9 & 33.9 & 0.51 & 1.8 & 38.0\\
\hline
\end{tabularx}
	\vskip -6px
	\caption{\textbf{Per-Attack Results on MNIST, Part I ($\mathbf{L_\infty}$).} Per-attack results considering \PGD-\FCE, as in \cite{MadryICLR2018}, our \PGD-\FConf and the remaining black-box attacks for the $L_\infty$ threat model. The used $\epsilon$ values are reported in the left-most column. For the black-box attacks, we report the per-example worst-case across all black-box attacks. In addition to FPR and \RTE, we include ROC AUC, \TE as well as \TE and \RTE in the standard, non-thresholded setting, as reference.}
	\label{tab:supp-experiments-main-mnist-1}
\end{table*}
\begin{table*}[t]
    \centering
    \scriptsize
    \begin{tabularx}{1\textwidth}{| r | X ||c|c|c|c|c||c|c|}
\hline
\multicolumn{9}{|c|}{\textbf{MNIST:} Supplementary Results for $\mathbf{L_2}$ Adversarial Examples}\\
\hline
&& \multicolumn{5}{c||}{\begin{tabular}{c}\textbf{Detection Setting}\\$\tau$@$99\%$TPR\end{tabular}} & \multicolumn{2}{c|}{\begin{tabular}{c}\textbf{Standard Setting}\\$\tau{=}0$\end{tabular}}\\
\hline
Attack & Training & \begin{tabular}{c}ROC\\AUC\end{tabular} & \begin{tabular}{c}FPR\\ in \%\end{tabular} & \begin{tabular}{c}\TE\\ in \%\end{tabular} & \begin{tabular}{c}\RTE\\ in \%\end{tabular} & $\tau$ & \begin{tabular}{c}\TE\\ in \%\end{tabular} & \begin{tabular}{c}\RTE\\ in \%\end{tabular}\\
\hline
\hline
\multirow{7}{*}{Worst-Case ($L_2$, $\epsilon = 1.5$)} & \Normal & 0.39 & 98.9 & 0.1 & 99.6 & 0.98 & 0.4 & 100.0\\
& \AdvTrainHalf & 0.96 & 2.3 & 0.0 & 2.4 & 1.00 & 0.5 & 13.6\\
& \AdvTrainFull & 0.96 & 2.6 & 0.0 & 2.7 & 1.00 & 0.5 & 12.0\\
& \ConfTrain & 1.00 & 0.0 & 0.2 & 0.2 & 0.85 & 0.4 & 6.7\\
& \Wong & 0.94 & 7.9 & 0.9 & 9.6 & 0.51 & 1.8 & 15.2\\
\hline
\end{tabularx}
    \vskip -6px
    \caption{\textbf{Per-Attack Results on MNIST, Part II ($\mathbf{L_2}$).} Per-attack results considering \PGD-\FCE, as in \cite{MadryICLR2018}, our \PGD-\FConf and the remaining black-box attacks for the $L_2$ threat model. The used $\epsilon$ values are reported in the left-most column. For the black-box attacks, we report the per-example worst-case across all black-box attacks. In addition to FPR and \RTE we include ROC AUC, \TE as well as \TE and \RTE in the standard, non-thresholded setting, as reference.}
    \label{tab:supp-experiments-main-mnist-2}
\end{table*}
\begin{table*}[t]
    \centering
    \scriptsize
    \begin{tabularx}{1\textwidth}{| r | X ||c|c|c|c|c||c|c|}
\hline
\multicolumn{9}{|c|}{\textbf{MNIST:} Supplementary Results for $\mathbf{L_1}$ Adversarial Examples}\\
\hline
&& \multicolumn{5}{c||}{\begin{tabular}{c}\textbf{Detection Setting}\\$\tau$@$99\%$TPR\end{tabular}} & \multicolumn{2}{c|}{\begin{tabular}{c}\textbf{Standard Setting}\\$\tau{=}0$\end{tabular}}\\
\hline
Attack & Training & \begin{tabular}{c}ROC\\AUC\end{tabular} & \begin{tabular}{c}FPR\\ in \%\end{tabular} & \begin{tabular}{c}\TE\\ in \%\end{tabular} & \begin{tabular}{c}\RTE\\ in \%\end{tabular} & $\tau$ & \begin{tabular}{c}\TE\\ in \%\end{tabular} & \begin{tabular}{c}\RTE\\ in \%\end{tabular}\\
\hline
\hline
\multirow{7}{*}{Worst-Case ($L_1$, $\epsilon = 12$)} & \Normal & 0.44 & 98.9 & 0.1 & 99.6 & 0.98 & 0.4 & 100.0\\
& \AdvTrainHalf & 0.93 & 8.6 & 0.0 & 8.8 & 1.00 & 0.5 & 26.3\\
& \AdvTrainFull & 0.91 & 8.2 & 0.0 & 8.4 & 1.00 & 0.5 & 23.6\\
& \ConfTrain & 1.00 & 0.8 & 0.2 & 1.0 & 0.85 & 0.4 & 16.5\\
& \Wong & 0.87 & 25.5 & 0.9 & 27.3 & 0.51 & 1.8 & 32.3\\
\hline
\end{tabularx}
    \vskip -6px
    \caption{\textbf{Per-Attack Results on MNIST, Part III ($\mathbf{L_1}$).} Per-attack results considering \PGD-\FCE, as in \cite{MadryICLR2018}, our \PGD-\FConf and the remaining black-box attacks for the $L_1$ threat model. The used $\epsilon$ values are reported in the left-most column. For the black-box attacks, we report the per-example worst-case across all black-box attacks. In addition to FPR and \RTE we include ROC AUC, \TE as well as \TE and \RTE in the standard, non-thresholded setting, as reference.}
    \label{tab:supp-experiments-main-mnist-3}
\end{table*}
\begin{table*}[t]
    \centering
    \scriptsize
    \begin{tabularx}{1\textwidth}{| r | X ||c|c|c|c|c||c|c|}
\hline
\multicolumn{9}{|c|}{\textbf{MNIST:} Supplementary Results for $\mathbf{L_0}$ Adversarial Examples and \textbf{Adversarial Frames}}\\
\hline
&& \multicolumn{5}{c||}{\begin{tabular}{c}\textbf{Detection Setting}\\$\tau$@$99\%$TPR\end{tabular}} & \multicolumn{2}{c|}{\begin{tabular}{c}\textbf{Standard Setting}\\$\tau{=}0$\end{tabular}}\\
\hline
Attack & Training & \begin{tabular}{c}ROC\\AUC\end{tabular} & \begin{tabular}{c}FPR\\ in \%\end{tabular} & \begin{tabular}{c}\TE\\ in \%\end{tabular} & \begin{tabular}{c}\RTE\\ in \%\end{tabular} & $\tau$ & \begin{tabular}{c}\TE\\ in \%\end{tabular} & \begin{tabular}{c}\RTE\\ in \%\end{tabular}\\
\hline
\hline
\multirow{7}{*}{Worst-Case ($L_0$, $\epsilon = 15$)} & \Normal & 0.63 & 91.6 & 0.1 & 92.3 & 0.98 & 0.4 & 99.1\\
& \AdvTrainHalf & 0.94 & 23.0 & 0.0 & 23.5 & 1.00 & 0.5 & 95.3\\
& \AdvTrainFull & 0.98 & 12.9 & 0.0 & 13.3 & 1.00 & 0.5 & 94.3\\
& \ConfTrain & 0.99 & 14.2 & 0.2 & 14.5 & 0.85 & 0.4 & 83.5\\
& \Wong & 0.75 & 64.2 & 0.9 & 66.2 & 0.51 & 1.8 & 74.2\\
\hline
\end{tabularx}
    \vskip -6px
    \caption{\textbf{Per-Attack Results on MNIST, Part IV ($\mathbf{L_0}$, Adversarial Frames).} Per-attack results considering \PGD-\FCE, as in \cite{MadryICLR2018}, our \PGD-\FConf and the remaining black-box attacks for $L_0$ threat models and adversarial frames. The used $\epsilon$ values are reported in the left-most column. For the black-box attacks, we report the per-example worst-case across all black-box attacks. In addition to FPR and \RTE we include ROC AUC, \TE as well as \TE and \RTE in the standard, non-thresholded setting, as reference.}
    \label{tab:supp-experiments-main-mnist-4}
\end{table*}
\begin{table*}[t]
	\centering
	\scriptsize
    \begin{tabularx}{1\textwidth}{| r | X ||c|c|c|c|c||c|c|}
\hline
\multicolumn{9}{|c|}{\textbf{SVHN:} Supplementary Results for $\mathbf{L_\infty}$ Adversarial Examples}\\
\hline
&& \multicolumn{5}{c||}{\begin{tabular}{c}\textbf{Detection Setting}\\$\tau$@$99\%$TPR\end{tabular}} & \multicolumn{2}{c|}{\begin{tabular}{c}\textbf{Standard Setting}\\$\tau{=}0$\end{tabular}}\\
\hline
Attack & Training & \begin{tabular}{c}ROC\\AUC\end{tabular} & \begin{tabular}{c}FPR\\ in \%\end{tabular} & \begin{tabular}{c}\TE\\ in \%\end{tabular} & \begin{tabular}{c}\RTE\\ in \%\end{tabular} & $\tau$ & \begin{tabular}{c}\TE\\ in \%\end{tabular} & \begin{tabular}{c}\RTE\\ in \%\end{tabular}\\
\hline
\hline
\multirow{4}{*}{Worst-Case ($L_\infty$, $\epsilon = 0.03$)} & \Normal & 0.17 & 95.8 & 2.6 & 99.9 & 0.78 & 3.6 & 100.0\\
& \AdvTrainHalf & 0.55 & 52.3 & 2.5 & 55.6 & 0.56 & 3.4 & 57.3\\
& \AdvTrainFull & 0.73 & 42.1 & 4.6 & 47.5 & 0.27 & 5.9 & 48.4\\
& \ConfTrain & 0.70 & 35.5 & 2.1 & 38.5 & 0.60 & 2.9 & 97.8\\
\hline
\multirow{4}{*}{\PGD\FConf ($L_\infty$, $\epsilon = 0.03$)} & \Normal & 0.17 & 95.8 & 2.6 & 99.9 & 0.78 & 3.6 & 99.9\\
\hline
\end{tabularx}
	\vskip -6px
	\caption{\textbf{Per-Attack Results on SVHN, Part I ($\mathbf{L_\infty}$).} Per-attack results considering \PGD-\FCE, as in \cite{MadryICLR2018}, our \PGD-\FConf and the remaining black-box attacks for the $L_\infty$ threat model. The used $\epsilon$ values are reported in the left-most column. For the black-box attacks, we report the per-example worst-case across all black-box attacks. In addition to FPR and \RTE we include ROC AUC, \TE as well as \TE and \RTE in the standard, non-thresholded setting, as reference.}
    \label{tab:supp-experiments-main-svhn-1}
\end{table*}
\begin{table*}[t]
    \centering
    \scriptsize
    \begin{tabularx}{1\textwidth}{| r | X ||c|c|c|c|c||c|c|}
\hline
\multicolumn{9}{|c|}{\textbf{SVHN:} Supplementary Results for $\mathbf{L_2}$ Adversarial Examples}\\
\hline
&& \multicolumn{5}{c||}{\begin{tabular}{c}\textbf{Detection Setting}\\$\tau$@$99\%$TPR\end{tabular}} & \multicolumn{2}{c|}{\begin{tabular}{c}\textbf{Standard Setting}\\$\tau{=}0$\end{tabular}}\\
\hline
Attack & Training & \begin{tabular}{c}ROC\\AUC\end{tabular} & \begin{tabular}{c}FPR\\ in \%\end{tabular} & \begin{tabular}{c}\TE\\ in \%\end{tabular} & \begin{tabular}{c}\RTE\\ in \%\end{tabular} & $\tau$ & \begin{tabular}{c}\TE\\ in \%\end{tabular} & \begin{tabular}{c}\RTE\\ in \%\end{tabular}\\
\hline
\hline
\multirow{4}{*}{Worst-Case ($L_2$, $\epsilon = 0.5$)} & \Normal & 0.27 & 93.3 & 2.6 & 97.4 & 0.78 & 3.6 & 97.7\\
& \AdvTrainHalf & 0.53 & 65.2 & 2.5 & 68.6 & 0.56 & 3.4 & 70.7\\
& \AdvTrainFull & 0.67 & 60.1 & 4.6 & 65.8 & 0.27 & 5.9 & 66.3\\
& \ConfTrain & 0.87 & 17.2 & 2.1 & 20.0 & 0.60 & 2.9 & 63.3\\
\hline
\multirow{4}{*}{\PGD\FConf ($L_2$, $\epsilon = 0.5$)} & \Normal & 0.32 & 91.1 & 2.6 & 95.1 & 0.78 & 3.6 & 96.4\\
\hline
\end{tabularx}
    \vskip -6px
    \caption{\textbf{Per-Attack Results on SVHN, Part II ($\mathbf{L_2}$).} Per-attack results considering \PGD-\FCE, as in \cite{MadryICLR2018}, our \PGD-\FConf and the remaining black-box attacks for $L_2$ threat model. The used $\epsilon$ values are reported in the left-most column. For the black-box attacks, we report the per-example worst-case across all black-box attacks. In addition to FPR and \RTE we include ROC AUC, \TE as well as \TE and \RTE in the standard, non-thresholded setting, as reference.}
    \label{tab:supp-experiments-main-svhn-2}
\end{table*}
\begin{table*}[t]
    \centering
    \scriptsize
    \begin{tabularx}{1\textwidth}{| r | X ||c|c|c|c|c||c|c|}
    \hline
    \multicolumn{9}{|c|}{\textbf{SVHN:} Supplementary Results for $\mathbf{L_1}$, $\mathbf{L_0}$ Adversarial Examples and \textbf{Adversarial Frames}}\\
    \hline
    && \multicolumn{5}{c||}{\begin{tabular}{c}\textbf{Detection Setting}\\$\tau$@$99\%$TPR\end{tabular}} & \multicolumn{2}{c|}{\begin{tabular}{c}\textbf{Standard Setting}\\$\tau{=}0$\end{tabular}}\\
    \hline
    Attack & Training & \begin{tabular}{c}ROC\\AUC\end{tabular} & \begin{tabular}{c}FPR\\ in \%\end{tabular} & \begin{tabular}{c}\TE\\ in \%\end{tabular} & \begin{tabular}{c}\RTE\\ in \%\end{tabular} & $\tau$ & \begin{tabular}{c}\TE\\ in \%\end{tabular} & \begin{tabular}{c}\RTE\\ in \%\end{tabular}\\
    \hline
    \hline
\multirow{4}{*}{Worst-Case ($L_1$, $\epsilon = 18$)} & \Normal & 0.17 & 95.9 & 2.6 & 100.0 & 0.78 & 3.6 & 100.0\\
& \AdvTrainHalf & 0.09 & 95.0 & 2.5 & 98.7 & 0.56 & 3.4 & 98.7\\
& \AdvTrainFull & 0.14 & 93.1 & 4.6 & 99.3 & 0.27 & 5.9 & 99.3\\
& \ConfTrain & 0.82 & 24.3 & 2.1 & 27.2 & 0.60 & 2.9 & 72.8\\
\hline
\multirow{4}{*}{\PGD\FConf ($L_1$, $\epsilon = 18$)} & \Normal & 0.17 & 95.9 & 2.6 & 100.0 & 0.78 & 3.6 & 100.0\\
\hline
\end{tabularx}
    \vskip -6px
    \caption{\textbf{Per-Attack Results on SVHN, Part III ($\mathbf{L_1}$, $\mathbf{L_0}$, Adversarial Frames).} Per-attack results considering \PGD-\FCE, as in \cite{MadryICLR2018}, our \PGD-\FConf and the remaining black-box attacks for $L_1$, $L_0$ threat models and adversarial frames. The used $\epsilon$ values are reported in the left-most column. For the black-box attacks, we report the per-example worst-case across all black-box attacks. In addition to FPR and \RTE we include ROC AUC, \TE as well as \TE and \RTE in the standard, non-thresholded setting, as reference.}
    \label{tab:supp-experiments-main-svhn-3}
\end{table*}
\begin{table*}[t]
	\centering
	\tiny
    \begin{tabularx}{1\textwidth}{| r | X ||c|c|c|c|c||c|c|}
\hline
\multicolumn{9}{|c|}{\textbf{CIFAR10:} Supplementary Results for $\mathbf{L_\infty}$ Adversarial Examples}\\
\hline
&& \multicolumn{5}{c||}{\begin{tabular}{c}\textbf{Detection Setting}\\$\tau$@$99\%$TPR\end{tabular}} & \multicolumn{2}{c|}{\begin{tabular}{c}\textbf{Standard Setting}\\$\tau{=}0$\end{tabular}}\\
\hline
Attack & Training & \begin{tabular}{c}ROC\\AUC\end{tabular} & \begin{tabular}{c}FPR\\ in \%\end{tabular} & \begin{tabular}{c}\TE\\ in \%\end{tabular} & \begin{tabular}{c}\RTE\\ in \%\end{tabular} & $\tau$ & \begin{tabular}{c}\TE\\ in \%\end{tabular} & \begin{tabular}{c}\RTE\\ in \%\end{tabular}\\
\hline
\hline
\multirow{8}{*}{Worst-Case ($L_\infty$, $\epsilon = 0.03$)} & \Normal & 0.20 & 93.0 & 7.4 & 100.0 & 0.59 & 8.3 & 100.0\\
& \AdvTrainHalf & 0.64 & 47.6 & 15.1 & 62.3 & 0.35 & 16.6 & 62.7\\
& \AdvTrainFull & 0.64 & 42.3 & 18.3 & 59.5 & 0.29 & 19.4 & 59.9\\
& \ConfTrain & 0.60 & 59.9 & 8.7 & 67.9 & 0.40 & 10.1 & 96.7\\
& \Wong & 0.66 & 35.3 & 17.6 & 53.0 & 0.24 & 18.4 & 53.2\\
\hline
\end{tabularx}
	\vskip -6px
	\caption{\textbf{Per-Attack Results on Cifar10, Part I ($\mathbf{L_\infty}$).} Per-attack results considering \PGD-\FCE, as in \cite{MadryICLR2018}, our \PGD-\FConf and the remaining black-box attacks for $L_\infty$ and $L_2$ threat models. The used $\epsilon$ values are reported in the left-most column. For the black-box attacks, we report the per-example worst-case across all black-box attacks. In addition to FPR and \RTE we include ROC AUC, \TE as well as \TE and \RTE in the standard, non-thresholded setting, as reference.}
    \label{tab:supp-experiments-main-cifar10-1}
\end{table*}
\begin{table*}[t]
    \centering
    \scriptsize
    \begin{tabularx}{1\textwidth}{| r | X ||c|c|c|c|c||c|c|}
\hline
\multicolumn{9}{|c|}{\textbf{CIFAR10:} Supplementary Results for $\mathbf{L_2}$ Adversarial Examples}\\
\hline
&& \multicolumn{5}{c||}{\begin{tabular}{c}\textbf{Detection Setting}\\$\tau$@$99\%$TPR\end{tabular}} & \multicolumn{2}{c|}{\begin{tabular}{c}\textbf{Standard Setting}\\$\tau{=}0$\end{tabular}}\\
\hline
Attack & Training & \begin{tabular}{c}ROC\\AUC\end{tabular} & \begin{tabular}{c}FPR\\ in \%\end{tabular} & \begin{tabular}{c}\TE\\ in \%\end{tabular} & \begin{tabular}{c}\RTE\\ in \%\end{tabular} & $\tau$ & \begin{tabular}{c}\TE\\ in \%\end{tabular} & \begin{tabular}{c}\RTE\\ in \%\end{tabular}\\
\hline
\hline
\multirow{8}{*}{Worst-Case ($L_2$, $\epsilon = 1$)} & \Normal & 0.20 & 93.0 & 7.4 & 100.0 & 0.59 & 8.3 & 100.0\\
& \AdvTrainHalf & 0.59 & 59.1 & 15.1 & 73.9 & 0.35 & 16.6 & 74.4\\
& \AdvTrainFull & 0.57 & 55.8 & 18.3 & 73.2 & 0.29 & 19.4 & 73.4\\
& \ConfTrain & 0.75 & 42.6 & 8.7 & 50.4 & 0.40 & 10.1 & 83.2\\
& \Wong & 0.68 & 34.3 & 17.6 & 52.0 & 0.24 & 18.4 & 52.2\\
\hline
\end{tabularx}
    \vskip -6px
    \caption{\textbf{Per-Attack Results on Cifar10, Part II ($\mathbf{L_2}$).} Per-attack results considering \PGD-\FCE, as in \cite{MadryICLR2018}, our \PGD-\FConf and the remaining black-box attacks for the $L_2$ threat model. The used $\epsilon$ values are reported in the left-most column. For the black-box attacks, we report the per-example worst-case across all black-box attacks. In addition to FPR and \RTE we include ROC AUC, \TE as well as \TE and \RTE in the standard, non-thresholded setting, as reference.}
    \label{tab:supp-experiments-main-cifar10-2}
\end{table*}
\begin{table*}[t]
    \centering
    \scriptsize
    \begin{tabularx}{1\textwidth}{| r | X ||c|c|c|c|c||c|c|}
\hline
\multicolumn{9}{|c|}{\textbf{CIFAR10:} Supplementary Results for $\mathbf{L_1}$ Adversarial Examples}\\
\hline
&& \multicolumn{5}{c||}{\begin{tabular}{c}\textbf{Detection Setting}\\$\tau$@$99\%$TPR\end{tabular}} & \multicolumn{2}{c|}{\begin{tabular}{c}\textbf{Standard Setting}\\$\tau{=}0$\end{tabular}}\\
\hline
Attack & Training & \begin{tabular}{c}ROC\\AUC\end{tabular} & \begin{tabular}{c}FPR\\ in \%\end{tabular} & \begin{tabular}{c}\TE\\ in \%\end{tabular} & \begin{tabular}{c}\RTE\\ in \%\end{tabular} & $\tau$ & \begin{tabular}{c}\TE\\ in \%\end{tabular} & \begin{tabular}{c}\RTE\\ in \%\end{tabular}\\
\hline
\hline
\multirow{8}{*}{Worst-Case ($L_1$, $\epsilon = 18$)} & \Normal & 0.20 & 93.0 & 7.4 & 100.0 & 0.59 & 8.3 & 100.0\\
& \AdvTrainHalf & 0.39 & 80.1 & 15.1 & 95.2 & 0.35 & 16.6 & 95.2\\
& \AdvTrainFull & 0.39 & 76.8 & 18.3 & 94.3 & 0.29 & 19.4 & 94.4\\
& \ConfTrain & 0.67 & 47.6 & 8.7 & 55.4 & 0.40 & 10.1 & 85.4\\
& \Wong & 0.63 & 40.0 & 17.6 & 57.7 & 0.24 & 18.4 & 57.9\\
\hline
\end{tabularx}
    \vskip -6px
    \caption{\textbf{Per-Attack Results on Cifar10, Part III ($\mathbf{L_1}$).} Per-attack results considering \PGD-\FCE, as in \cite{MadryICLR2018}, our \PGD-\FConf and the remaining black-box attacks for the $L_1$ threat model. The used $\epsilon$ values are reported in the left-most column. For the black-box attacks, we report the per-example worst-case across all black-box attacks. In addition to FPR and \RTE we include ROC AUC, \TE as well as \TE and \RTE in the standard, non-thresholded setting, as reference.}
    \label{tab:supp-experiments-main-cifar10-3}
\end{table*}
\begin{table*}[t]
    \centering
    \scriptsize
    \begin{tabularx}{1\textwidth}{| r | X ||c|c|c|c|c||c|c|}
\hline
\multicolumn{9}{|c|}{\textbf{CIFAR10:} Supplementary Results for $\mathbf{L_0}$ Adversarial Examples and \textbf{Adversarial Frames}}\\
\hline
&& \multicolumn{5}{c||}{\begin{tabular}{c}\textbf{Detection Setting}\\$\tau$@$99\%$TPR\end{tabular}} & \multicolumn{2}{c|}{\begin{tabular}{c}\textbf{Standard Setting}\\$\tau{=}0$\end{tabular}}\\
\hline
Attack & Training & \begin{tabular}{c}ROC\\AUC\end{tabular} & \begin{tabular}{c}FPR\\ in \%\end{tabular} & \begin{tabular}{c}\TE\\ in \%\end{tabular} & \begin{tabular}{c}\RTE\\ in \%\end{tabular} & $\tau$ & \begin{tabular}{c}\TE\\ in \%\end{tabular} & \begin{tabular}{c}\RTE\\ in \%\end{tabular}\\
\hline
\hline
\multirow{8}{*}{Worst-Case ($L_0$, $\epsilon = 10$)} & \Normal & 0.80 & 77.7 & 7.4 & 84.6 & 0.59 & 8.3 & 96.2\\
& \AdvTrainHalf & 0.80 & 59.3 & 15.1 & 74.1 & 0.35 & 16.6 & 75.9\\
& \AdvTrainFull & 0.75 & 54.7 & 18.3 & 72.0 & 0.29 & 19.4 & 73.2\\
& \ConfTrain & 0.97 & 14.4 & 8.7 & 21.9 & 0.40 & 10.1 & 54.7\\
& \Wong & 0.84 & 21.4 & 17.6 & 39.1 & 0.24 & 18.4 & 39.8\\
\hline
\end{tabularx}
    \vskip -6px
    \caption{\textbf{Per-Attack Results on Cifar10, Part IV ($\mathbf{L_0}$, Adversarial Frames).} Per-attack results considering \PGD-\FCE, as in \cite{MadryICLR2018}, our \PGD-\FConf and the remaining black-box attacks for the $L_0$ threat model and adversarial frames. The used $\epsilon$ values are reported in the left-most column. For the black-box attacks, we report the per-example worst-case across all black-box attacks. In addition to FPR and \RTE we include ROC AUC, \TE as well as \TE and \RTE in the standard, non-thresholded setting, as reference.}
    \label{tab:supp-experiments-main-cifar10-4}
\end{table*}
\begin{table*}
    \centering
    \scriptsize
    \begin{tabularx}{1\textwidth}{| r | X ||c|c|c|c|c||c|}
\hline
\multicolumn{8}{|c|}{\textbf{MNIST-C:} Supplementary Results for Corruption}\\
\hline
&& \multicolumn{5}{c||}{\begin{tabular}{c}\textbf{Detection Setting}\\$\tau$@$99\%$TPR\end{tabular}} & \multicolumn{1}{c|}{\begin{tabular}{c}\textbf{Standard}\\\textbf{Setting}\\$\tau{=}0$\end{tabular}}\\
\hline
Corruption & Training & \begin{tabular}{c}ROC\\AUC\end{tabular} & \begin{tabular}{c}FPR\\ in \%\end{tabular} & \begin{tabular}{c}TNR\\ in \%\end{tabular} & \begin{tabular}{c}\TE\\ in \%\end{tabular} & $\tau$ & \begin{tabular}{c}\TE\\ in \%\end{tabular}\\
\hline
\hline
\multirow{7}{*}{\texttt{all}} & \Normal & 0.75 & 82.8 & 17.2 & 31.9 & 0.98 & 36.4\\
& \AdvTrainHalf & 0.80 & 59.0 & 41.0 & 4.0 & 1.00 & 26.9\\
& \AdvTrainFull & 0.78 & 63.5 & 36.5 & 8.6 & 1.00 & 27.4\\
& \ConfTrain & 0.96 & 29.0 & 71.0 & 6.7 & 0.85 & 41.5\\
& \Wong & 0.75 & 75.0 & 25.0 & 5.4 & 0.51 & 14.0\\
\hline
\end{tabularx}
    \caption{\textbf{Per-Corruptions Results on MNIST-C, PART I.} Results on MNIST-C, broken down by individual corruptions (first column); \texttt{mean} are the averaged results over all corruptions. We report ROC AUC, FPR and the true negative rate (TNR) in addition to the thresholded and unthresholded \TE on the corrupted examples. The table is continued in \tabref{tab:supp-experiments-corruption-mnist-2}.}
    \label{tab:supp-experiments-corruption-mnist-1}
\end{table*}
\begin{table*}
    \centering
    \scriptsize
    \begin{tabularx}{1\textwidth}{| r | X ||c|c|c|c|c||c|}
\hline
\multicolumn{8}{|c|}{\textbf{MNIST-C:} Supplementary Results for Corruption}\\
\hline
&& \multicolumn{5}{c||}{\begin{tabular}{c}\textbf{Detection Setting}\\$\tau$@$99\%$TPR\end{tabular}} & \multicolumn{1}{c|}{\begin{tabular}{c}\textbf{Standard}\\\textbf{Setting}\\$\tau{=}0$\end{tabular}}\\
\hline
Corruption & Training & \begin{tabular}{c}ROC\\AUC\end{tabular} & \begin{tabular}{c}FPR\\ in \%\end{tabular} & \begin{tabular}{c}TNR\\ in \%\end{tabular} & \begin{tabular}{c}\TE\\ in \%\end{tabular} & $\tau$ & \begin{tabular}{c}\TE\\ in \%\end{tabular}\\
\hline
\hline
\multirow{7}{*}{\texttt{scale}} & \Normal & 0.84 & 89.5 & 10.5 & 0.7 & 0.98 & 3.1\\
& \AdvTrainHalf & 0.86 & 78.5 & 21.5 & 0.1 & 1.00 & 3.0\\
& \AdvTrainFull & 0.80 & 82.1 & 17.9 & 0.4 & 1.00 & 3.0\\
& \ConfTrain & 0.75 & 96.7 & 3.3 & 0.8 & 0.85 & 2.0\\
& \Wong & 0.84 & 64.4 & 35.6 & 8.1 & 0.51 & 18.1\\
\hline
\end{tabularx}
    \caption{\textbf{Per-Corruptions Results on MNIST-C, PART II.} Continued results of \tabref{tab:supp-experiments-corruption-mnist-1} including results on MNIST-C focusing on individual corruptions. texttt{mean} are the averaged results over all corruptions. We report ROC AUC, FPR and the true negative rate (TNR) in addition to the thresholded and unthresholded \TE on the corrupted examples.}
    \label{tab:supp-experiments-corruption-mnist-2}
\end{table*}
\begin{table*}
    \centering
    \tiny
    \begin{tabularx}{1\textwidth}{| r | X ||c|c|c|c|c||c|}
\hline
\multicolumn{8}{|c|}{\textbf{CIFAR10-C:} Supplementary Results for Corruption}\\
\hline
&& \multicolumn{5}{c||}{\begin{tabular}{c}\textbf{Detection Setting}\\$\tau$@$99\%$TPR\end{tabular}} & \multicolumn{1}{c|}{\begin{tabular}{c}\textbf{Standard}\\\textbf{Setting}\\$\tau{=}0$\end{tabular}}\\
\hline
Corruption & Training & \begin{tabular}{c}ROC\\AUC\end{tabular} & \begin{tabular}{c}FPR\\ in \%\end{tabular} & \begin{tabular}{c}TNR\\ in \%\end{tabular} & \begin{tabular}{c}\TE\\ in \%\end{tabular} & $\tau$ & \begin{tabular}{c}\TE\\ in \%\end{tabular}\\
\hline
\hline
\multirow{8}{*}{\texttt{all}} & \Normal & 0.57 & 97.1 & 2.9 & 12.2 & 0.59 & 13.7\\
& \AdvTrainHalf & 0.53 & 96.2 & 3.8 & 16.2 & 0.35 & 18.1\\
& \AdvTrainFull & 0.53 & 97.2 & 2.8 & 19.6 & 0.29 & 20.9\\
& \ConfTrain & 0.66 & 72.1 & 27.9 & 10.4 & 0.40 & 27.2\\
& \Wong & 0.53 & 98.2 & 1.8 & 19.2 & 0.24 & 20.2\\
\hline
\end{tabularx}
    \caption{\textbf{Per-Corruptions Results on Cifar10-C, PART I.} Results on Cifar10-C focusing on individual corruptions (first column); texttt{mean} are the averaged results over all corruptions. We report ROC AUC, FPR and the true negative rate (TNR) in addition to the thresholded and unthresholded \TE on the corrupted examples. The table is continued in \tabref{tab:supp-experiments-corruption-cifar10-2}.}
    \label{tab:supp-experiments-corruption-cifar10-1}
\end{table*}
\begin{table*}
    \centering
    \tiny
    \begin{tabularx}{1\textwidth}{| r | X ||c|c|c|c|c||c|}
\hline
\multicolumn{8}{|c|}{\textbf{CIFAR10-C:} Supplementary Results for Corruption}\\
\hline
&& \multicolumn{5}{c||}{\begin{tabular}{c}\textbf{Detection Setting}\\$\tau$@$99\%$TPR\end{tabular}} & \multicolumn{1}{c|}{\begin{tabular}{c}\textbf{Standard}\\\textbf{Setting}\\$\tau{=}0$\end{tabular}}\\
\hline
Corruption & Training & \begin{tabular}{c}ROC\\AUC\end{tabular} & \begin{tabular}{c}FPR\\ in \%\end{tabular} & \begin{tabular}{c}TNR\\ in \%\end{tabular} & \begin{tabular}{c}\TE\\ in \%\end{tabular} & $\tau$ & \begin{tabular}{c}\TE\\ in \%\end{tabular}\\
\hline
\hline
\multirow{8}{*}{\texttt{impulse\_noise}} & \Normal & 0.59 & 97.0 & 3.0 & 13.1 & 0.59 & 14.5\\
& \AdvTrainHalf & 0.53 & 96.4 & 3.6 & 16.4 & 0.35 & 17.9\\
& \AdvTrainFull & 0.52 & 97.6 & 2.4 & 19.8 & 0.29 & 21.1\\
& \ConfTrain & 1.00 & 0.1 & 99.9 & 0.0 & 0.40 & 61.0\\
& \Wong & 0.51 & 98.4 & 1.6 & 18.3 & 0.24 & 19.2\\
\hline
\end{tabularx}
    \caption{\textbf{Per-Corruptions Results on Cifar10-C, PART II} Continued results of \tabref{tab:supp-experiments-corruption-cifar10-1} including results on Cifar10-C focusing on individual corruptions. texttt{mean} are the averaged results over all corruptions. We report ROC AUC, FPR and the true negative rate (TNR) in addition to the thresholded and unthresholded \TE on the corrupted examples.}
    \label{tab:supp-experiments-corruption-cifar10-2}
\end{table*}
\FloatBarrier
\end{appendix}

\end{document}